\documentclass[10pt]{article}
\usepackage[numbers]{natbib}
\usepackage[margin=1in]{geometry}
\usepackage{parskip}

\usepackage{algorithm}
\usepackage[noend]{algorithmic}

\usepackage{enumitem}
\usepackage{graphicx}
\usepackage{wrapfig}
\usepackage{subcaption}

\usepackage[utf8]{inputenc} 
\usepackage[T1]{fontenc}    
\usepackage{hyperref}       
\usepackage{url}            
\usepackage{booktabs}       
\usepackage{amsfonts}       
\usepackage{amsmath}
\usepackage{nicefrac}       
\usepackage{microtype}      
\usepackage{xcolor}         

\usepackage{amssymb}
\usepackage{amsthm}

\newcommand{\cO}{\mathcal{O}}
\newcommand{\tO}{\widetilde{\cO}}
\newcommand{\cM}{\mathcal{M}}
\newcommand{\cS}{\mathcal{S}}
\newcommand{\cA}{\mathcal{A}}
\newcommand{\cB}{\mathcal{B}}
\newcommand{\cX}{\mathcal{X}}
\newcommand{\cP}{\mathcal{P}}

\renewcommand{\P}{\mathbb{P}}
\newcommand{\E}{\mathbb{E}}
\newcommand{\R}{\mathbb{R}}

\renewcommand{\a}{\mathbf{a}}

\newcommand{\<}{\left\langle}
\renewcommand{\>}{\right\rangle}

\newcommand{\reg}{{\rm reg}}
\newcommand{\barreg}{\overline{\rm reg}}

\newcommand{\ones}{{\mathbf{1}}}

\newcommand{\set}[1]{{\left\{ #1 \right\}}}
\newcommand{\abs}[1]{{\left| #1 \right|}}
\newcommand{\paren}[1]{{\left( #1 \right)}}
\newcommand{\brac}[1]{{\left[ #1 \right]}}

\newcommand{\defeq}{\mathrel{\mathop:}=}
\newcommand{\eqdef}{=\mathrel{\mathop:}}
\DeclareMathOperator*{\argmin}{arg\,min}
\DeclareMathOperator*{\argmax}{arg\,max}

\newcommand{\setto}{\leftarrow}
\newcommand{\matrixgamealg}{{\sf MatrixGameAlg}}
\newcommand{\NE}{{\sf MatrixNE}}
\newcommand{\KL}{\textup{KL}}
\newcommand{\eps}{\varepsilon}
\renewcommand{\epsilon}{\varepsilon}
\renewcommand{\hat}{\widehat}

\newcommand{\negap}{{\rm NEGap}}
\newcommand{\NElayerone}{{\rm NEGap-Layer-1}}
\newcommand{\NElayertwo}{{\rm NEGap-Layer-2}}
\newcommand{\ccegap}{{\rm CCEGap}}
\newcommand{\cbeta}{c_\beta}
\newcommand{\Amax}{{A_{\rm max}}}

\newcommand{\up}{\overline}
\newcommand{\low}{\underline}

\newcommand{\grad}{\nabla}
\newcommand{\err}{{\rm ERR}}
\newcommand{\stab}{{\rm STAB}}

\newcommand{\indic}[1]{\mathbf{1}\left\{#1\right\}} 

\newcommand{\norm}[1]{\left\|{#1}\right\|} 
\newcommand{\lone}[1]{\norm{#1}_1} 
\newcommand{\linf}[1]{\norm{#1}_\infty} 

\usepackage{xifthen}

\newtheorem{theorem}{Theorem}
\newtheorem*{theorem*}{Theorem}
\newtheorem{lemma}[theorem]{Lemma}

\newtheorem*{remark*}{Remark}
\newtheorem*{lemma*}{Lemma}

\newtheorem{prop}[theorem]{Proposition}
\newtheorem{definition}[theorem]{Definition}

\renewcommand{\theassumption}{\Alph{assumption}}

\newenvironment{proof-sketch}{\noindent{\bf Proof Sketch}
  \hspace*{1em}}{\qed\bigskip\\}
\newenvironment{proof-idea}{\noindent{\bf Proof Idea}
  \hspace*{1em}}{\qed\bigskip\\}
\newenvironment{proof-of-lemma}[1][{}]{\noindent{\bf Proof of Lemma {#1}}
  \hspace*{1em}}{\qed\bigskip\\}
\newenvironment{proof-of-proposition}[1][{}]{\noindent{\bf
    Proof of Proposition {#1}}
  \hspace*{1em}}{\qed\bigskip\\}
\newenvironment{proof-of-theorem}[1][{}]{\noindent{\bf Proof of Theorem {#1}}
  \hspace*{1em}}{\qed\bigskip\\}
\newenvironment{inner-proof}{\noindent{\bf Proof}\hspace{1em}}{
  $\bigtriangledown$\medskip\\}
\newenvironment{proof-attempt}{\noindent{\bf Proof Attempt}
  \hspace*{1em}}{\qed\bigskip\\}

\newcounter{example}

\newenvironment{example}[1][]{
  \refstepcounter{example}
  \ifthenelse{\isempty{#1}}{%
    \noindent \textbf{Example \theexample:}\hspace*{.05em}
  }{%
    \noindent \textbf{Example \theexample} ({#1})\textbf{:}\hspace*{.05em}
  }
}{%
  $\Diamond$ 
}

\newenvironment{example*}[1][]{
  \ifthenelse{\isempty{#1}}{%
    \noindent \textbf{Example:}\hspace*{.05em}
  }{%
    \noindent \textbf{Example} ({#1})\textbf{:}\hspace*{.05em}
  }
}{%
  $\Diamond$ 
}

\newenvironment{talign}
 {\align}
 {\endalign}
\newenvironment{talign*}
 {\csname align*\endcsname}
 {\endalign}

\allowdisplaybreaks

\hypersetup{
    colorlinks,
    linkcolor={blue!70!black},
    citecolor={blue!70!black},
}
\colorlet{linkequation}{blue}

\def\shownotes{1}  
\ifnum\shownotes=1
\newcommand{\authnote}[2]{{\scriptsize $\ll$\textsf{#1 notes: #2}$\gg$}}
\else
\newcommand{\authnote}[2]{}
\fi

\title{Policy Optimization for Markov Games: Unified Framework and Faster Convergence}

\author{%
Runyu Zhang \!\!\thanks{The two authors contributed equally to this work} \thanks{Harvard University. Email: \texttt{\{runyuzhang@fas.harvard.edu, nali@seas.harvard.edu\}}}
\and
Qinghua Liu \!\!\footnotemark[1] \thanks{Princeton University. Email: \texttt{qinghual@princeton.edu}}
\and
Huan Wang\thanks{Salesforce Research. Email: \texttt{\{huan.wang,cxiong,yu.bai\}@salesforce.com}}
\and
Caiming Xiong\footnotemark[4]
\and
Na Li\footnotemark[2]
\and
Yu Bai\footnotemark[4]
}

\begin{document}

\maketitle

\begin{abstract}
  This paper studies policy optimization algorithms for multi-agent reinforcement learning. We begin by proposing an algorithm framework for two-player zero-sum Markov Games in the full-information setting, where each iteration consists of a policy update step at each state using a certain matrix game algorithm, and a value update step with a certain learning rate. This framework unifies many existing and new policy optimization algorithms. We show that the \emph{state-wise average policy} of this algorithm converges to an approximate Nash equilibrium (NE) of the game, as long as the matrix game algorithms achieve low weighted regret at each state, with respect to weights determined by the speed of the value updates. Next, we show that this framework instantiated with the Optimistic Follow-The-Regularized-Leader (OFTRL) algorithm at each state (and smooth value updates) can find an $\mathcal{\widetilde{O}}(T^{-5/6})$ approximate NE in $T$ iterations, and a similar algorithm with slightly modified value update rule achieves a faster $\mathcal{\widetilde{O}}(T^{-1})$ convergence rate. These improve over the current best $\mathcal{\widetilde{O}}(T^{-1/2})$ rate of symmetric policy optimization type algorithms. We also extend this algorithm to multi-player general-sum Markov Games and show an $\mathcal{\widetilde{O}}(T^{-3/4})$ convergence rate to Coarse Correlated Equilibria (CCE). Finally, we provide a numerical example to verify our theory and investigate the importance of smooth value updates, and find that using ``eage'' value updates instead (equivalent to the independent natural policy gradient algorithm) may significantly slow down the convergence, even on a simple game with $H=2$ layers.
\end{abstract}

\section{Introduction}


Policy optimization, i.e. algorithms that learn to make sequential decisions by local search on the agent's policy directly, is a widely used class of algorithms in reinforcement learning~\citep{peters2006policy,schulman2015trust,schulman2017proximal}.
Policy optimization algorithms are particularly advantageous in the multi-agent reinforcement learning (MARL) setting (e.g. compared with value-based counterparts), due to their typically lower representational cost and better scalability in both training and execution. A variety of policy optimization algorithms such as Independent PPO~\citep{de2020independent}, MAPPO~\citep{yu2021surprising}, QMix \citep{rashid2018qmix} have been proposed to solve real-world MARL problems~\citep{bard2020hanabi,mordatch2018emergence, samvelyan2019starcraft}. These algorithms share a same high-level structure with iterative \emph{value updates} (for certain value estimates) and \emph{policy updates} (often independently with each agent) using information from the value estimates and/or true rewards.





While policy optimization for MARL has been studied theoretically in a growing body of work, there are still gaps between algorithms used in practice and provably-efficient algorithms studied in theory---Algorithms in practice generally follow two natural design principles: \emph{symmetric updates} among all agents, and \emph{simultaneous learning} of values and policies~\citep{zhang2021multi,yu2021surprising}. By contrast, policy optimization algorithms studied in theory often diverge from these principles and incorporate some tweaks, such as (i) asymmetric updates, where one agent takes a much smaller learning rate than the others (two time-scale)~\citep{daskalakis2020independent} or waits until the other agents learn an approximate best response~\citep{zhao2021provably}; and (ii) batch-like learning, where policies are optimized to sufficient precision with respect to the current value estimate before the next value update~\citep{cen2021extra}. There is so far a lacking of systematic studies on the performance of the more vanilla policy optimization algorithms following the above two principles, even under the setting where full-information feedback from the game is available.



Towards bridging these gaps, this paper studies policy optimization algorithms for Markov games, with a focus on algorithms with symmetric updates and simultaneous learning of values and policies. Our contributions can be summarized as follows:
\begin{itemize}[leftmargin=1.5pc, topsep=0pt]
    \item We propose an algorithm framework for two-player zero-sum Markov games in the full-information setting (Section~\ref{section:framework}). This framework unifies many existing and new policy optimization algorithms such as Nash V-Learning, Gradient Descent/Ascent, as well as seemingly disparate algorithms such as Nash Q-Learning (Section~\ref{section:framework-examples}). We prove that the \emph{state-wise average policy} outputted by the above algorithm is an approximate Nash Equilibrium (NE), so long as suitable per-state weighted regrets are bounded (Section~\ref{section:framework-general}). This generic result can be instantiated in a modular fashion to derive convergence guarantees for the many examples above.
    \item We instantiate our framework to show that a new algorithm based on Optimistic Follow-The-Regularized-Leader (OFTRL) and smooth value updates finds an $\tO(T^{-5/6})$ approximate NE in $T$ iterations (Section~\ref{section:optimistic}). This improves over the current best rate of $\tO(T^{-1/2})$ achieved by symmetric policy optimization type algorithms. In addition, we also propose a slightly modified OFTRL algorithm that further improves the rate to $\tO(T^{-1})$, which matches with the known best rate for all policy optimization type algorithm.
    \item We additionally extend the above OFTRL algorithm to multi-player general-sum Markov games and show an $\tO(T^{-3/4})$ convergence rate to Coarse Correlated Equilibria (CCE), which is also the first rate faster than $\tO(T^{-1/2})$ for policy optimization in general-sum Markov games (Section~\ref{section:multi-player}).
    \item We perform simulations on a carefully constructed zero-sum Markov game with $H=2$ layers to verify our convergence guarantees. The numerical tests further suggest the importance of \emph{smooth value updates}: the Independent Natural Policy Gradient algorithm (as one instantiation of our algorithm framework with ``eager'' value updates) appears to converge much slower (Section~\ref{section:eager}).
\end{itemize}

\subsection{Related work}
\paragraph{Two-player zero-sum MGs} 
Markov games (MGs)~\citep{littman1994markov} (also known as Stochastic Games~\citep{shapley1953stochastic}) is a widely studied model for multi-agent reinforcement learning. In the most basic setting of two-player zero-sum MGs, algorithms for computing the NE have been extensively studied in both the full-information setting~\citep{littman2001friend,hu2003nash,hansen2013strategy} and the sample-based/online setting~\citep{brafman2002r,wei2017online, jia2019feature,sidford2020solving,zhang2020model,bai2020provable,xie2020learning, bai2020near,liu2021sharp,jin2021power,huang2021towards,yu2021provably,chen2022almost,liu2022learning}. Our algorithm framework incorporates (the full-information version of) several algorithms in this line of work.

\paragraph{Policy optimization for zero-sum MGs} Policy optimization for single-agent Markov Decision Processes has been extensively in a recent line of work, e.g.~\citep{agarwal2021theory,bhandari2019global,liu2019neural,shani2020adaptive,mei2020global,liu2020improved,cen2021,ding2020natural,xiao2022convergence} and the many references therein. For two-player zero-sum MGs, the Nash V-Learning algorithm of~\citet{bai2020near} (originally proposed for the sample-based online setting) can be viewed as an independent policy optimization algorithm, and can be adapted to the full-information setting with $\tO(T^{-1/2})$ convergence rate. \citet{daskalakis2020independent} prove that the independent policy gradient algorithm with an asymmetric two time-scale learning rate can learn the NE (for one player only) with polynomial iteration/sample complexity. \citet{zhao2021provably} show that another asymmetric algorithm that simulates a policy gradient/best response dynamics converges to NE with $\tO(1/T)$ rate. \citet{cen2021extra} use a symmetric optimistic (extragradient) subroutine for matrix games to learn zero-sum MGs in a \emph{layer-wise} fashion (more like a Value Iteration type algorithm), and also derive an $\tO(1/T)$ convergence rate. The closest to our work is \citet{wei2021last}~which proves that Optimistic Gradient Descent/Ascent (OGDA), combined with smooth value updates at all layers simultaneously, converges to an NE with $\tO(T^{-1/2})$ rate for both the average duality gap and the last iterate. The $\tO(T^{-5/6})$ rate of our OFTRL algorithm improves over~\citep{wei2021last} and is the first such faster rate for symmetric, policy optimization type algorithms. The $\tO(1/T)$ rate of our modified OFTRL algorithm matches with rates in \citep{zhao2021provably,cen2021}, while still maintaining symmetric update and simultaneous learning of values and policies.




\paragraph{Multi-player general-sum MGs}
A recent line of work shows that a generalization of the V-learning algorithm to the multi-player general-sum setting can learn Coarse Correlated Equilibria (CCE)~\citep{song2021can,jin2021v,mao2021provably} and Correlated Equilibria (CE)~\citep{song2021can,jin2021v}. The algorithmic designs in these works are specially tailored to the sample-based setting, where the best possible rate is $\tO(T^{-1/2})$.\footnote{Even if we specialize their algorithms to the full-information setting, the attained rates are no better than $\tO(T^{-1/2})$ because the bandit subroutines they deployed in V-learning converge no faster than ${\Omega}(T^{-1/2})$ even in the simplest setting of full-information matrix games.} In contrast, this paper considers the full-information setting and proposes new algorithms achieving the faster $\tO(T^{-3/4})$ for learning CCE in general-sum MGs. Another recent line of work considers learning NE in Markov Potential Games~\citep{zhang2021gradient,leonardos2021global,song2021can,ding2022independent,zhang2022logbarrier}, which can be seen as a cooperative-type subclass of general-sum MGs. 

\paragraph{Optimistic algorithms in normal-form games} Technically, our accelerated rates build on the recent line of work on faster rates for optimistic no-regret algorithms in normal-form games~\citep{syrgkanis2015fast,rakhlin2013optimization,chen2020hedging,daskalakis2021near}. 
Specifically, our $\tilde{\cO}(T^{-5/6})$ rate for two-player zero-sum MGs builds upon a first-order smoothness analysis of~\citep{chen2020hedging}, our improved $\tO(T^{-1})$ rate in the same setting (achieved by the modified OFTRL algorithm) leverages analysis of~\citep{rakhlin2013optimization,syrgkanis2015fast} on bounding the summed regret over the two players, and our $\tilde{\cO}(T^{-3/4})$ rate for multi-player general-sum MGs follows from the RVU-property~\citep[Definition 3]{syrgkanis2015fast}. Our incorporation of these techniques involves non-trivial new components such as \emph{weighted} first-order smoothness bounds and handling changing game rewards.

\vspace{-2.5pt}
\section{Preliminaries}
\label{section:prelim}
\vspace{-2.5pt}

We consider the tabular episodic (finite-horizon) two-player-zero-sum Markov games (MGs), which can be denoted as $\cM(H,\cS, \cA,\cB,\P, r)$, where $H$ is the horizon length; $\cS$ is the state space with $|\cS| = S$; $\cA, \cB$ are the action space of the \emph{max-player} and \emph{min-player} respectively, with $|\cA| = A, |\cB| = B$; $\P = \{\P_h\}_{h=1}^H$ is the transition probabilities, where each $\P_h(s'|s, a,b)$ gives the probability of transition to state $s'$ from state-action $(s,a,b)$; $r = \{r_h\}_{h=1}^H$ are the reward functions, such that $r_h(s,a,b)$ is reward\footnote{This assumes deterministic rewards; our results can be generalized directly to the case of stochastic rewards.} of the max-player and $-r_h(s,a,b)$ is the reward of the min-player at time step $h$ and state-action $(s,a,b)$.
In each episode, the MG starts with a deterministic initial state $s_1$. Then at each time step $1\le h\le H$, both players observes the state $s_h$, the max-player takes an action $a_h\in \cA$, and the min-player takes an action $b_h\in \cB$. Then, both players receive their rewards  $r_h(s_h,a_h,b_h)$ and $-r_h(s_h,a_h,b_h)$, respectively, and the system transits to the next state $s_{t+1}\sim \P_h(\cdot|s_h,a_h,b_h)$. 



\vspace{-6pt}
\paragraph{Policies \& value functions}
A (Markov) policy $\mu$ of the max-player is a collection of policies $\mu = \{\mu_h: \cS \to \Delta_\cA\}_{h=1}^H$, where each $\mu_h(\cdot|s_h)\in\Delta_\cA$ specifies the probability of taking action $a_h$ at $(h, s_h)$. Similarly, a (Markov) policy $\nu$ of the min-player is defined as $\nu=\{\nu_h:\cS\to\Delta_\cB\}$. For any policy $(\mu, \nu)$ (not necessarily Markov), we use $V_h^{\mu,\nu}: \cS\to\R$ and $ Q_h^{\mu,\nu}:\cS\times\cA\times\cB\to \R$ to denote the value function and Q-function at time step $h$, respectively, i.e.
\begin{talign}
    V_h^{\mu,\nu}(s)&:= \E_{\mu,\nu}\brac{\sum_{h=h'}^H r_{h'}(s_{h'}, a_{h'}, b_{h'})~|~s_h = s}, \\
    Q_h^{\mu,\nu}(s,a,b)&:= \E_{\mu,\nu}\brac{\sum_{h=h'}^H r_{h'}(s_{h'}, a_{h'}, b_{h'})~|~s_h = s, a_h = a, b_h = b}.
\end{talign}
For notational simplicity, we use the following abbreviation: $[\P_h V](s,a,b):= \E_{s'\sim\P_h(\cdot|s,a,b)}V(s')$ for any value function $V$. By definition of the value functions and Q-functions, we have the following Bellman equations
\begin{align*}
    Q_h^{\mu,\nu}(s,a,b)&= \paren{r_h + \P_hV_{h+1}^{\mu,\nu}} (s,a,b), \\
    V_h^{\mu,\nu}(s,a,b)&=\E_{a\sim\mu_h(\cdot|s),b\sim\nu_h(\cdot|s)}\brac{Q_h^{\mu,\nu}(s,a,b)} = \<Q_h^{\mu,\nu}(s,\cdot,\cdot), \mu(\cdot|s)\times\nu(\cdot|s)\>.
\end{align*}
The goal for the max-player is to maximize the value function, whereas the goal for the min-player is to minimize the value function. 


\paragraph{Best response \& Nash equilibrium}
For any Markov policy $\mu$ of a max-player, there exists a best response for the min-player, which can be taken as a Markov policy $\nu^\dagger(\mu)$ such that 
$V_h^{\mu,\nu^\dagger(\mu)}(s) = \inf_\nu V_h^{\mu,\nu}(s)$ for all $(s,h)\in \cS\times[H]$. For simplicity we define $V_h^{\mu,\dagger}:= V_h^{\mu,\nu^\dagger(\mu)}$. By symmetry, we can also define $\mu^\dagger(\nu)$ and $V_h^{\dagger,\nu}$. It is known (e.g.~\citep{filar2012competitive}) that there exist Markov policies $(\mu^\star, \nu^\star)$ that perform optimally against best responses. These policies are also equivalent to Nash Equilibria (NEs) of the game, where no player can gain by switching to a different policy unilaterally. It can also be shown that any NE $(\mu^\star,\nu^\star)$ satisfies the following minimax equation
\vspace{-5pt}
\begin{equation*}
  \textstyle
    \sup_\mu\inf_\nu V_h^{\mu,\nu}(s) = V_h^{\mu^\star,\nu^\star}(s) = \inf_\nu\sup_\mu V_h^{\mu,\nu}(s).
\end{equation*}
Thus, while the NE policy $(\mu^\star,\nu^\star)$ may not be unique, all of them share the same value functions, which we denote as $V_h^\star:= V_h^{\mu^\star,\nu^\star}$. The Q-function $Q_h^\star$ can be defined similarly. In this paper, our main goal is to find an approximate NE, which is formally defined below. 
\begin{definition}[$\eps$-approximate Nash Equilibrium]
For any $\eps\ge 0$, a policy $(\mu, \nu)$ is an $\epsilon$-approximate Nash Equilibrium ($\eps$-NE) if $
    \negap(\mu, \nu) \defeq V_1^{\dagger, \nu}(s_1) - V_1^{\mu, \dagger}(s_1) \le \eps.
$
\end{definition}
\vspace{-8pt}
\paragraph{Additional notation}
For any $(h,s)\in[H]\times\cS$ and Q function $Q_h:\cS\times\cA\times\cB\to\R$, we define shorthand $[(\mu_h)^\top Q_h\nu_h](s)\defeq \<Q_h(s,\cdot,\cdot), \mu_h(\cdot|s)\times\nu_h(\cdot|s)\>$ for any policy $(\mu,\nu)$. Similarly, we let $[Q_h\nu_h](s,\cdot)\defeq \E_{b\sim \nu_h(\cdot|s)}[Q_h(s,\cdot,b)]\in\R^A$, and $[Q_h^\top\mu_h](s,\cdot)\defeq \E_{a\sim \mu_h(\cdot|s)}[Q_h(s,a,\cdot)]\in\R^B$. We use $A\vee B\defeq \max\{A, B\}$.

\section{An algorithm framework for zero-sum Markov games}
\label{section:framework}





\vspace{-6pt}
We begin by presenting an algorithm framework that unifies many existing and new algorithms for two-player zero-sum Markov Games, and its performance guarantee that could be specialized to yield concrete convergence results for many specific algorithms.

Our algorithm framework, described in Algorithm~\ref{alg:framework}, consists of two main components: the \emph{policy update} step computing policies $(\mu^t, \nu^t)$, and the \emph{value update} step computing the Q estimate $Q_h^t$'s. 

\begin{algorithm}[t]
\caption{Algorithm framework for two-player zero-sum Markov Games}
\label{alg:framework}
\begin{algorithmic}[1]
\STATE \textbf{Require:} Learning rate $\set{\beta_t}_{t\ge 1}\subset[0,1]$ (with $\beta_1 = 1$); Algorithm $\matrixgamealg$.
\STATE \textbf{Initialize:} $Q_h^0(s,a,b)\setto H-h+1$ for all $(h,s,a,b)\in[H]\times\cS\times\cA\times\cB$.
\FOR{$t=1,\dots,T$}
\FOR{$h=H,\dots,1$}
\STATE \textbf{Policy update:} Update policies for all $s\in\cS$:
\begin{align*}
    (\mu_h^t(\cdot|s), \nu_h^t(\cdot|s)) \setto \matrixgamealg\paren{ \set{Q_h^{i}(s, \cdot, \cdot)}_{i=1}^{t-1}, 
    \set{\mu_h^{i}(\cdot|s)}_{i=1}^{t-1}, 
    \set{\nu_h^{i}(\cdot|s)}_{i=1}^{t-1}
    }.
\end{align*}
\label{line:policy-update}
\vspace{-0.7em}
\STATE \textbf{Value update:} Update Q-value for all $(s,a,b)\in\cS\times\cA\times\cB$:
\begin{align}\label{eq:Q-update}
    Q_h^{t}(s,a,b) \setto (1-\beta_t) Q_h^{t-1}(s,a,b) + \beta_t \paren{r_h + \P_h[ (\mu_{h+1}^t)^\top Q_{h+1}^t\nu_{h+1}^t ]} (s,a,b).
\end{align}
\label{line:value-update}
\vspace{-0.7em}
\ENDFOR
\ENDFOR
\STATE {\bf Output}: State-wise average policy $(\hat{\mu}^T, \hat{\nu}^T)$, with $\beta_T^t$ defined in \eqref{eq:def-beta-k-t}:
\vspace{-5pt}
\begin{equation}
  \label{equation:state-wise-avg-policy}
  \textstyle
    \widehat{\mu}_h^T(\cdot|s) \setto \sum_{t=1}^T\beta_T^t\mu_{h}^{t}(\cdot|s),\quad 
    \widehat{\nu}_h^T(\cdot|s)\setto\sum_{t=1}^T\beta_T^t \nu_h^{t}(\cdot|s)~~~\textrm{for all}~(h,s)\in[H]\times\cS.
\end{equation}
\end{algorithmic}
\end{algorithm}

\vspace{-6pt}
\paragraph{Policy update via matrix game algorithms}
In the policy update step (Line~\ref{line:policy-update}), for each $(h,s)$, the two players update policies $(\mu^t_h(\cdot|s), \nu^t_h(\cdot|s))$ at $(h,s)$ using some matrix game algorithm $\matrixgamealg$ which takes as input all past Q matrices and all past policies of both players. The $\matrixgamealg$ offers a flexible interface that allows many choices such as the matrix NE subroutine over the most recent Q matrix $\NE(Q^{t-1}_h(s,\cdot,\cdot))$, or any independent no-regret algorithm (for both players), such as Follow-The-Regularized-Leader (FTRL)~\eqref{eq:FTRL-V-learning} or projected Gradient Descent-Ascent~\eqref{eq:GDA-example} considered in the examples later.
%
\vspace{-6pt}
\paragraph{Value update with learning rate $\set{\beta_t}$}
For any $(h,s,a,b)$, the value update step (Line~\ref{line:value-update}) updates $Q_h^t(s,a,b)$ by the newest value function $r_h + \P_h[ (\mu_{h+1}^t)^\top Q_{h+1}^t\nu_{h+1}^t ]$ propagated from layer $h+1$, using a sequence of learning rates $\set{\beta_t}_{t\ge 1}$ which we assume to be within $[0,1]$ (with $\beta_1\defeq 1$). 
$\set{\beta_t}$ controls the speed of the value update, with two important special cases: 
\begin{enumerate}[label=(\arabic*), leftmargin=2em, topsep=0em]
\item \emph{Eager} value updates, where we set $\beta_t=1$ so that $Q_h^{t}$ performs \emph{policy evaluation} of the current policy $(\mu^t, \nu^t)$, that is, $Q_h^t = Q_h^{\mu^t,\nu^t}$.
\item \emph{Smooth} (incremental) value updates, where we choose $\beta_t \to 0$ as $t\to \infty$. In this case, the $Q_h^{t}$ moves slower (resembling a \emph{critic} in Actor-Critic like algorithms), and becomes a weighted average of all past updates. A standard choice that is frequently used is from ~\citep{jin2018q} (and many subsequent work),  \begin{align}\label{eq:def-alpha-t}
\textstyle
    \beta_t = \alpha_t := (H+1)/(H+t).
\end{align}
\end{enumerate}
For any $\set{\beta_t}$, the update~\eqref{eq:Q-update} implies that
\begin{equation*}\label{eq:Q-relationship}
    \textstyle
    Q_h^{t}(s,a,b) = \sum_{i=1}^t\beta_t^i \paren{ \brac{r_h + \P_h[\paren{\mu_{h+1}^{i}}^\top Q_{h+1}^{i}\nu_{h+1}^{i}] }(s,a,b) },
\end{equation*}
where $\beta_t^i$'s are a group of weights summing to one ($\sum_{i=1}^t\beta_t^i = 1$) defined as
\begin{equation}\label{eq:def-beta-k-t}
\textstyle
    \beta_t^t = \beta_t;\qquad \beta_t^i =\prod_{j=i+1}^{t}(1-\beta_j)\beta_i, ~~\textup{for~~} i\in[t-1].
\end{equation}
Note that with smooth value updates ($\beta_t<1$), $Q_h^t$ is not necessarily the Q-function of any policy. Upon finishing, the algorithm outputs the \emph{state-wise average policy} $(\hat{\mu}^T, \hat{\nu}^T)$ defined in~\eqref{equation:state-wise-avg-policy}, where each $\hat{\mu}^T_h(\cdot|s)$ is the weighted average of $\mu_h^t(\cdot|s)$ using weights $\set{\beta_T^t}_{t=1}^T$ (and similarly for $\hat{\nu}^T$).

\paragraph{Symmetric \& simultaneous learning, (de)centralization}
We remark that Algorithm~\ref{alg:framework} by definition performs simultaneous learning (of policies and values) at all layers, and also yields symmetric (policy) updates if $\matrixgamealg$ is a symmetric algorithm with respect to $\mu$ and $\nu$. Also, although Algorithm~\ref{alg:framework} appears to be a \emph{centralized} algorithm as it maintains Q values in~\eqref{eq:Q-update}, this does not preclude possibilities that the algorithm can be executed in a \emph{decentralized} fashion. This can happen e.g. when the Q-update~\eqref{eq:Q-update} can be rewritten as an equivalent V-update (cf. Example~\ref{example:Nash-V} \&~\ref{example:GDA-critic}).

\subsection{Theoretical guarantee}
\label{section:framework-general}


We are now ready to state the main theoretical guarantee of Algorithm~\ref{alg:framework}, which states that the $(\widehat{\mu}^T, \widehat{\nu}^T)$ is an approximate NE, as long as the algorithm achieves low \emph{per-state weighted regrets} w.r.t. weights $\{\beta_t^i\}_{i=1}^t$, defined as
\begin{equation}\label{eq:def-regret}
\begin{split}
    & \textstyle \reg_{h,\mu}^t(s) \defeq \max_{\mu^\dagger\in\Delta_\cA} \sum_{i=1}^t \beta_t^i \<\mu^\dagger - \mu_h^i(\cdot|s), \brac{Q_h^i\nu_h^i}(s, \cdot)\>, \\
    & \textstyle \reg_{h,\nu}^t(s) \defeq \max_{\nu^\dagger\in\Delta_\cB} \sum_{i=1}^t \beta_t^i \< \nu_h^i(\cdot|s) - \nu^\dagger, \brac{(Q_h^i)^\top\mu_h^i}(s, \cdot) \>,  \\
    & \textstyle \reg_h^t \defeq \max_{s\in\cS} \max\{\reg_{h,\mu}^t(s), \reg_{h,\nu}^t(s)\}.
\end{split}
\end{equation}





\begin{theorem}[Main guarantee of Algorithm~\ref{alg:framework}]
\label{theorem:master}
Suppose that the per-state regrets can be upper-bounded as $\reg_h^t\le \barreg_h^t$ for all $(h,t)\in[H]\times[T]$, where $\barreg_h^t$ is non-increasing in $t$: $\barreg_h^t\ge \barreg_h^{t+1}$ for all $t\ge 1$. Then, the output policy $(\hat{\mu}^T, \hat{\nu}^T)$ of Algorithm \ref{alg:framework} satisfies
\begin{equation}
\label{eq:NE-gap}
    \negap(\hat{\mu}^T,\hat{\nu}^T)\le C\brac{H\max_{h\in[H]} \barreg_h^T + H^2\cbeta^H\log T \cdot \frac{1}{T} \sum_{t=1}^T \max_{h\in[H]}\barreg_h^t}
\end{equation}
for all $T\ge 2$ and some absolute constant $C>0$, where $\cbeta$ is a constant depending on $\set{\beta_t}_{t\ge 1}$:
\begin{equation}\label{eq:def-c}
    \textstyle
    \cbeta \defeq \sup_{j\ge 1} \sum_{t=j}^{\infty}\beta_t^j \ge 1.
\end{equation}
Specifically, $\cbeta = \paren{1+\frac1H}$ if $\beta_t=\alpha_t = \frac{H+1}{H+t}$, and $\cbeta = 1$ if $\beta_t = 1$.
\end{theorem}


Bound~\eqref{eq:NE-gap} is typically dominated by the second term on the right hand side, suggesting that the $\negap$ can be bounded by the average weighted regret $\widetilde{\cO}\paren{\frac{1}{T} \sum_{t=1}^T \max_h\barreg_h^t}$, if $\cbeta^H=O(1)$. Theorem~\ref{theorem:master} serves as a modular tool for analyzing a broad class of algorithms: As long as this average regret is sublinear in $T$ (including---but not limited to---choosing $\matrixgamealg$ as uncoupled no-regret algorithms), the output policy will be an approximate NE.
We emphasize though that this result is not yet end-to-end, as each $\reg_h^t$ is a weighted regret w.r.t. \emph{the particular set of weights} $\set{\beta_t^i}_{i=1}^t$, minimizing which may require careful algorithm designs and/or case-by-case analyses. We provide some concrete examples in Section \ref{section:framework-examples} to demonstrate the usefulness of Theorem~\ref{theorem:master}. 

We remark that the state-wise average policy considered in Theorem~\ref{theorem:master} is an average policy that is also \emph{Markovian} by definition, which is different from existing work which considers either the (Markovian) last iterate~\citep{wei2021last} or non-Markovian average policies (e.g.~\citep{bai2020near}). However, this guarantee relies on full-information feedback (so that per-state regret bounds are available), and it remains an open question how such guarantees could be generalized to sample-based settings.




\paragraph{Proof overview} 
The proof of Theorem~\ref{theorem:master} follows by (1) bounding $\negap(\hat{\mu}^T,\hat{\nu}^T)$ in terms of per-state regrets w.r.t. the \emph{Nash value functions} $Q_h^\star$'s by performance difference arguments (Lemma \ref{lem.duality}); (2) recursively bounding the value estimation error $\delta_h^t:=\linf{Q_h^t-Q_h^\star}$ (Lemma \ref{lemma:recursion-of-value-estimation}) which yields the constant $\cbeta$; and (3) combining the above to translate the regret from $Q_h^\star$'s to $Q_h^t$'s (which we assume to be bounded by $\barreg_h^t$) and obtain the theorem. The full proof can be found in Appendix~\ref{apdx:theorem-master}. 





\subsection{Examples}
\label{section:framework-examples}

We now demonstrate the generality of Algorithm~\ref{alg:framework} and Theorem~\ref{theorem:master} by showing that they subsume many existing algorithms (and yield new algorithms) for two-player-zero-sum Markov games, and provide new guarantees with the particular output policy~\eqref{equation:state-wise-avg-policy}.




\begin{example}[Nash V-Learning \citep{bai2020near}, full-information version]\label{example:Nash-V}
The full algorithm (Algorithm \ref{alg:Nash-V-learning}) can be found in Appendix \ref{appendix:nash-v-learning}. The algorithm is a special case of Algorithm \ref{alg:framework} with $\beta_t=\alpha_t=(H+1)/(H+t)$, and $\matrixgamealg$ chosen as the weighted FTRL algorithm
\begin{equation}\label{eq:FTRL-V-learning}
    \mu_h^t(a|s)\! \propto_a\! \exp\!\paren{\!\frac{\eta}{w_{t\!-\!1}}\!\sum_{i=1}^{t-1}\!w_i\brac{Q_h^{i}\nu_h^i}(s,a)\!},~
    \nu_h^t(b|s)\!\propto_b\! \exp\!\paren{\! - \frac{\eta}{w_{t\!-\!1}}\!\sum_{i=1}^{t-1}\!w_i \brac{\paren{Q_h^{i}}^{\!\top}\!\! \mu_h^i}(s,b)\!},
\end{equation}
where $w_t \defeq {\alpha_t^t}/{\alpha_t^1}$. Combining Theorem~\ref{theorem:master} with the standard regret bound of weighted FTRL, this algorithm achieves $\negap(\widehat{\mu}^T,\widehat{\nu}^T)\le \tO(H^{7/2}/\sqrt{T})$ choosing $\eta\asymp 1/\sqrt{T}$ (Proposition \ref{prop:Nash-V-learning}).

Additionally, although the original Nash V-learning algorithm~\citep{bai2020near} updates the V values (which makes the algorithm implementable in a decentralized fashion) instead of the Q values used in Algorithm~\ref{alg:framework}, these two forms are actually equivalent in the full-information setting (Proposition~\ref{prop:equivalence-NashV-framework}).
\end{example}

Compared with the $\tO(\sqrt{H^5S\max\{A,B\}/T})$ guarantee of (the non-Markovian output policy of) Nash V-Learning in the sample-based online setting~\citep{bai2020near,tian2021online,jin2021v}, our rate achieves better (logarithmic) $S,A,B$ dependence due to our full-information setting, and worse $H$ dependence which happens as our output policy is the (Markovian) state-wise average policies, whose guarantee (Theorem~\ref{theorem:master}) follows from a different analysis. 


\begin{example}[GDA-Critic]\label{example:GDA-critic}
This algorithm is a special case of Algorithm~\ref{alg:framework} with $\beta_t=\alpha_t=(H+1)/(H+t)$, and $\matrixgamealg$ as projected gradient descent/ascent (GDA), i.e., 
\begin{equation}\label{eq:GDA-example}
    \mu_h^t(\cdot|s) \!\setto\! \mathcal{P}_{\!\Delta_{\cA}}\!\paren{\mu_h^{t\!-\!1}(\cdot|s) \!+\! \eta \brac{Q_h^{t\!-\!1}\nu_h^{t\!-\!1}}(s)},~~
  \nu_h^{t}(\cdot|s)\!\setto\! \mathcal{P}_{\!\Delta_{\cB}}\!\paren{\nu_h^{t\!-\!1}(\cdot|s) \!-\! \eta (\brac{Q_h^{t\!-\!1})^{\!\!\top}\! \mu_h^{t\!-\!1}}(s)}.
\end{equation}
Similar as Nash V-Learning, GDA-Critic also admits an equivalent form with V value updates (full description in Algorithm~\ref{alg:gda-critic}). As GDA achieves weighted regret bounds with any monotone weights including $\set{\alpha_t^i}_{i=1}^t$ (Lemma~\ref{lem:gd.regret}), we can invoke Theorem~\ref{theorem:master} to show that this algorithm achieves $\negap(\widehat{\mu}^T,\widehat{\nu}^T)\le \tO(H^{7/2}({A\vee B})^{1/2}/\sqrt{T})$ if we choose $\eta\asymp 1/\sqrt{T}$ (Proposition~\ref{prop:GDA-critic}).

The GDA-critic algorithm is also similar to the OGDA-MG algorithm of~\citet{wei2021last}, except that we use the (non-optimistic) vanilla version of GDA. To our best knowledge, the above algorithm and guarantee are not known. We remark that even ignoring difference between GDA and OGDA, the above guarantee cannot be obtained by direct adaptation of the results of~\citep{wei2021last} which focus on either the average duality gap and/or last-iterate convergence.
\end{example}


Besides the above examples, Algorithm \ref{alg:framework} also incorporates the following algorithms which are typically not categorized as policy optimization algorithms.

\begin{example}[Nash Q-Learning \citep{hu2003nash,bai2020near}, full-information version]\label{example:Nash-Q} This algorithm is a special case of Algorithm \ref{alg:framework} with $\beta_t=\alpha_t=(H+1)/(H+t)$ and $\matrixgamealg$ as the matrix Nash subroutine
\begin{align*}
    (\mu_h^t(\cdot|s), \nu_h^t(\cdot|s)) \setto \NE(Q_h^{t-1}(s, \cdot, \cdot)) \defeq \arg\paren{\min_{\mu\in\Delta_\cA} \max_{\nu\in\Delta_\cB} \mu^\top Q_h^{t-1}(s,\cdot,\cdot)\nu}.
\end{align*}
(Full description in Algorithm~\ref{alg:Nash-Q}.)
Although $\NE(Q_h^{t-1}(s,\cdot,\cdot))$ is not by default a no-regret algorithm, using the fact that $\linf{Q_h^t-Q_h^{t-1}}$ is small (due to the small $\alpha_t$) we can show that it is close to a (hypothetical) ``Be-The-Leader'' style algorithm that computes the matrix NE of the current Q matrix $Q_h^t$ which achieves $\le 0$ regret (Lemma~\ref{lemma:reg-Nash-Q}). Combining this with Theorem~\ref{theorem:master} shows that this algorithm achieves $\negap(\widehat{\mu}^T,\widehat{\nu}^T)\le \tO(H^4/T)$ (Proposition~\ref{prop:Nash-Q}).
\end{example}

\begin{example}[Nash Policy Iteration (Nash-PI)]\label{example:Nash-PI}
This classical algorithm (Algorithm~\ref{alg:Nash-PI}) performs iterative policy evaluation and policy improvement (also similar to Nash Value Iteration~\citep{shapley1953stochastic,bai2020provable,liu2021sharp}): 
\begin{align}
\label{equation:nash-pi}
    (\mu^{t+1}_h(\cdot|s), \nu^{t+1}_h(\cdot|s)) \setto \NE(Q^{\mu^t, \nu^t}_h(s, \cdot, \cdot)).
\end{align}
This is also a special case of Algorithm \ref{alg:framework} with
$\beta_t=1$ and $\matrixgamealg$ set as $\NE$.
It is a standard result that this algorithm converges exactly (achieving zero NE gap) in $H$ steps, and this fact can be obtained using our framework as well (Proposition~\ref{prop:Nash-PI}).
\end{example}




\section{Fast convergence of optimistic FTRL}
\label{section:optimistic}


In this section, we instantiate Algorithm~\ref{alg:framework} by choosing $\matrixgamealg$ as the Optimistic Follow-The-Regularized-Leader (OFTRL) algorithm. OFTRL is also an uncoupled no-regret algorithm that is known to enjoy faster convergence than standard FTRL under additional loss smoothness assumptions~\citep{rakhlin2013optimization,syrgkanis2015fast, chen2020hedging,daskalakis2021near}. We show that, using OFTRL, Algorithm~\ref{alg:framework} enjoys faster convergence than the $\tO(1/\sqrt{T})$ rate of using FTRL or GDA (cf. Example~\ref{example:Nash-V} \&~\ref{example:GDA-critic}). 

Concretely, we use the following weighted OFTRL algorithm at each $(h,s,t)$:
\begin{align}
\label{equation:oftrl-munu}
 \begin{aligned}
    & \textstyle \mu^{t}_h(a | s) \propto_a \exp\paren{ (\eta/w_t) \cdot \brac{ \sum_{i=1}^{t-1} w_i(Q_h^i\nu_h^i)(s, a) + w_{t-1}(Q_h^{t-1}\nu_h^{t-1})(s, a) } }, \\
    & \textstyle \nu^{t}_h(b | s) \propto_b \exp\paren{ -(\eta/w_t) \cdot \brac{ \sum_{i=1}^{t-1} w_i((Q_h^i)^\top\mu_h^i)(s, b) + w_{t-1}((Q_h^{t-1})^\top\mu_h^{t-1})(s, b) } },
\end{aligned}
\end{align}
where $w_t$ is the same weights as defined in Example \ref{example:Nash-V}, and we choose $\beta_t=\alpha_t=(H+1)/(H+t)$.




\begin{theorem}[Fast convergence of OFTRL in zero-sum Markov Games]
\label{theorem:oftrl-main}
Suppose Algorithm~\ref{alg:framework} is instantiated with $\beta_t=\alpha_t=(H+1)/(H+t)$ and $\matrixgamealg$ to be the OFTRL algorithm~\eqref{equation:oftrl-munu} with any $\eta\le 1/H$ (full description in Algorithm~\ref{algorithm:oftrl-mg}). Then the per-state regret can be bounded as follows for some absolute constant $C>0$:
\begin{align}
  \label{equation:oftrl-main-per-state-regret}
    \reg_h^t \le \barreg_h^t \defeq C \brac{ \frac{H^2\log(A\vee B)}{\eta t} + \eta^5H^6 }~~~\textrm{for all}~(h,t)\in[H]\times[T].
\end{align}
Further, choosing $\eta={\rm poly}(H, \log(A\vee B), \log T)\cdot T^{-1/6}$, the output (state-wise average) policy $(\hat{\mu}^T, \hat{\nu}^T)$ achieves approximate NE guarantee
\begin{align}
  \label{equation:oftrl-main-policy-guarantee}
  \textstyle
    \negap(\hat{\mu}^T, \hat{\nu}^T) \le \cO\paren{{\rm poly}(H, \log(A\vee B), \log T)\cdot T^{-5/6}}.
\end{align}
\end{theorem}
To our best knowledge, the $\tO(T^{-5/6})$ rate asserted in Theorem~\ref{theorem:oftrl-main} is the first rate faster than the standard $\tO(1/\sqrt{T})$ for symmetric, policy optimization type algorithms in two-player zero-sum Markov games. The closest existing result to this is of~\citet{wei2021last}, who analyze the OGDA algorithm with smooth value updates and show a $\tO(1/\sqrt{T})$ convergence of both the average $\negap$ and the $\negap$ of the last-iterate. However, these only imply at most a $\tO(1/\sqrt{T})$ rate for the average policies, and not our faster rate\footnote{See also~\citep{golowich2020last} for another example where last-iterates are \emph{provably} slower than averages.}. 
\citet{cen2021extra,zhao2021provably} show $\tO(1/T)$ convergence of policy optimization-\emph{like} algorithms with optimistic subroutines, which are however very different styles of algorithms that either performs \emph{layer-wise learning} similar as Value Iteration (the matrix games at each state are learned to sufficient precision before the backup)~\citep{cen2021extra}, or uses strongly \emph{asymmetric} updates that simulate a policy gradient-best response dynamics~\citep{zhao2021provably}. By contrast, our Algorithm~\ref{algorithm:oftrl-mg} (as well as its modified version in Algorithm~\ref{algorithm:modified-oftrl-q} with $\tO(T^{-1})$ rate) runs symmetric no-regret dynamics for both players, simultaneously at all layers.




 

\paragraph{Proof overview} The proof of Theorem \ref{theorem:oftrl-main} (deferred to Appendix~\ref{appendix:proof-oftrl-main}) builds upon the recent line of work on fast convergence of optimistic algorithms~\citep{rakhlin2013optimization,syrgkanis2015fast,chen2020hedging}, in particular the work of~\citet{chen2020hedging} which shows an $\tO(T^{-5/6})$ convergence rate of OFTRL for two-player normal-form games. Our regret bound~\eqref{equation:oftrl-main-per-state-regret} generalizes this result non-trivially by additionally handling (1) The \emph{weighted} regret, which requires bounding the weighted stability of the OFTRL iterates by a new analysis of the potential functions (Lemma~\ref{lemma:oftrl-stability}), and (2) The errors induced by \emph{changing game matrices}, as $Q_h^t(s,\cdot,\cdot)$ changes over $t$. Plugging~\eqref{equation:oftrl-main-per-state-regret} into Theorem~\ref{theorem:master} yields the policy guarantee~\eqref{equation:oftrl-main-policy-guarantee}.


\paragraph{Modified OFTRL algorithm with $\tO(T^{-1})$ rate}

We further slightly modify Algorithm~\ref{algorithm:oftrl-mg} to design a new OFTRL style algorithm with $\tO(T^{-1})$ convergence rate (Algorithm \ref{algorithm:modified-oftrl-q} and Theorem~\ref{theorem:modified-alg}), which improves over the $\tO(T^{-5/6})$ of Theorem~\ref{theorem:oftrl-main} and matches the known best convergence rate for policy optimization type algorithms in two-player zero-sum Markov games. Algorithm~\ref{algorithm:modified-oftrl-q} still uses OFTRL in its policy update step, and the main difference from Algorithm~\ref{algorithm:oftrl-mg} is in its value update step: Rather than maintaining a single $Q_h^t$, the two players now each maintain their own value estimate $\up Q_h^t$, $\low Q_h^t$ which are still updated in an incremental fashion similar to (though not strictly speaking an instantiation of) the update rule \eqref{eq:Q-update} in our main algorithm framework. Details of the algorithm as well as the proofs are deferred to Appendix \ref{apdx:modified-OFTRL}.

\subsection{Extension to multi-player general-sum Markov games}
\label{section:multi-player}



Our fast convergence result can be extended to the more general setting of \emph{multi-player general-sum} Markov games. Concretely, we consider general-sum Markov games with $m\ge 2$ players, $S$ states, $H$ steps, where the $i$-th player has action space $\cA_i$ with $\Amax\defeq\max_{i\in[m]}|\cA_i|$ and her own reward function. The goal is to find a correlated policy over all players that is an approximate Coarse Correlated Equilibrium (CCE) of the game (see Appendix~\ref{appendix:general-sum-prelim} for the detailed setup).


We show that the OFTRL algorithm works for general-sum Markov games as well, with a fast $\tO(T^{-3/4})$ convergence to CCE. The formal statement and proof is in Theorem~\ref{theorem:oftrl-general-sum-formal} \& Appendix~\ref{appendix:proof-oftrl-general-sum-formal}.
\begin{theorem}[Fast convergence of OFTRL in general-sum Markov Games; Informal version of Theorem~\ref{theorem:oftrl-general-sum-formal}]
\label{theorem:oftrl-general-sum}
For $m$-player general-sum Markov Games, running the OFTRL algorithm (Algorithm~\ref{algorithm:oftrl-general-sum-mg}) for $T$ rounds, the output (correlated) policy $\hat{\pi}$ is an $\eps$-approximate CCE, where
\begin{align*}
  \textstyle
    \eps \le \cO\paren{ {\rm poly}(H, \log\Amax, \log T) \cdot (m-1)^{1/2} \cdot T^{-3/4} }.
\end{align*}
\end{theorem}
A baseline result for this problem would be $\tO(T^{-1/2})$, which may be obtained directly by adapting existing proofs of the V-Learning algorithm~\citep{song2021can,jin2021v} to the full-information setting. Our Theorem~\ref{theorem:oftrl-general-sum} shows that a faster $\tO(T^{-3/4})$ rate is available by using the OFTRL algorithm, which to our best knowledge is the first such result for policy optimization in general-sum Markov games. We also remark that the output policy $\hat{\pi}$ above is not a state-wise average policy as in the zero-sum setting, but rather a mixture policy that is in general non-Markov (cf. Algorithm~\ref{algorithm:certified-policy}), which is similar as (and slightly simpler than) the ``certified policies'' used in existing work~\citep{bai2020near,song2021can,jin2021v}. The proof of Theorem~\ref{theorem:oftrl-general-sum} builds upon the RVU property of OFTRL~\citep{syrgkanis2015fast} and additionally handles changing game rewards, similar as in Theorem~\ref{theorem:oftrl-main}. A proof sketch and comparison with the $\tO(T^{-5/6})$ analysis of the zero-sum case can be found in Appendix~\ref{appendix:general-sum-formal}.





\section{Simulations}
\label{section:eager}


We perform numerical studies on the various policy optimization algorithms. Our goal is two-fold: (1) Verify the convergence guarantees in our theorems and examples; (2) Test some other important special cases of Algorithm~\ref{alg:framework} that may not yet admit a provable guarantee. 

To this end, we consider three algorithms covered by the framework in Algorithm~\ref{alg:framework}:
\begin{enumerate}[leftmargin=2em, topsep=0em, itemsep=0em]
    \item {\bf FTRL} (Nash V-Learning) with smooth value updates $\beta_t=\alpha_t$ (Example~\ref{example:Nash-V} \& Algorithm~\ref{alg:Nash-V-learning}). Here the output policy $(\hat{\mu}^T, \hat{\nu}^T)$ are the state-wise averages with weights $\set{\alpha_T^i}_{i=1}^T$, and achieves $\negap(\hat{\mu}^T, \hat{\nu}^T)\lesssim T^{-1/2}$ if we choose $\eta\asymp T^{-1/2}$ (Proposition~\ref{prop:Nash-V-learning}).
    \item {\bf OFTRL} with smooth value updates $\beta_t=\alpha_t$ (Algorithm~\ref{algorithm:oftrl-mg}). Here the output policy $(\hat{\mu}^T, \hat{\nu}^T)$ are the state-wise averages with weights $\set{\alpha_T^i}_{i=1}^T$, and achieves $\negap(\hat{\mu}^T, \hat{\nu}^T)\lesssim T^{-5/6}$ if we choose $\eta\asymp T^{-1/6}$ (Theorem~\ref{theorem:oftrl-main}). We also consider the more aggressive choice $\eta=1$.
    \item {\bf INPG} (Independent Natural Policy Gradients). This algorithm is an instantiation of Algorithm~\ref{alg:framework} (cf. Appendix~\ref{apdx:exp-theoretical} for formal justifications) with \emph{eager} value updates ($\beta_t=1$), and $\matrixgamealg$ chosen as standard unweighted FTRL (a.k.a. Hedge) for all $(h,s,t)$:
    \begin{equation*}
        \mu_h^t(a|s) \propto_a \mu_h^{t-1}(a|s)\exp\!\paren{\eta\brac{Q_h^{t-1}\nu_h^{t-1}}(s)},~~
        \nu_h^t(b|s)\propto_b \nu_h^{t-1}(b|s)\exp\!\paren{ - \eta  \brac{\paren{Q_h^{t-1}}^{\!\top}\!\! \mu_h^{t-1}}(s)}.
    \end{equation*}
    For this algorithm, we choose two standard learning rates: $\eta=1$, and $\eta=T^{-1/2}$, and use the \emph{vanilla (state-wise) average} as the output policies (since the last-iterate is known to be cyclic):
    \begin{equation*}
    \textstyle
        \hat{\mu}_h^T(\cdot|s) = \frac{1}{T}\sum_{t=1}^T\mu_h^t(\cdot|s),~~\hat{\nu}_h^T(\cdot|s) = \frac{1}{T}\sum_{t=1}^T\nu_h^t(\cdot|s)~~~\textrm{for all}~(h,s)\in[H]\times\cS.
    \end{equation*}
\end{enumerate}
The main motivation for considering INPG is that it is a natural generalization of both the widely-studied NPG algorithm for single-agent RL, and the standard Hedge algorithm for zero-sum matrix games. In both cases the algorithm admits favorable convergence guarantees: NPG converges with rate $\cO(T^{-1})$~\citep{agarwal2021theory,Khodadadian2021,Mei2021,cen2021} (in both last iterate and averaging) using $\eta=O(1)$; Hedge converges with rate $\cO(T^{-1/2})$ in zero-sum matrix games (e.g.~\citep{rakhlin2013optimization}) using $\eta\asymp T^{-1/2}$. However, to our best knowledge, the convergence of INPG for zero-sum Markov games is unclear, and it is commented by \citet[Section 5]{wei2021last} that eager value updates ($\beta_t=1$) could cause the value function of the $(h+1)$th layer to oscillate, which make learning unstable or even biased within the $h$-th layer.



\paragraph{A two-layer numerical example}
We design a simple zero-sum Markov game with two layers and small state/action spaces ($H=2$, $S=4$, $A=2$; see Appendix~\ref{apdx:details} for the detailed description). The main feature of this game is that the reward in the first layer is much lower magnitude than that of the second layer (the scale is roughly $|r_1(s,\cdot,\cdot)|\approx 0.1 |r_2(s,\cdot,\cdot)|$), which may exaggerate the aforementioned unstable effect. We also choose a careful initialization $(\mu^1, \nu^1)$ which is non-uniform (and modify the FTRL / OFTRL algorithms to start at this initialization, cf. Appendix~\ref{apdx:details}) but with all entries bounded in $[0.15, 0.85]$. We test all three algorithms above on this game, with this initialization, $T\in\{10^3, 3\times 10^3, 10^4, \dots, 10^7\}$, and $\eta$ chosen correspondingly as described above.

\vspace{-0.5em}
\paragraph{Results}
Figure~\ref{fig:overall} plots the $\negap$ of the final output policies, one for each \{algorithm, $(T, \eta)$\}. Observe that FTRL converges with rate roughly $T^{-.570}\lesssim T^{-1/2}$, and OFTRL with $\eta=T^{-1/6}$ converges with rate $T^{-.835}\approx T^{-5/6}$, both corroborating our theory. Further, OFTRL with $\eta=1$ appears to converge with rate $T^{-1}$; showing this may be an interesting open theoretical question. 

On the other hand, the INPG algorithm appears to be much slower: The $\eta=1$ version does not seem to converge, whereas the convergence of $\eta=T^{-1/2}$ version is not clear but at least substantially slower than $T^{-1/2}$  ($T^{-.308}$ given by the linear fit) .

To further understand the behavior of INPG, we visualize its \emph{layer-wise} $\negap$'s for $h\in\{1,2\}$ (on our example), defined as the $\negap$ of the $h$-th layer's policies with respect to $Q^\star_h$:
\begin{align*}
\textstyle
    {\textup{NEGap-Layer-}}h(\mu,\nu):= \max_{s} \paren{ \max_{\mu_h^\dagger} \brac{(\mu_h^\dagger)^\top Q^\star_h \nu_h}(s) -  \min_{\nu_h^\dagger} \brac{\mu_h^\top Q^\star_h \nu_h^\dagger}(s) },~h \!= \!1,2.
\end{align*}
Note that \NElayerone~is a lower bound of $\negap(\mu, \nu)$ (cf. Appendix~\ref{apdx:exp-theoretical}) and thus needs to be minimized by any convergent algorithm. By contrast, on our example, \NElayertwo~is concerned with the last layer only, and can be minimized by any algorithm that works on matrix games.


Figure~\ref{fig:layer2} \&~\ref{fig:layer1} plot the layer-wise $\negap$'s of INPG against FTRL, on the single run with $T=10^7$ and $\eta=T^{-1/2}$. As expected, the \NElayertwo~converges nicely for both algorithms with similar rates (Figure~\ref{fig:layer2}) albeit the oscillation of INPG, whereas their behavior on \NElayerone~is drastically different: FTRL still converges, whereas INPG seems to be oscillating around a non-zero bias (Figure~\ref{fig:layer1}). This suggests that INPG may indeed be suffer from a non-vanishing bias in the first layer caused by the second layer's learning dynamics. (See Appendix~\ref{apdx:additional} for additional illustrations.) It would be an interesting open question to investigate the convergence of INPG theoretically.

\begin{figure}[t]
\centering
\begin{minipage}{.32\textwidth}
\centering
    \subcaption{{\footnotesize Overall $\negap$}}\label{fig:overall}
    \vspace{-.7em}
    \includegraphics[width=0.95\textwidth]{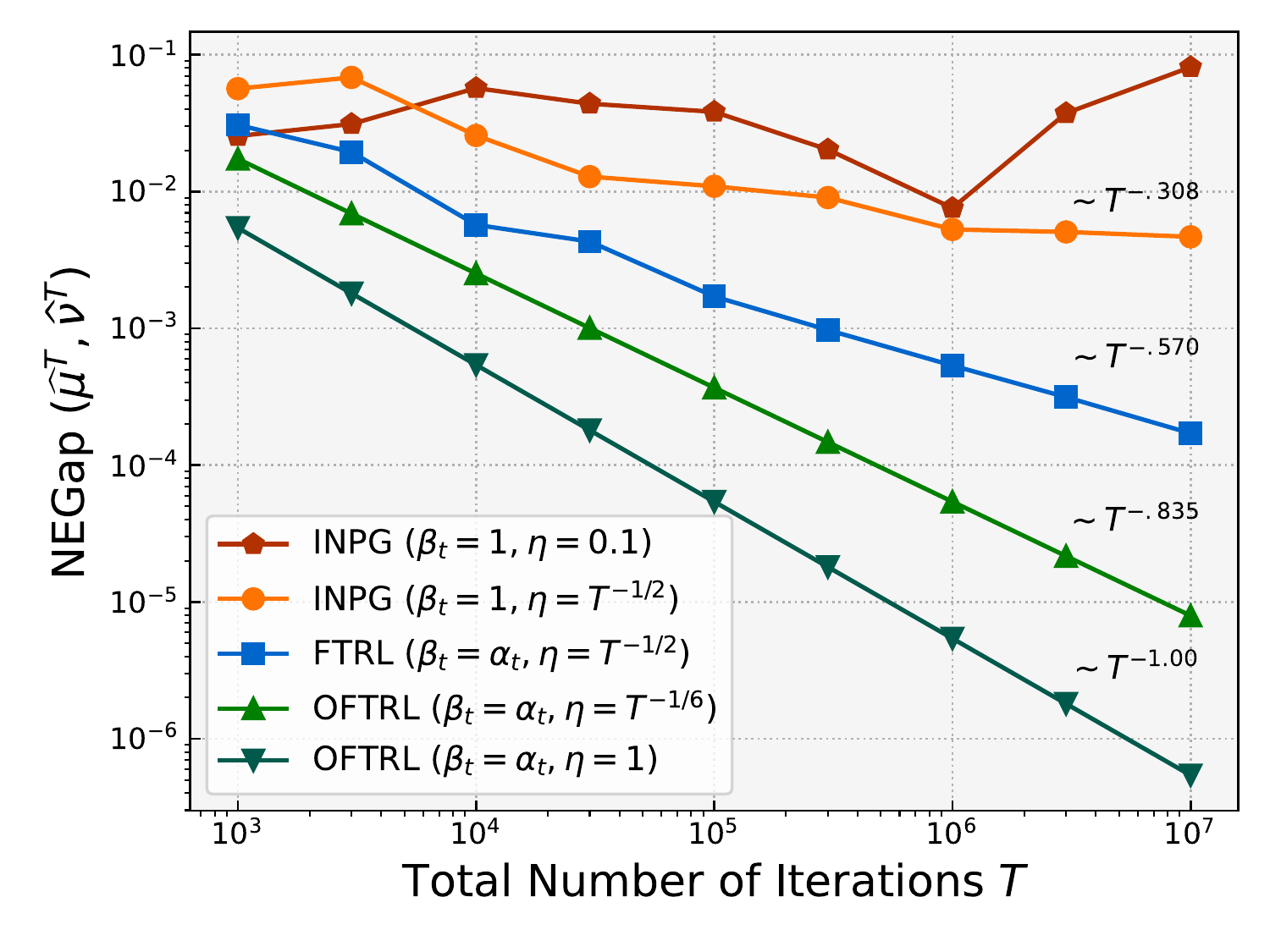}
\end{minipage}
\begin{minipage}{.32\textwidth}
    \centering
    \subcaption{{\footnotesize $\negap$ on layer $h=2$}}\label{fig:layer2}
    \vspace{-.7em}
    \includegraphics[width=.95\textwidth]{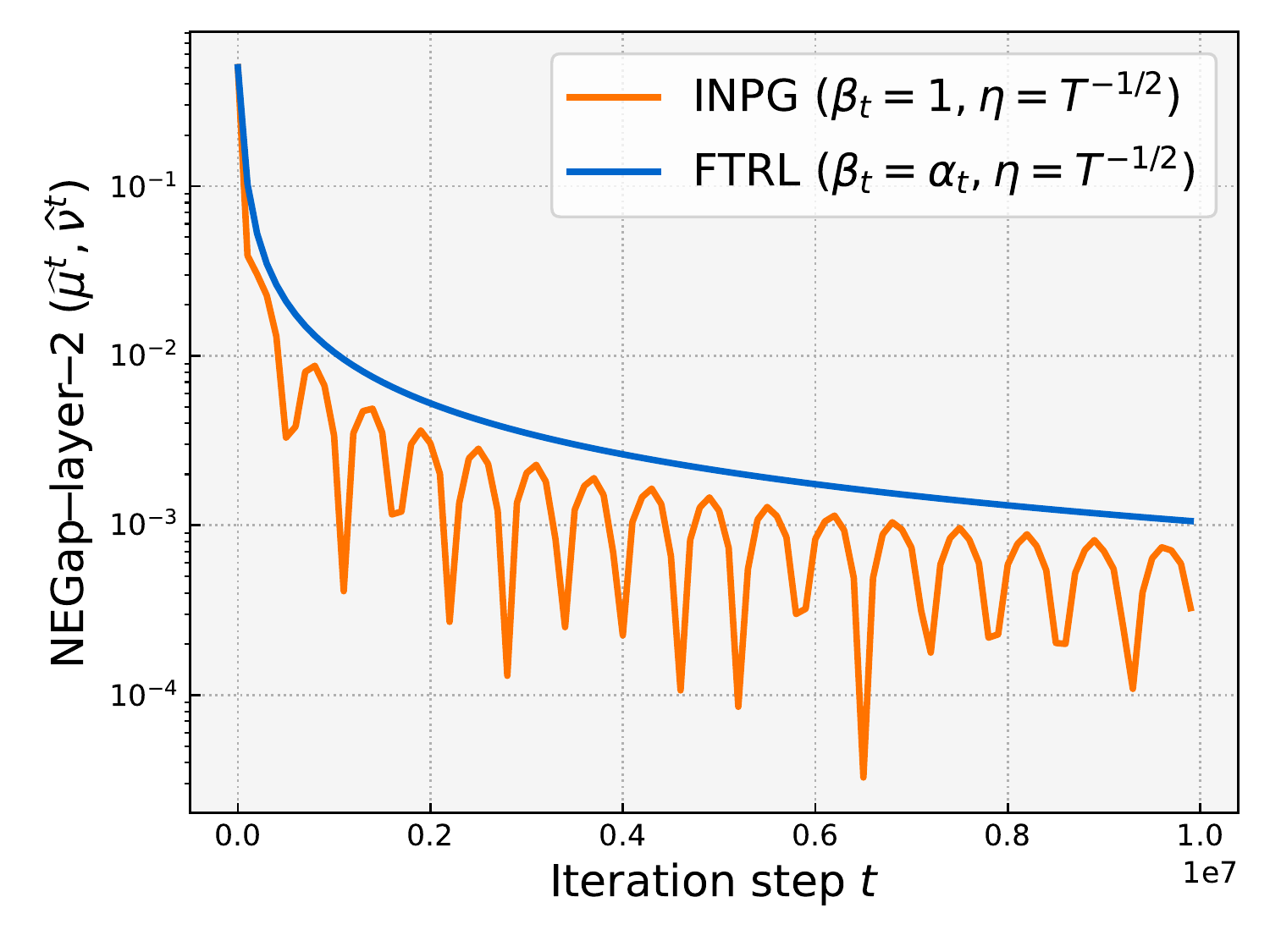}
\end{minipage}
\begin{minipage}{.32\textwidth}
    \centering
    \subcaption{{\footnotesize $\negap$ on layer $h=1$}}\label{fig:layer1}
    \vspace{-.7em}
    \includegraphics[width=.95\textwidth]{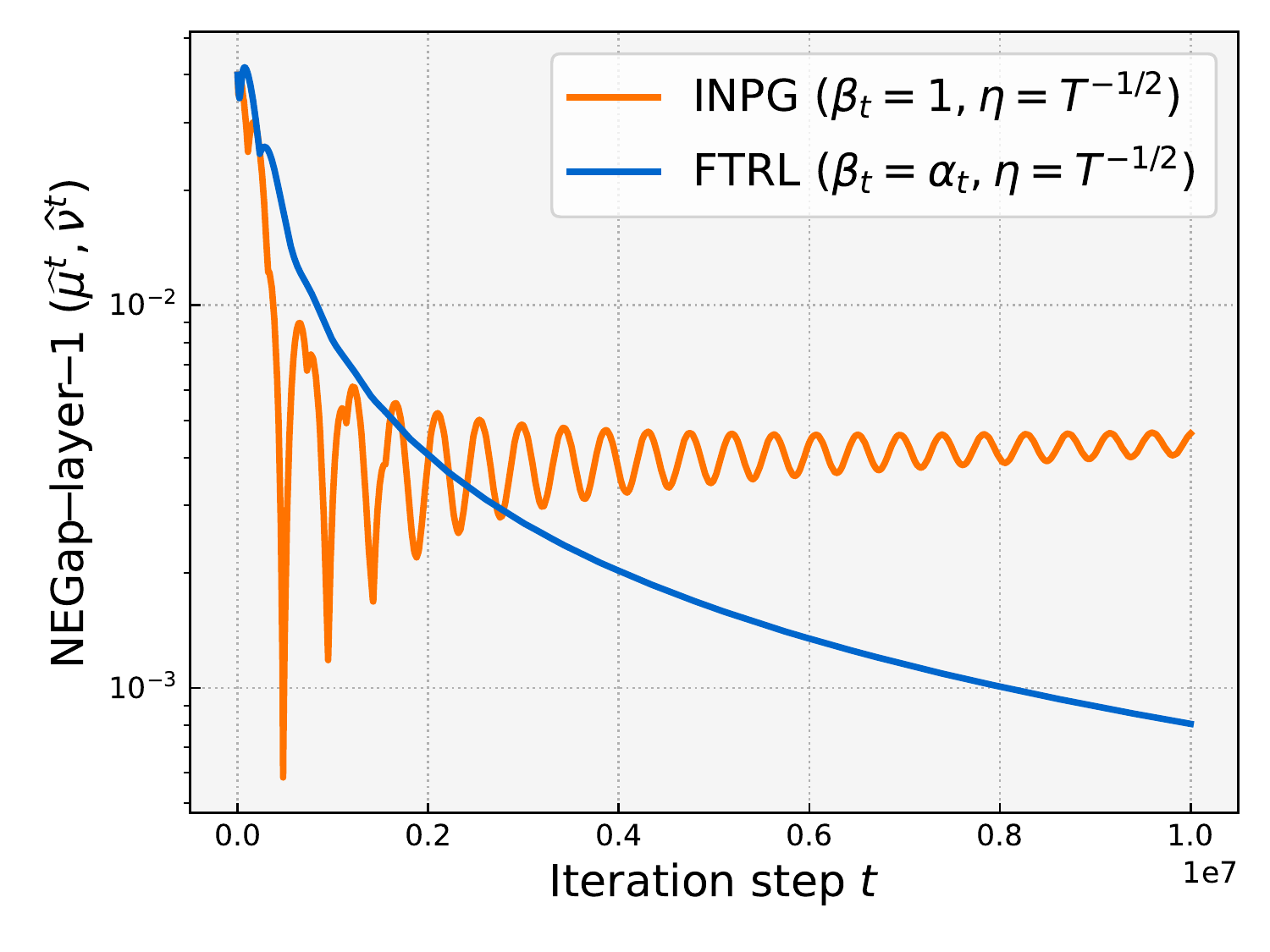}
\end{minipage}
\caption{
\small
    {\bf (a)} $\negap$ of the final output policies ($y$-axis) against total \#~iterations $T$ ($x$-axis) on the two-layer example (cf. Appendix~\ref{apdx:details}) in log-log scale. \textbf{Each dot represents a different run with its own $(T,\eta)$}. The scalings of the form $\sim T^{-\alpha}$ are obtained via best linear fits in the log space. {\bf (b,c)} Layer-wise $\negap$s ($y$-axis, log-scale) against iteration count $t$ ($x$-axis) for \{INPG, FTRL\} on a {\bf single run} with $T=10^7$ and $\eta=T^{-1/2}$.
    }
\label{fig:main}
\vspace{-1em}
\end{figure}

\section{Conclusion}
This paper provides a unified framework for analyzing a large class of policy optimization algorithms for two-player zero-sum Markov games. Using our framework, we prove new fast convergence rates for the OFTRL algorithm with smooth value updates: $\tO(T^{-5/6})$ for learning Nash Equilibria two-player zero-sum Markov games, which can be further accelerated to $\o(T^{-1})$ by slightly modifying the framework; and $\tO(T^{-3/4})$ for learning Coarse Correlated Equilibria in multi-player general-sum Markov games. We further demonstrate the importance of smooth value updates on a simple numerical example. We believe our work opens up many other interesting directions, such as whether improved rates (e.g. $\tO(T^{-1})$) are available for the unmodified OFTRL algorithm, or further investigation of policy optimization algorithms with eager value updates (such as Independent Natural Policy Gradients). Finally, a limitation of this work is its focus on the full-information setting, and it is an important open question how to generalize our analyses to the sample-based setting. 

\section*{Acknowledgment}
The authors would like to thank Chi Jin, Yuanhao Wang, Tiancheng Yu, Shicong Cen, and Song Mei for the valuable discussions. Runyu Zhang is supported by NSF AI institute: 2112085 and ONR YIP: N00014-19-1-2217.

\bibliographystyle{abbrvnat}
\bibliography{ref}

\makeatletter
\def\renewtheorem#1{%
  \expandafter\let\csname#1\endcsname\relax
  \expandafter\let\csname c@#1\endcsname\relax
  \gdef\renewtheorem@envname{#1}
  \renewtheorem@secpar
}
\def\renewtheorem@secpar{\@ifnextchar[{\renewtheorem@numberedlike}{\renewtheorem@nonumberedlike}}
\def\renewtheorem@numberedlike[#1]#2{\newtheorem{\renewtheorem@envname}[#1]{#2}}
\def\renewtheorem@nonumberedlike#1{  
\def\renewtheorem@caption{#1}
\edef\renewtheorem@nowithin{\noexpand\newtheorem{\renewtheorem@envname}{\renewtheorem@caption}}
\renewtheorem@thirdpar
}
\def\renewtheorem@thirdpar{\@ifnextchar[{\renewtheorem@within}{\renewtheorem@nowithin}}
\def\renewtheorem@within[#1]{\renewtheorem@nowithin[#1]}
\makeatother

\renewtheorem{theorem}{Theorem}[section]
\renewtheorem{lemma}{Lemma}[section]
\renewtheorem{remark}{Remark}
\renewtheorem{corollary}{Corollary}[section]
\renewtheorem{observation}{Observation}[section]
\renewtheorem{proposition}{Proposition}[section]
\renewtheorem{definition}{Definition}[section]
\renewtheorem{claim}{Claim}[section]
\renewtheorem{fact}{Fact}[section]
\renewtheorem{assumption}{Assumption}[section]
\renewcommand{\theassumption}{\Alph{assumption}}
\renewtheorem{conjecture}{Conjecture}[section]

\appendix

\section{Technical tools}

\subsection{Properties of $\alpha_t^i$}


Throughout this section, the sequence $\set{\beta_t^i}_{i\in[t]}$ is defined through sequence $\set{\beta_t}_{t\ge 1}$ as in~\eqref{eq:def-beta-k-t}, and $\alpha_t^i$ is its special case with $\beta_t=\alpha_t$, where $\set{\alpha_t}_{t\ge 1}$ is defined in \eqref{eq:def-alpha-t}. We present some basic algebraic properties of $\alpha_t^i$ that will be used in later proofs.
\begin{lemma}\label{lem:delta.monotonic}
Given a sequence $\{\Delta_h^t\}_{h,t}$ defined by
\begin{equation}
    \begin{cases}
    \Delta_h^t =  \sum_{i=1}^{t} \alpha_{t}^i \Delta_{h+1}^{i} + \beta_t,\\
    \Delta_{H+1}^t= 0,~ \text{ for all } t,
    \end{cases}
\end{equation}
where $\{\beta_t\}$ is non-increasing w.r.t. $t$. Then $\Delta_h^{t+1}\le \Delta_h^t$ for all $(t,h)\in\mathbb{N}\times[H+1]$.

\begin{proof}
We prove by doing backward induction on $h$. For the base case of induction, notice that the claim is true for $H+1$. Assume the claim is true for $h+1$. At step $h$, we have 
\begin{align*}
    \Delta_h^{t+1} & =  \sum_{i=1}^{t+1} \alpha_{t+1}^i \Delta_{h+1}^{i} + \beta_{t+1}\\
    & = (1-\alpha_{t+1}) \sum_{i=1}^{t} \alpha_{t}^i \Delta_{h+1}^{i} + \alpha_{t+1}\Delta_{h+1}^{t+1} +\beta_{t+1}\\
    & \le (1-\alpha_{t+1}) \sum_{i=1}^{t} \alpha_{t}^i \Delta_{h+1}^{i} + \alpha_{t+1} \sum_{i=1}^{t} \alpha_{t}^i \Delta_{h+1}^{i} +\beta_{t} = \Delta_h^t,
\end{align*}
where the inequality follows from the inductive hypothesis and $\beta_{t+1}\le \beta_t$.
\end{proof}
\end{lemma}

The following lemma is taken from~\citep{jin2018q}.
\begin{lemma}
\label{lemma:alpha-ti}
The sequence $\alpha_t^i$ satisfies the following:
\begin{enumerate}[wide, label=(\alph*)]
\item $\sum_{t=i}^\infty \alpha_t^i = 1 + 1/H$ for all $i\ge 1$.
\end{enumerate}
\end{lemma}




\begin{lemma}[Convolution of $\beta_T^t$ with decaying sequences]
\label{lemma:alpha-convolution}
From general $\{\beta_t\}_{t\ge1}$ sequence, we have that
\begin{equation*}
    X_T\defeq\sum_{t=1}^T\frac{1}{t}\beta_T^t\le \frac{2\cbeta\log(T)}{T},~~~\mbox{for $T\ge 2$}.
\end{equation*}
Specifically if $\beta_t = \alpha_t$, the following holds for all $T\ge 1$:
\begin{enumerate}[wide, label=(\alph*)]
\item $A_T \defeq \sum_{t=1}^T \alpha_T^t \cdot \frac{1}{t^2} \le \frac{4}{T}$.
\item $B_T\defeq \sum_{t=1}^T \alpha_T^t \alpha_t\le \frac{(H+1)^2}{H(H+T)}$
\item $C_T \defeq \sum_{t=1}^T \alpha_T^t \cdot \alpha_t^2 \le \frac{4H}{T}$.
\end{enumerate}
\end{lemma}
\begin{proof}
We first prove for the inequality on $X_T$ for general $\{\beta_t\}_{t\ge1}$ sequence. We start with showing that $X_{t+1}\le X_t,~\forall t\ge 1$.
\begin{align*}
    X_{t+1} &= \sum_{i=1}^{t+1}\frac{1}{i}\beta_{t+1}^i = \sum_{i=1}^t\frac{1}{i}\beta_{t+1}^i + \frac{\beta_{t+1}}{t+1}\\
    &= (1-\beta_{t+1})\sum_{i=1}^t\frac{1}{i}\beta_{t}^i + \frac{\beta_{t+1}}{t+1} = (1-\beta_{t+1})X_t + \frac{\beta_{t+1}}{t+1}\\
    \Longrightarrow~~ &X_{t+1} - X_t = \beta_{t+1}\paren{\frac{1}{t+1}-X_t}. 
\end{align*}
Since
\begin{align*}
    \frac{1}{t+1}-X_t = \frac{1}{t+1}-\sum_{i=1}^t\frac{1}{i}\beta_{t}^i  = \sum_{i=1}^t\beta_{t}^i\paren{\frac{1}{t+1}-\frac{1}{i}}\le 0,
\end{align*}
we have that $X_{t+1} - X_t\le 0$.Thus
\begin{align*}
    X_T &\le \frac{1}{T}\sum_{t=1}^TX_t = \frac{1}{T}\sum_{t=1}^T\sum_{i=1}^t\frac{1}{i}\beta_t^i=\frac{1}{T}\sum_{i=1}^T\frac{1}{i}\paren{\sum_{t=i}^T\beta_t^i}\le \frac{\cbeta}{T}\sum_{i=1}^T\frac{1}{i}\le \frac{2\cbeta\log(T)}{T}.
\end{align*}

Now we prove for the specific case where $\beta_t = \alpha_t$:
\begin{enumerate}[wide, label=(\alph*)]
\item Note that $A_1=1$ and we have the recursive relationship
\begin{align*}
    A_{T+1} = (1-\alpha_{T+1})A_T + \alpha_{T+1}\cdot \frac{1}{(T+1)^2}
\end{align*}
by definition of the sequence $\alpha_T^t$. In particular this implies $A_{T+1}\le A_T$, since $A_T$ is a weighted average of $1/t^2\ge 1/(T+1)^2$. Therefore we have
\begin{align*}
    & \quad A_T \le \frac{1}{T}\sum_{t=1}^T A_t = \frac{1}{T} \sum_{t=1}^T \sum_{s=1}^t \alpha_t^s \cdot \frac{1}{s^2} \\
    & = \frac{1}{T} \sum_{s=1}^T \underbrace{\sum_{t=s}^T \alpha_t^s}_{\le 1+1/H\le 2} \cdot \frac{1}{s^2} \le \frac{2}{T} \sum_{s=1}^T \frac{1}{s^2} \le \frac{2}{T} \sum_{s=1}^\infty \frac{1}{s^2} \le \frac{4}{T}.
\end{align*}
Above, the step $\sum_{t=s}^T \alpha_t^s\le \sum_{t=s}^\infty \alpha_t^s=1+1/H$ follows from Lemma~\ref{lemma:alpha-ti}.

\item From the definition of $B_T$ we have that
\begin{align*}
    &B_{T+1} =  \sum_{t=1}^{T+1} \alpha_{T+1}^t \alpha_t = \sum_{t=1}^T (1-\alpha_{T+1})\alpha_T^t\alpha_t + \alpha_{T+1}^2 = (1-\alpha_{T+1}^t)B_T + \alpha_{T+1}\\
    \Longrightarrow ~~&B_{T+1} = \frac{T}{H+T+1}B_T + \frac{(H+1)^2}{(H+T+1)^2}\le \frac{T}{H+T+1}B_T  + \frac{(H+1)^2}{(H+T+1)(H+T)}\\
    \Longrightarrow ~~&\paren{B_{T+1} - \frac{(H+1)^2}{H(H+T+1)}}\le\frac{T}{H+T+1} \paren{B_{T} - \frac{(H+1)^2}{H(H+T)}}.
\end{align*}
Since $B_1 = \alpha_1^2 = 1\le \frac{(H+1)^2}{H(H+1)}$, we have that $B_T \le \frac{(H+1)^2}{H(H+T)} $ via proof by induction.
\item Since $\alpha_t\le 1$, we have that $C_T\le B_T$, thus by part (b)
\begin{align*}
    C_T\le B_T \le \frac{(H+1)^2}{H(H+T)}\le \frac{4H}{T}.
\end{align*}

\end{enumerate}
\end{proof}

Consider the sequence $\set{w_t}_{t\ge 1}$ defined by (cf. also Example~\ref{example:Nash-V})
\begin{align}
    w_t = \alpha_t^t / \alpha_t^1.
\end{align}
Note that we also have $w_t=\alpha_T^t/\alpha_T^1$ for any $T\ge t$.
\begin{lemma}[Properties of $w_t$]
\label{lemma:wt}
The following holds for all $t\ge 2$:
\begin{enumerate}[wide, label=(\alph*)]
\item $w_t / w_{t-1}=(H+t-1)/(t-1)$.
\item $(\frac{1}{w_{t-1}} - \frac{1}{w_t})\sum_{i=1}^{t-1}w_i = H/(H+1)$.
\end{enumerate}
\end{lemma}
\begin{proof}
\begin{enumerate}[wide, label=(\alph*)]
    \item We have
    \begin{align*}
        \frac{w_t}{w_{t-1}} = \frac{\alpha_t}{\alpha_{t-1}(1-\alpha_t)} = \frac{(H+1)/(H+t)}{(H+1)/(H+t-1) \cdot (t-1)/(H+t)} = \frac{H+t-1}{t-1}.
    \end{align*}
    \item We have
    \begin{align*}
        & \quad \paren{\frac{1}{w_{t-1}} - \frac{1}{w_t}}\sum_{i=1}^{t-1}w_i = \frac{1}{w_{t-1}} \paren{1 - \frac{w_{t-1}}{w_t}}\sum_{i=1}^{t-1}w_i = \paren{1 - \frac{w_{t-1}}{w_t}} \cdot \frac{1}{\alpha_{t-1}^{t-1}} \\
        & \stackrel{(i)}{=} \frac{H}{H+t-1} \cdot \frac{H+t-1}{H+1} = \frac{H}{H+1}.
    \end{align*}
    Above, (i) used part (a).
\end{enumerate}
\end{proof}

\subsection{Other technical lemmas}

\begin{lemma}[Smoothness of Exponential Weights]
\label{lemma:Hedge-smoothness}
Let $g_1, g_2\in \mathbb{R}^n$ and
\begin{align*}
    x_1 &= \argmax_{x\in\Delta_{[n]}} \<x, g_1\> - H(x), \\
    x_2 &= \argmax_{x\in\Delta_{[n]}}\<x,g_2\> - H(x),
\end{align*}
where $H(x)\defeq \sum_{i=1}^n x_i\log x_i$ is the standard entropy functional. Then $\|x_1 - x_2\|_1 \le 2\|g_1 - g_2\|_\infty$.
\begin{proof}
Since $H$ is $1$-strongly convex in $\|\cdot\|_1$ (Pinsker's inequality), we have that
\begin{align*}
    \frac{\|x_1-x_2\|_1^2}{2} &\le H(x_1) - H(x_2) - \<\nabla H(x_2), x_1-x_2\>\\
    &= H(x_1) + \left(\<x_2,g_1\>-H(x_2)\right) - \<x_2,g_1\> - \<\nabla H(x_2), x_1-x_2\>\\
    &\le  H(x_1) + \left(\<x_1,g_1\>-H(x_1)\right) - \<x_2,g_1\> - \<\nabla H(x_2), x_1-x_2\>\\
    &= \<x_1-x_2,g_1\> - \<\nabla H(x_2), x_1-x_2\>\\
    &= \<x_1-x_2,g_1-g_2\> + \<g_2 -\nabla H(x_2), x_1-x_2 \>\\
    &\le \<x_1-x_2,g_1-g_2\> ~~~\textup{(From the optimality of $x_2$)}\\
    &\le \|x_1-x_2\|_1\|g_1-g_2\|_\infty,
\end{align*}
which completes the proof.
\end{proof}
\end{lemma}



\section{Bound for regret minimization algorithms}

\subsection{Projected gradient descent}

\begin{algorithm}[h]
\caption{Projected gradient descent}
\label{algorithm:pgd}
\begin{algorithmic}[1]
\REQUIRE Learning rate $\eta>0$.
\STATE \textbf{Initialize} $x_1\in\cX\subset \R^d$.
\FOR{$t=1,\ldots,T$}
\STATE Receive loss $g_t\in\R^d$.
\STATE Compute update $y_{t+1} = x_t - \eta g_t$, $x_{t+1} = \cP_\cX( y_{t+1}) $.
\ENDFOR
\end{algorithmic}
\end{algorithm}

The following weighted regret bound for projected gradient descent is standard. For completeness we provide a proof here. For simplicity of notation, denote the diameter of $\cX\subset\R^d$ by $R$ and $G=\max_{t\in[T]}\|g_t\|_2$.

\begin{lemma}[Weighted regret bound for projected gradient descent]
\label{lem:gd.regret}
 For any weights $\set{w_t}_{t\ge 1}\in\R_{>0}$ with $w_t\le w_{t+1}$ for all $t\ge 1$, Algorithm~\ref{algorithm:pgd} achieves
$$
 \max_{z\in\cX} \sum_{t=1}^T  w_t \langle x_t -z, g_t\rangle \le \frac{w_T}{2\eta} R^2  + \frac{\eta \sum_{t=1}^T w_t \cdot G^2}{2}.
$$
\end{lemma}
\begin{proof}
By following the standard GD analysis, we first have
\begin{align*}
    \langle x_t -z, g_t\rangle =& \frac{1}{\eta} \langle x_t -z, x_t - y_{t+1}\rangle \\
    = & \frac{1}{2\eta} \left[ \|x_t-z\|^2+\|x_t-y_{t+1}\|^2 - \|z-y_{t+1}\|^2 \right]\\
    \le& \frac{1}{2\eta} \left[ \|x_t-z\|^2+\|x_t-y_{t+1}\|^2 - \|z-x_{t+1}\|^2 \right]\\
    = &\frac{1}{2\eta} \left[ \|x_t-z\|^2 - \|z-x_{t+1}\|^2 + \eta^2\|g_t\|^2\right],
\end{align*}
where the inequality follows from $z\in\cX$ and $x_{t+1} = \cP_\cX( y_{t+1})$. By multiplying both sides with $w_t$ and taking summation over $t\in[T]$, we have
\begin{align*}
   \sum_{t=1}^T  w_t \langle x_t -z, g_t\rangle \le & \frac{1}{2\eta}\sum_{t=1}^T  w_t  \left[ \|x_t-z\|^2 - \|z-x_{t+1}\|^2 + \eta^2\|g_t\|^2\right]\\
   =  &\frac{1}{2\eta}\sum_{t=1}^{T-1} (w_{t+1}-w_t)\|z-x_{t+1}\|^2
   + w_1 \|x_1 - z\|^2 + \frac{ \eta \sum_{t=1}^T w_t\cdot G^2}{2}\\
   \stackrel{(i)}{\le} & \frac{1}{2\eta}\sum_{t=1}^{T-1} (w_{t+1}-w_t) R^2 + w_1 R^2 + \frac{ \eta G^2}{2} = \frac{1}{2\eta} w_T R^2  +\frac{ \eta \sum_{t=1}^T w_t\cdot G^2}{2},
\end{align*}
Above, (i) follows as $w_{t+1}\ge w_t$. This completes the proof.
\end{proof}


\subsection{Follow-The-Regularized Leader (FTRL)}

In this subsection, we consider the following weighted FTRL algorithm over the probability simplex $\Delta_{[A]}$ with the standard (negative) entropy regularizer $\Phi(x)\defeq \sum_{a\in[A]}x(a)\log x(a)$. Below the notation $x_t(a)$ denotes the $a$-th entry of $x_t$.


\begin{algorithm}[h]
\caption{Weighted FTRL with changing learning rate}
\label{alg:FTRL}
\begin{algorithmic}[1]
\REQUIRE Learning rate $\eta>0$; Weights $\set{w_t}_{t\ge 1}\subset \R_{>0}$.
\STATE \textbf{Initialize} $x_1\setto \ones_{A}/A$ to be the uniform distribution over $[A]$.
\FOR{$t=1,\ldots,T$}
\STATE Receive loss $g_t\in\R^A$.
\STATE Compute FTRL update
\begin{align}
\label{equation:ftrl}
    x_{t+1}&\setto \argmin_{x\in\Delta_{[A]}} \< x, \sum_{s=1}^t w_sg_s\> + \frac{w_t}{\eta} \Phi(x).
\end{align}
\ENDFOR
\end{algorithmic}
\end{algorithm}

Note that~\eqref{equation:ftrl} has a closed-form solution via exponential weights:
\begin{align}
    x_{t+1}(a)\propto_a \exp\paren{-\frac{\eta}{w_t}\sum_{s=1}^t w_sg_s(a)}.
\end{align}
    
\begin{lemma}[Regret bound of weighted FTRL]
\label{lemma:FTRL}
Suppose the weights are non-decreasing: $w_{t+1}\ge w_t$ for all $t\ge 1$, and $\max_{t}\|g_t\|_\infty\le G$. Then Algorithm \ref{alg:FTRL} achieves weighted regret bound
$$
 \max_{z\in\Delta_{[A]}}\sum_{t=1}^T  w_t \langle x_t -z, g_t\rangle \le \frac{w_{T}}{\eta}\log A + \frac{\eta G^2}{2} \sum_{t=1}^T w_t.
$$
\begin{proof}
Applying standard anytime FTRL analysis (see, e.g., Excercise 28.12 in \cite{lattimore2020bandit}) with loss sequence $\set{w_tg_t}_{t\ge 1}$ and learning rate $\set{\eta/w_t}_{t\ge 1}$, we have that
\begin{align*}
    \sum_{t=1}^T  w_t \langle x_t -z, g_t\rangle &\le \frac{w_{T}(\Phi(x_1) - \min_x \Phi(x))}{\eta} + \sum_{t=1}^Tw_t\paren{\<x_t - x_{t+1}, g_t\> - \frac{ \KL(x_{t+1}||x_t)}{\eta}}\\
    &\le \frac{w_{T}\log A}{\eta} + \sum_{t=1}^T\paren{w_t\<x_t - x_{t+1}, g_t\> - \frac{w_t \|x_t - x_{t+1}\|_1^2}{2\eta}}~~~~\textup{(Pinsker's inequality)}\\
    &\le \frac{w_{T}\log A}{\eta} + \sum_{t=1}^T \frac{w_t\eta}{2}\|g_t\|_\infty^2\\
    &\le \frac{w_{T}\log A}{\eta}+\frac{\eta G^2}{2}\sum_{t=1}^Tw_t,
\end{align*}
which completes the proof.
\end{proof}
\end{lemma}

\subsection{Optimistic Follow-The-Regularized-Leader (OFTRL)}



We consider the following OFTRL algorithm on the probability simplex $\Delta_{[A]}$ with standard (negative) entropy regularizer $\Phi(x)=\sum_{a\in[A]} x(a)\log x(a)$.

\begin{algorithm}[h]
\caption{Anytime OFTRL}
\label{alg:OFTRL}
\begin{algorithmic}[1]
\REQUIRE Learning rate $\set{\eta_t}_{t\ge 1}$.
\STATE \textbf{Initialize} $x_1\setto \ones_{A}/A$ to be the uniform distribution over $[A]$.
\FOR{$t=1,\ldots,T$}
\STATE Receive loss $g_t\in\R^A$.
\STATE Compute a prediction vector $M_{t+1}\in\R^A$ using past observations.
\STATE Compute OFTRL update
\begin{align}
\label{equation:oftrl}
    x_{t+1}&\setto \argmin_{x\in\Delta_{[A]}} \eta_{t+1} \< x, \sum_{s=1}^t g_s + M_{t+1}\> + \Phi(x).
\end{align}
\ENDFOR
\end{algorithmic}
\end{algorithm}

Note that~\eqref{equation:oftrl} has a closed-form solution via exponential weights:
\begin{align}
    x_{t+1}(a)\propto_a \exp\paren{-\eta_{t+1} \brac{\sum_{s=1}^t g_s(a) + M_{t+1}(a)}}.
\end{align}

The following regret bound for OFTRL follows similarly as standard OFTRL analysis, see, e.g.~\citep[Lemma 1]{rakhlin2013optimization}. For completeness, we provide a proof here.
\begin{lemma}[Regret bound for OFTRL]
\label{lemma:OFTRL-auxillary}
Suppose the learning rates are non-increasing: $\eta_t\ge \eta_{t+1}$ for all $t\ge 1$. Then Algorithm~\ref{alg:OFTRL} achieves the following bound for all $x\in\cX$:
$$
\sum_{t=1}^T \langle x_t - x, g_t\rangle\le 
\frac{\log A}{\eta_T} + 
\sum_{t=1}^{T} \eta_t \|g_t-M_t\|_\infty^2  -\sum_{t=1}^{T-1} \frac{1}{8\eta_t}\|x_t-x_{t+1}\|_1^2.
$$
\end{lemma}
\begin{proof}
Consider a fixed $T\ge 1$. Note that Algorithm~\ref{alg:OFTRL} is equivalent to Algorithm~\eqref{equation:oftrl-general} with regularizer $R_t(\cdot)\defeq (\Phi(\cdot) + \log A)/\eta_{t+1}\ge 0$ for $t\ge 0$, and $R_{T}(\cdot)\defeq R_{T-1}(\cdot)$ (Note that the shifting by $\log A$ does not affect the algorithm.)

We first decompose the regret into the following three terms
$$
\sum_{t=1}^T \langle x_t - x, g_t\rangle =
\sum_{t=1}^{T} \langle q_{t+1}-x, g_t\rangle + \sum_{t=1}^{T} \langle x_t - q_{t+1}, M_t\rangle + \sum_{t=1}^{T}\langle x_t-q_{t+1}, g_t-M_t\rangle,
$$
where $\set{q_t}_{t\ge 1}$ is defined in~\eqref{equation:qt}.
By Lemma \ref{lem:OFTRL-recursion}, we can upper bound the first two terms by $R_{T}(x)-\min_{x'}R_0(x')+S_{T}\le R_T(x) + S_T$ and obtain
\begin{align*}
&\sum_{t=1}^T \langle x_t - x, g_t\rangle \\
\le& R_{T}(x)+\sum_{t=1}^{T}
\left(R_{t-1}(q_{t+1})- R_{t}(q_{t+1})\right)+
\sum_{t=1}^{T} \left(\langle x_t-q_{t+1}, g_t-M_t\rangle - D_{R_{t-1}}(q_{t+1},x_t) -  D_{R_{t-1}}(x_t,q_t)\right)\\
\le &  R_{T}(x)+
\sum_{t=1}^{T} \left(\|x_t-q_{t+1}\|_1 \|g_t-M_t\|_\infty - \frac{1}{2\eta_t}\|q_{t+1}-x_t\|_1^2 -  \frac{1}{2\eta_t}\|x_t-q_t\|_1^2\right),
\end{align*}
where the second inequality uses $R_{t-1} \le R_{t}$ and $R_{t-1}$ is $1/\eta_{t}$ strongly-convex w.r.t. $\|\cdot\|_1$.
Finally, we conclude the proof by applying Cauchy-Schwarz inequality: 
$$
\|x_t-q_{t+1}\|_1 \|g_t-M_t\|_\infty - \frac{1}{4\eta_t}\|q_{t+1}-x_t\|_1^2  \le  \eta_t \|g_t-M_t\|_\infty^2, 
$$
and triangle inequality
$$
- \frac{1}{4\eta_t}\|q_{t+1}-x_t\|_1^2 - \frac{1}{4\eta_{t}}\|q_{t+1}-x_{t+1}\|_1^2
\le -\frac{1}{8\eta_t}\|x_{t+1}-x_t\|_1^2,
$$
and the bound $R_T(x)\le \log A / \eta_{T+1}=\log A/\eta_T$ for any $x\in\Delta_{[A]}$.
\end{proof}

The following lemma bounds the total variation of the iterates in terms of the smoothness of loss vectors and prediction vectors. This can be seen as a generalization of~\citep[Lemma 3.2]{chen2020hedging} to the case with changing learning rate and arbitrary prediction vectors.

\begin{lemma}[Bounding stability by the smoothness of loss]
\label{lemma:oftrl-stability}
Suppose the learning rates are non-increasing: $\eta_t\ge \eta_{t+1}$ for all $t\ge 1$. Then the OFTRL algorithm~\eqref{equation:oftrl} satisfies (understanding $M_1\defeq 0$)
\begin{align}
    & \quad \sum_{t=2}^{T} \frac{1}{2\eta_t} \lone{x_t - x_{t-1}}^2 \le \frac{\log A}{\eta_T} + \max_{x\in\Delta_{[d]}} \sum_{t=1}^{T-1} \<x_t - x, g_t\> + \sum_{t=2}^{T} \linf{M_t - M_{t-1}} + \linf{M_T} \label{equation:l1-bound-1} \\
    & \le \frac{2\log A}{\eta_T} + \sum_{t=1}^{T-1} \eta_t\linf{g_t - M_t}^2 + \sum_{t=2}^{T} \linf{M_t - M_{t-1}} + \linf{M_T}. \label{equation:l1-bound-2}
\end{align}
In particular, choosing the prediction vector $M_t=g_{t-1}$ with $g_0\defeq 0$, and assume  $\linf{g_t - g_{t-1}}\le G_t$ for all $t\ge 1$, we have
\begin{align}
    \label{equation:l1-bound-3}
    \begin{aligned}
    & \quad \sum_{t=2}^{T} \frac{1}{2\eta_t} \lone{x_t - x_{t-1}}^2 \le \frac{2\log A}{\eta_T} + \sum_{t=1}^{T-1} \eta_t\linf{g_t - g_{t-1}}^2 + \sum_{t=2}^{T} \linf{g_{t-1} - g_{t-2}} + \linf{g_{T-1}} \\
    & \le \frac{2\log A}{\eta_T} + \sum_{t=1}^{T-1} (1 + \eta_tG_t) \linf{g_t - g_{t-1}} + \linf{g_{T-1}}.
\end{aligned}
\end{align}
\end{lemma}
\begin{proof}
We first prove~\eqref{equation:l1-bound-1}. For any $t\ge 2$, the optimality condition of~\eqref{equation:oftrl} for $x_t$ gives
\begin{align*}
    \<\sum_{s=1}^{t-1} g_s + M_t + \frac{\grad\Phi(x_t)}{\eta_t}, x' - x_t\> \ge 0
\end{align*}
for all $x'\in\Delta_{[A]}$. In particular, this holds for $x'=x_{t-1}$, from which we get
\begin{align*}
    & \quad \frac{1}{2\eta_t}\lone{x_{t-1} - x_t}^2 \stackrel{(i)}{\le} \frac{1}{\eta_t}\KL(x_{t-1}\|x_t) \stackrel{(ii)}{=} \frac{\Phi(x_{t-1}) - \Phi(x_t)}{\eta_t} - \<\frac{\grad\Phi(x_t)}{\eta_t}, x_{t-1} - x_t\> \\
    & \le \frac{\Phi(x_{t-1}) - \Phi(x_t)}{\eta_t} + \<\sum_{s=1}^{t-1} g_s + M_t, x_{t-1} - x_t\>,
\end{align*}
where (i) is by Pinsker's inequality, and (ii) is since the KL divergence is the Bregman divergence of $\Phi$. Summing the above over $t=2,\dots,T$ yields
\begin{align*}
    & \quad \sum_{t=2}^T \frac{1}{2\eta_t}\lone{x_{t-1} - x_t}^2 \\
    & \le \underbrace{-\frac{\Phi(x_T)}{\eta_T}}_{\le \log A/\eta_T} + \sum_{t=2}^{T} \underbrace{\paren{\frac{1}{\eta_t} - \frac{1}{\eta_{t-1}}}}_{\ge 0}\underbrace{\Phi(x_{t-1})}_{\le 0} - \<\sum_{s=1}^{T-1} g_s + M_T, x_T\> + \sum_{t=2}^{T} \<g_{t-1} + M_t - M_{t-1}, x_{t-1} \> \\
    & \le \frac{\log A}{\eta_T} + \sum_{t=1}^{T-1} \<x_t, g_t\> - \sum_{t=1}^{T-1} \<x_T, g_t\> + \<M_T, x_T\> + \sum_{t=2}^{T}\<M_t - M_{t-1}, x_{t-1}\> \\
    & \le \frac{\log A}{\eta_T} + \max_{x\in\Delta_{[d]}} \sum_{t=1}^{T-1} \<x_t-x, g_t\> + \sum_{t=2}^{T}\linf{M_t - M_{t-1}} + \linf{M_T}.
\end{align*}
This proves~\eqref{equation:l1-bound-1}. Then,~\eqref{equation:l1-bound-2} follows by plugging in the regret bound given by Lemma~\ref{lemma:OFTRL-auxillary}:
\begin{align*}
    \max_{x\in\Delta_{[d]}} \sum_{t=1}^{T-1} \<x_t-x, g_t\> \le \frac{\log A}{\eta_{T-1}} + \sum_{t=1}^{T-1}\eta_t \linf{g_t - M_t}^2 \le \frac{\log A}{\eta_{T}} + \sum_{t=1}^{T-1}\eta_t \linf{g_t - M_t}^2.
\end{align*}

Finally,~\eqref{equation:l1-bound-3} is a direct consequence of~\eqref{equation:l1-bound-2} by plugging in $M_t=g_{t-1}$ and $\linf{g_t - g_{t-1}}\le G_t$ for all $t\ge 1$.
\end{proof}

\subsubsection{Auxiliary lemma for OFTRL with general regularizers}

Consider an OFTRL algorithm with loss function $\set{g_t}_{t\ge 0}\subset \R^d$, parameter space $\cX\subset \R^d$, and convex regularizers $R_t:\cX\to \R$ for $t\ge 0$: 
\begin{align}
\label{equation:oftrl-general}
    x_{t+1}&\setto \argmin_{x\in\cX} \< x, \sum_{s=1}^t g_s + M_{t+1}\> + R_t(x).
\end{align}
Define auxiliary sequence
\begin{align}
\label{equation:qt}
    q_{t+1} = \arg\min_{x\in\cX} \left\langle x, \sum_{s=1}^t g_s \right\rangle + R_t(x).
\end{align}
Recall the Bregman divergence associated with any convex regularizer $R:\cX\to\R$ is given by
\begin{align*}
    D_{R}(x, y) \defeq R(x) - R(y) - \<\grad R(y), x-y\> \ge 0.
\end{align*}



\begin{lemma}[Auxiliary lemma for OFTRL with general regularizers]
\label{lem:OFTRL-recursion}
Algorithm~\eqref{equation:oftrl-general} achieves the following for any $T\ge 1$ and $x\in\cX$:
\begin{align*}
\sum_{t=1}^{T} \langle q_{t+1}, g_t\rangle + \sum_{t=1}^{T} \langle x_t - q_{t+1}, M_t\rangle \le \sum_{t=1}^{T}\langle x, g_t\rangle + R_T (x) - \min_{x'\in\cX} R_0(x') + S_{T},
\end{align*}
where
\begin{align*}
S_{T} \defeq 
\sum_{t=1}^{T} \left(R_{t-1}(q_{t+1})- R_{t}(q_{t+1})\right)-
\sum_{t=1}^{T} \left(D_{R_{t-1}}(q_{t+1},x_t) +  D_{R_{t-1}}(x_t,q_t)\right).
\end{align*}
\end{lemma}
\begin{proof}
We prove the lemma by induction.
The above relation holds trivially for $T=0$. Assume the relation holds for $\tau= T-1$. For $\tau= T$, we have
\begin{align*}
    &\sum_{t=1}^{T} \langle q_{t+1}, g_t\rangle + \sum_{t=1}^{T} \langle x_t - q_{t+1}, M_t\rangle\\ 
    \le & \min_{x\in\cX} \brac{\sum_{t=1}^{T-1}\langle x, g_t\rangle + R_{T-1} (x)} - \min_{x'\in\cX} R_0(x') + S_{T-1}+  \langle q_{T+1}, g_T\rangle +  \langle x_T - q_{T+1}, M_T\rangle \\
    = & \sum_{t=1}^{T-1}\langle q_T, g_t\rangle + R_{T-1} (q_T) - \min_{x'\in\cX} R_0(x') \\
    & \qquad + S_{T-1}+  \langle q_{T+1}, g_T\rangle +  \langle x_T - q_{T+1}, M_T\rangle \quad \mbox{(definition of $q_T$)}\\
    \le &\sum_{t=1}^{T-1}\langle x_T, g_t\rangle + R_{T-1} (x_T) - D_{R_{T-1}}(x_T,q_T) - \min_{x'\in\cX} R_0(x') \\
    & \qquad + S_{T-1}+  \langle q_{T+1}, g_T\rangle +  \langle x_T - q_{T+1}, M_T\rangle \quad \mbox{(optimality of $q_T$)}\\
    = &\min_{x\in\cX} \brac{ \langle x, \sum_{t=1}^{T-1}g_t + M_T\rangle + R_{T-1} (x)} - D_{R_{T-1}}(x_T,q_T) - \min_{x'\in\cX} R_0(x') \\
    & \qquad + S_{T-1}+  \langle q_{T+1}, g_T-M_T\rangle  \quad \mbox{(definition of $x_T$)} \\
    \le &  \langle q_{T+1}, \sum_{t=1}^{T-1}g_t + M_T\rangle + R_{T-1} (q_{T+1}) - D_{R_{T-1}}(q_{T+1},x_T)  \\ &\qquad - D_{R_{T-1}}(x_T,q_T) - \min_{x'\in\cX} R_0(x') + S_{T-1}+  \langle q_{T+1}, g_T-M_T\rangle  \quad \mbox{(optimality of $x_T$)} \\
    = &\min_{x\in\cX} \langle x, \sum_{t=1}^{T}g_t\rangle + R_{T} (x) + (R_{T-1}(q_{T+1})- R_{T}(q_{T+1})) \\ &\qquad- D_{R_{T-1}}(q_{T+1},x_T) -  D_{R_{T-1}}(x_T,q_T) - \min_{x'\in\cX} R_0(x') + S_{T-1}\quad \mbox{(definition of $q_{T+1}$)}, 
\end{align*}
which completes the induction.
\end{proof}

\section{Proofs for Section~\ref{section:framework-general}} 
\label{apdx:theorem-master}


In this section we prove Theorem~\ref{theorem:master}. The proof relies on the following two lemmas.


\begin{lemma}[Performance difference for Markov policies]
\label{lem.duality}
In two-player zero-sum Markov games, suppose a Markov policy $(\mu,\nu)$ satisfies the following for all $h\in[H+1]$:
\begin{align*}
    &\max_{s} \max_{\mu^\dagger\in\Delta_\cA} \left(\brac{(\mu^\dagger)^\top Q_h^\star \nu_h}(s)-V_{h}^\star(s) \right)\le \epsilon_h, \\
    &\max_{s} \max_{\nu^\dagger\in\Delta_\cB} \left(V_{h}^\star(s) - \brac{\mu_h^\top Q_h^\star \nu^\dagger}(s) \right)\le \epsilon_h.
\end{align*}
Then we have for all $h\in[H]$ that
\begin{align*}
    \max\set{ \|V_h^{\dagger,\nu} - V_h^\star\|_\infty, \|V_h^{\mu,\dagger} - V_h^\star\|_\infty } \le \sum_{h'=h}^H\epsilon_{h'}.
\end{align*}
\end{lemma}
\begin{proof}
We prove by backward induction over $h$.
The claim is trivial for $h=H+1$.
Suppose the claim holds for step $h+1$. At step $h$,
\begin{align*}
 \|V_h^{\dagger,\nu} - V_h^\star\|_\infty&= \max_s \abs{ \max_{\mu^\dagger\in\Delta_\cA} \brac{(\mu^\dagger)^\top Q_{h}^{\dagger,\nu}\nu_h}(s) - V_{h}^\star(s) }\\
&\le   \max_s  \abs{ \max_{\mu^\dagger\in\Delta_\cA} \brac{(\mu^\dagger)^\top Q_{h}^\star\nu_h}(s) - V_{h}^\star(s) }+ \|Q_{h}^{\dagger,\nu}-Q_{h}^\star\|_\infty\\
&\le \epsilon_h + \|Q_{h}^{\dagger,\nu}-Q_{h}^\star\|_\infty.
\end{align*}
Notice that
\begin{align*}
    & \quad \|Q_{h}^{\dagger,\nu}-Q_{h}^\star\|_\infty \le \max_{s,a,b}\abs{\paren{r_h+\P_h V_{h+1}^{\dagger,\nu} }(s,a,b) - \paren{r_h+\P_h V_{h+1}^{\star}}(s,a,b) }\\
    & \le \max_{s,a,b}\abs{\P_h\brac{V_{h+1}^{\dagger,\nu} - V_{h+1}^\star}(s,a,b)} \le \|V_{h+1}^{\dagger,\nu} - V_{h+1}^\star\|_\infty \le \sum_{h'=h+1}^H \epsilon_{h'}. \quad \textup{(by inductive hypothesis)}
\end{align*}
Combining the two inequalities we get
\begin{align*}
    \|V_h^{\dagger,\nu} - V_h^\star\|_\infty \le \sum_{h'=h}^H \epsilon_{h'}.
\end{align*}
This proves the claim for $\|V_h^{\dagger,\nu} - V_h^\star\|_\infty$. The same argument also holds for $\|V_h^{\mu,\dagger} - V_h^\star\|_\infty$, which completes the proof.
\end{proof}

Throughout the rest of this section, we define the following shorthand for the value estimation error:
\begin{align*}
    \delta_h^t \defeq \linf{ Q_h^t - Q_h^\star } = \max_{s,a,b} \abs{Q_h^t(s,a,b) - Q_h^\star(s,a,b)},
\end{align*}
where $Q_h^t$ is the estimated value in Algorithm \ref{alg:framework}.

\begin{lemma}[Recursion of value estimation]\label{lemma:recursion-of-value-estimation}
Algorithm \ref{alg:framework} guarantees that for all $(t,h)\in[T]\times[H]$,
\begin{align*}
    \delta_h^{t} \le \sum_{i=1}^t \beta_t^i \delta_{h+1}^{i} + \reg_{h+1}^t.
\end{align*}
Further, suppose that $\reg_h^t\le \barreg_h^t$ for all $(h,t)\in[H]\times[T]$, where $\barreg_h^t$ is non-increasing in $t$: $\barreg_h^t\ge \barreg_h^{t+1}$ for all $t\ge 1$. Then we have
\begin{equation*}
    \delta_h^t\le H\cbeta^{H-1} \cdot \frac{1}{t}\sum_{i=1}^t\max_{h'}\barreg_{h'}^i, 
\end{equation*}
where $\cbeta$ is defined in \eqref{eq:def-c}.
\end{lemma}
\begin{proof}
Fix $(h,s,a,b)\in[H]\times\cS\times\cA\times\cB$. From the definition of $Q_h^\star$ we have that
\begin{align*}
    Q^\star_h(s,a,b) &= r_h(s,a,b) + \max_{\mu_{h+1}}\min_{ \nu_{h+1}} \P_h\brac{\mu_{h+1}^\top Q_{h+1}^\star \nu_{h+1}}(s,a,b)\\
    &\le r_h(s,a,b) + \max_{\mu_{h+1}} \P_h\brac{\mu_{h+1}^\top Q_{h+1}^\star \paren{\sum_{i=1}^t \beta_t^i \nu_{h+1}^{i}}}(s,a,b)\\
    &= r_h(s,a,b) + \max_{\mu_{h+1}}\sum_{i=1}^t \beta_t^i \P_h\brac{\mu_{h+1}^\top Q_{h+1}^\star\nu_{h+1}^{i}} (s,a,b)\\
    &\le  r_h(s,a,b) + \max_{\mu_{h+1}}\sum_{i=1}^t \beta_t^i \paren{ \P_h\brac{\mu_{h+1}^\top Q_{h+1}^{i}\nu_{h+1}^{i}} (s,a,b) + \|Q_{h+1}^i-Q^\star_{h+1}\|_\infty }\\
    &\le   r_h(s,a,b) + \sum_{i=1}^t \beta_t^i \P_h\brac{(\mu_{h+1}^i)^\top Q_{h+1}^{i}\nu_{h+1}^{i}} (s,a,b) + \sum_{i=1}^t\beta_t^i\delta_{h+1}^i + \reg_{h+1}^t\\
    &= Q_h^{t}(s,a,b) +\sum_{i=1}^t\beta_t^i \delta_{h+1}^{i}+\reg_{h+1}^t.
\end{align*}
Above, the last equality is derived from the update rule~\eqref{eq:Q-relationship}, which implies that
\begin{equation*}
    Q_h^{t}(s,a,b) = \sum_{i=1}^t\beta_t^i \paren{r_h + \P_h\brac{(\mu_{h+1}^i)^\top Q_{h+1}^{i}\nu_{h+1}^{i}}} (s,a,b).
\end{equation*}
Therefore we have
\begin{equation*}
    Q^\star_h(s,a,b) - Q_h^{t}(s,a,b) \le \sum_{i=1}^t\beta_t^i \delta_{h+1}^{i}+\reg_{h+1}^t,~~\forall~s,a,b.
\end{equation*}
Apply similar analysis to the min-player, we get
\begin{equation*}
      Q_h^{t}(s,a,b) -  Q^\star_h(s,a,b)  \le \sum_{i=1}^t\beta_t^i \delta_{h+1}^{i}+\reg_{h+1}^t,~~\forall~s,a,b.
\end{equation*}
Thus we get
\begin{equation*}
    \delta_h^{t} \le \sum_{i=1}^t\beta_t^i \delta_{h+1}^{i}+{\rm reg}_{h+1}^t,
\end{equation*}
which completes the proof of the first inequality in the Lemma. Now consider an auxiliary sequence 
$\{\Delta_h^t \}_{h,t}$ defined by 
\begin{equation}
    \begin{cases}
    \Delta_h^t =  \sum_{i=1}^{t} \beta_{t}^i \Delta_{h+1}^{i} + \barreg_{h+1}^t,\\
    \Delta_{H+1}^t= 0,~ \text{ for all } t.
    \end{cases}
\end{equation}
Where $\barreg_{h}^t$ is the upperbound of $\reg_h^t$ defined in Theorem \ref{theorem:master}. Observe that $\{\Delta_h^t \}_{h,t}$ satisfies the following properties 
\begin{equation}
    \begin{cases}
    \Delta_h^t \ge \delta_h^t\qquad &\text{ (by definition)},\\
    \Delta_h^t \le \Delta_h^{t-1} \quad &\text{ (by Lemma \ref{lem:delta.monotonic})}.
    \end{cases}
\end{equation}
Therefore, to control $\delta_h^t$, it suffices to bound $\Delta_h^t\le \frac{1}{t}\sum_{i=1}^t \Delta_h^i$, which follows from the standard argument in \cite{jin2018q}:
\begin{align*}
     \frac{1}{t}\sum_{i=1}^t \Delta_h^i 
     &=  \frac{1}{t}\sum_{i=1}^t  \sum_{j=1}^{i} \beta_{i}^j \Delta_{h+1}^{j} + \frac{1}{t}\sum_{i=1}^t \barreg_{h+1}^i\\
     & \le\frac{1}{t}\sum_{j=1}^{t}\paren{\sum_{i=j}^t\beta_{i}^j} \Delta_{h+1}^{j} + \frac{1}{t}\sum_{i=1}^t \barreg_{h+1}^i\\
     &\le \cbeta \cdot \frac{1}{t}\sum_{i=1}^{t}\Delta_{h+1}^{i} + \frac{1}{t}\sum_{i=1}^t \barreg_{h+1}^i\\
     &\le \cbeta^2 \cdot \frac{1}{t}\sum_{i=1}^{t}\Delta_{h+2}^{i} + \cbeta\cdot \frac{1}{t}\sum_{i=1}^t \barreg_{h+2}^i +
     \frac{1}{t}\sum_{i=1}^t \barreg_{h+1}^i\\
     &\le \cdots\\
     &\le \paren{\sum_{h'=h+1}^{H}\cbeta^{h'-h}}\cdot\frac{1}{t}\sum_{i=1}^t \max_{1\le h'\le H}\barreg_{h'}^i\\
     &\le H\cbeta^{H-1}\cdot \frac{1}{t}\sum_{i=1}^t \max_{1\le h'\le H}\barreg_{h'}^i.
\end{align*}
Above, the last step used the fact that $c_\beta\ge 1$. This completes the proof of the second inequality in the Lemma.
\end{proof}

We are now ready to prove the main theorem.
\begin{proof}[Proof of Theorem~\ref{theorem:master}]
Fix any $(h,s)\in[H]\times\cS$. We first give a bound for $\max_{\mu^\dagger\in\Delta_\cA,\nu^\dagger\in\Delta_\cB}\brac{(\mu^\dagger)^\top Q_h^\star \hat\nu_h^T -\paren{\hat\mu_h^T}^\top Q_h^\star \nu^\dagger }(s)$, i.e. the per-state duality gap of $(\hat{\mu}^T, \hat{\nu}^T)$ with respect to $Q_h^\star$.
We have
\begin{align*}
    &\quad \max_{\mu^\dagger\in\Delta_\cA,\nu^\dagger\in\Delta_\cB}\brac{(\mu^\dagger)^\top Q_h^\star \hat\nu_h^T -\paren{\hat\mu_h^T}^\top Q_h^\star \nu^\dagger }(s) \\
    &= \max_{\mu^\dagger\in\Delta_\cA,\nu^\dagger\in\Delta_\cB}\sum_{t=1}^T\beta_T^t\brac{(\mu^\dagger)^\top Q_h^\star \nu_h^t -\paren{\mu_h^t}^\top Q_h^\star \nu^\dagger }(s) \\
    &\le \underbrace{\max_{\mu^\dagger\in\Delta_\cA,\nu^\dagger\in\Delta_\cB}\sum_{t=1}^T\beta_T^t\brac{(\mu^\dagger)^\top Q_h^t \nu_h^t -\paren{\mu_h^t}^\top Q_h^t \nu^\dagger }(s)}_{\reg_{\mu, h}^T(s)+\reg_{\nu, h}^T(s)} +2\sum_{t=1}^T\beta_T^t\delta_h^{t} \\
    &\le 2\barreg_{h}^T + 2H\cbeta^{H-1}\sum_{t=1}^T\beta_T^t\cdot\frac{1}{t}\sum_{i=1}^t\max_{h'}\barreg_{h'}^i~~~\textup{(Lemma \ref{lemma:recursion-of-value-estimation})}\\
    &\le 2\barreg_{h}^T + 2H\cbeta^{H-1}\paren{\sum_{t=1}^T\frac{1}{t}\beta_T^t}\paren{\sum_{i=1}^T\max_{h'}\barreg_{h'}^i}.
\end{align*}
Apply Lemma \ref{lemma:alpha-convolution} into the above inequality, we get that
\begin{align*}
    \max_{\mu^\dagger\in\Delta_\cA,\nu^\dagger\in\Delta_\cB}\brac{(\mu^\dagger)^\top Q_h^\star \hat\nu_h^T \!-\!\paren{\hat\mu_h^T}^\top Q_h^\star \nu^\dagger }(s) \le 2\barreg_{h}^T + 4H\cbeta^{H}\log T\cdot \frac{1}{T}\sum_{t=1}^T\max_{h'}\barreg_{h'}^t.
\end{align*}
Since
\begin{align*}
     \max_{\mu^\dagger\!\in\!\Delta_\cA} \left(\brac{(\mu^\dagger)^{\!\!\top} \!Q_h^\star \hat{\nu}^T_h}(s)\!-\!V_{h}^\star(s) \right), \max_{\nu^\dagger\!\in\!\Delta_\cB} \left(V_{h}^\star(s) \!-\! \brac{(\hat{\mu}_h^T)^\top Q_h^\star \nu^\dagger}(s) \right)\!\le\!\! \max_{\mu^\dagger\!\in\!\Delta_\cA,\nu^\dagger\!\in\!\Delta_\cB}\brac{(\mu^\dagger)^{\!\!\top} \! Q_h^\star \hat\nu_h^T \!-\!\paren{\hat\mu_h^T}^{\!\!\top} \!Q_h^\star \nu^\dagger }(s),
\end{align*}
by applying Lemma \ref{lem.duality} and the preceding per-state duality gap bound, we have
\begin{align*}
    & \quad \negap(\hat{\mu}^T,\hat{\nu}^T)=\paren{V_1^{\dagger,\hat{\nu}^T}(s_1) - V_1^\star(s_1)} + \paren{V_1^\star(s_1)- V_1^{\hat{\mu}^T,\dagger}(s_1)}\\
    &\le 2\sum_{h=1}^H \max_s\max_{\mu^\dagger\in\Delta_\cA,\nu^\dagger\in\Delta_\cB}\brac{(\mu^\dagger)^\top Q_h^\star \hat\nu_h^T \!-\!\paren{\hat\mu_h^T}^\top Q_h^\star \nu^\dagger }(s)\\
    &\le 2\sum_{h=1}^H\paren{2\barreg_{h}^T + 4H\cbeta^{H}\log T\cdot \frac{1}{T}\sum_{t=1}^T\max_{h'}\barreg_{h'}^t}\\
    &\le 4H\max_{h}\barreg_{h}^T + 8H^2\cbeta^{H}\log T\cdot \frac{1}{T}\sum_{t=1}^T\max_{h'}\barreg_{h'}^t.
\end{align*}
This completes the proof.
\end{proof}

\section{Algorithm details and proofs for Section~\ref{section:framework-examples}}\label{apdx:framework-examples}


\subsection{Nash V-Learning (full-information version)}
\label{appendix:nash-v-learning}

The full description of Nash V-Learning (Example~\ref{example:Nash-V}) using V updates is presented in Algorithm~\ref{alg:Nash-V-learning}.

\begin{algorithm}[h]
\caption{Nash V-learning (full-information version)}
\label{alg:Nash-V-learning}
\begin{algorithmic}
\STATE  \textbf{Require:} Learning rate $\{\alpha_t\}_{t\ge1}$ in \eqref{eq:def-alpha-t} and corresponding $\{w_t\}_{t\ge1}$ (cf. Example \ref{example:Nash-V}); $\eta>0$.
\STATE \textbf{Initialize:} Set $V_{h}^0(s)=H-h+1$ for all $(h,s)\in[H]\times\cS$.
\FOR{$t=1,\ldots,T$}
\FOR{$h = H, \ldots, 1$}
\STATE Update policy for all $s\in\cS$~(understanding $w_0\defeq 1$): 
\begin{equation}\label{eq:mu-nu-update-NashV}
\begin{split}
     \mu_h^t(a|s)&\propto_a\exp\paren{\frac{\eta}{w_{t-1}}\sum_{i=1}^{t-1}w_i\brac{\paren{r_h +\P_hV_{h+1}^i}\nu_h^i}(s,a)}\\
     \nu_h^t(a|s)&\propto_b\exp\paren{-\frac{\eta}{w_{t-1}}\sum_{i=1}^{t-1}w_i\brac{\paren{r_h +\P_hV_{h+1}^i}^\top\mu_h^i}(s,b)}.
\end{split}
\end{equation}
\STATE Update V value for all $s\in\cS$:
\begin{align}\label{eq:V-update}
    V_h^{t}(s) \setto (1-\alpha_t) V_h^{t-1}(s) + \alpha_t \brac{\paren{\mu_h^{t}}^\top \paren{r_h + \P_h V_{h+1}^{t}}\nu_h^{t}}(s).
\end{align}
\ENDFOR
\ENDFOR
\end{algorithmic}
\end{algorithm}

\begin{prop}[Equivalence between V update and Q update]
\label{prop:equivalence-NashV-framework}
Nash V-learning (full-information version) in Algorithm \ref{alg:Nash-V-learning} is equivalent to Algorithm~\ref{alg:framework} with the $\matrixgamealg$ as weighted FTRL~\eqref{eq:FTRL-V-learning}.
\begin{proof}
It suffices to show that, for the Q value defined in~\eqref{eq:Q-update} and the V value defined in~\eqref{eq:V-update}, the following holds for all $(h,s,a,b)$ and all $t\in[T]$:
\begin{equation}\label{eq:relation-Q-V}
    Q_h^{t}(s,a,b) = \brac{r_h+ \mathbb{P}_hV_{h+1}^t}(s,a,b).
\end{equation}
Since $\alpha_1=1$, it is not hard to verify that
\begin{equation*}
    Q_h^{1}(s,a,b) = \brac{r_h+ \mathbb{P}_hV_{h+1}^1}(s,a,b),
\end{equation*}
We now prove by induction on both $t$ and $h$. Given that $ Q_h^{1}(s,a,b) = \brac{r_h+ \mathbb{P}_hV_{h+1}^1}(s,a,b)$, it is not hard to verify that $Q_H^{t}(s,a,b) = r_H(s,a,b),~\forall~k\ge 0$. We assume that \eqref{eq:relation-Q-V} holds for $(t-1,h)$ and $(t, h+1)$, then for $(t,h)$, from \eqref{eq:Q-relationship}
\begin{align*}
     & \quad Q_h^{t}(s,a,b) = (1-\alpha_t) Q_h^{t-1}(s,a,b) + \alpha_t \paren{r_h + \P_h[ (\mu_{h+1}^t)^\top Q_{h+1}^t\nu_{h+1}^t ]} (s,a,b)\\
     &\le (1\!-\!\alpha_t)\brac{r_h\!+\! \mathbb{P}_hV_{h+1}^{t-1}}(s,a,b)+\alpha_t \paren{r_h\!+\! \P_h\brac{(\mu_{h+1}^t)^\top \paren{r_{h+1}\!+\!\P_{h+1} V_{h+2}}\nu_{h+1}^t}} (s,a,b)\\
     &\qquad\qquad\qquad\qquad\qquad\qquad\qquad\qquad\qquad\qquad\qquad\qquad\qquad\qquad\qquad\textup{(inductive hypothesis)}\\
     &=\brac{r_h+ \mathbb{P}_h\left((1-\alpha_t)V_{h+1}^{t-1} +\alpha_t (\mu_{h+1}^t)^\top \paren{r_{h+1}\!+\!\P_{h+1} V_{h+2}}\nu_{h+1}^t\right) }(s,a,b)\\
     &= \brac{r_h + \mathbb{P}_h V_{h+1}^t}(s,a,b)~~\textup{(from \eqref{eq:V-update})},
\end{align*}
which completes the proof by induction.
\end{proof}
\end{prop}
\begin{lemma}[Per-state regret bound for Nash V-learning]\label{lemma:regret-Nash-V-learning} Algorithm \ref{alg:Nash-V-learning} achieves the following per-state regret bound:
\begin{equation*}
    \reg_{h}^t \le \frac{(H+1)\log(A\vee B)}{\eta t} + \frac{\eta H^2}{2}, ~~\forall h\in [H], t \ge 1.
\end{equation*}
\begin{proof}
Fix any $(h,s)$. Note that the update of $\set{\mu^t_h(\cdot|s)}_{t\ge 1}$ in~\eqref{eq:mu-nu-update-NashV} is equivalent to the weighted FTRL algorithm (Algorithm~\ref{alg:FTRL}) with loss vectors $g_i=-\brac{Q_h^{i}\nu_h^{i}}(s)$. Thus by Lemma~\ref{lemma:FTRL} we get for any $t\ge 1$ that
\begin{align*}
    \max_{\mu^\dagger\in\Delta_\cA} \sum_{i=1}^t w_i \<\mu^\dagger - \mu_h^i(\cdot|s), \brac{Q_h^i\nu_h^i}(s, \cdot)\>\le \frac{w_{t}}{\eta}\log A + \frac{\eta H^2}{2}\sum_{i=1}^t w_i.
\end{align*}
Further, recalling $\alpha_t^i = w_i \cdot \alpha_t^1$ for $1\le i\le t$, we have
\begin{align*}
    & \quad \reg_{h,\mu}^{t} \le \alpha_t^1\paren{\frac{w_{t}}{\eta}\log A + \frac{\eta H^2}{2}\sum_{i=1}^t w_i} =\frac{\alpha_t^1w_{t}}{\eta}\log A + \frac{\eta H^2}{2}\sum_{i=1}^t \alpha_t^i\\
    &= \frac{\alpha_t}{\eta}\log A + \frac{\eta H^2}{2}
    \le \frac{(H+1)\log A}{\eta t} + \frac{\eta H^2}{2}.
\end{align*}
The similar bound also holds for $\reg_{h,\nu}^t$, and thus we have that
\begin{equation*}
    \reg_{h}^t \le \frac{(H+1)\log(A\vee B)}{\eta t} + \frac{\eta H^2}{2}.
\end{equation*}
\end{proof}
\end{lemma}
\begin{prop}[Guarantee of Nash V-Learning]
\label{prop:Nash-V-learning}
Algorithm \ref{alg:Nash-V-learning} achieves
\begin{equation*}
    \negap(\hat{\mu}^T,\hat{\nu}^T)\le 14\eta H^4\log(T) + \frac{104\log(A\vee B)\log(T)^2H^3}{\eta T}.
\end{equation*}
Specifically, choosing $\eta = \frac{4}{\sqrt{HT}}$, we have
\begin{align*}
    &\negap(\hat{\mu}^T,\hat{\nu}^T) \le \frac{82\log(A\vee B)\log(T)^2 H^{7/2}}{\sqrt{T}}.
\end{align*}
\begin{proof}
From Lemma \ref{lemma:regret-Nash-V-learning}, we can take $\barreg_h^t$ as $\barreg_h^t =  \frac{(H+1)\log(A\vee B)}{\eta t} + \frac{\eta H^2}{2}$. Then from Theorem \ref{theorem:master}, we have
\begin{align*}
    &\quad \negap(\hat{\mu}^T,\hat{\nu}^T) \le 4H\max_h \barreg_{h}^T + 8H^2\paren{1+\frac{1}{H}}^{H}\log(T)\cdot \frac{1}{T}\sum_{t=1}^T\max_{h}\barreg_{h}^t\\
    &\le 4H\paren{\frac{(H+1)\log(A\vee B)}{\eta T} + \frac{\eta H^2}{2}} + 24H^2\log(T)\cdot \frac{1}{T}\sum_{t=1}^T\paren{\frac{(H+1)\log(A\vee B)}{\eta t} + \frac{\eta H^2}{2}}\\
    &\le 14\eta H^4\log(T) + 24H^2\log(T)\cdot \frac{1}{T}\sum_{t=1}^T\frac{(H+1)\log(A\vee B)}{\eta t} + \frac{8H^2\log(A\vee B)}{\eta T}\\
    &\le 14\eta H^4\log(T) + \frac{104\log(A\vee B)\log(T)^2H^3}{\eta T}.
\end{align*}
Thus, choosing $\eta = \frac{4}{\sqrt{HT}}$, we get
\begin{align*}
    &\negap(\hat{\mu}^T,\hat{\nu}^T) \le \frac{82\log(A\vee B)\log(T)^2 H^{7/2}}{\sqrt{T}}.
\end{align*}
\end{proof}
\end{prop}

\subsection{GDA-Critic}
The full description of GDA-Critic (Example \ref{example:GDA-critic}) using V updates is presented in Algorithm \ref{alg:gda-critic}.
\begin{algorithm}[h]
\caption{GDA-Critic}
\label{alg:gda-critic}
\begin{algorithmic}
\STATE  \textbf{Require:} Learning rate $\{\alpha_t\}_{t\ge1}$ (defined in \eqref{eq:def-alpha-t}), and $\eta>0$.
\STATE \textbf{Initialize:}  set $V_{h}^0(s)=H-h+1$ and $\mu^0(\cdot|s),\nu^0(\cdot|s)$ to be uniform for all $(h,s)\in[H]\times\cS$.
\FOR{$t=1,\ldots,T$}
\FOR{$h = H, \ldots, 1$}
\STATE Update policy for all $s\in\cS$:
\begin{equation}\label{eq:mu-nu-update_GDA-critic}
\begin{split}
    \mu_h^t(\cdot|s)& \setto \mathcal{P}_{\Delta_{\cA}}\paren{\mu_h^{t-1}(\cdot|s) + \eta\brac{\paren{r_h + \P_h V_{h+1}^{t-1}}\nu_h^{t-1}}(s,\cdot,b)}\\
    \nu_h^t(\cdot|s)& \setto \mathcal{P}_{\Delta_{\cB}}\paren{\nu_h^{t-1}(\cdot|s) - \eta\brac{\paren{r_h + \P_h V_{h+1}^{t-1}}^\top\mu_h^{t-1}}(s,a,\cdot)}\\
\end{split}
\end{equation}
\STATE Update V value for all $s\in\cS$:
\begin{align*}
    V_h^{t}(s) \setto (1-\alpha_t) V_h^{t-1}(s) + \alpha_t \brac{\paren{\mu_h^{t}}^\top \paren{r_h + \P_h V_{h+1}^{t}}\nu_h^{t}}(s).
\end{align*}
\ENDFOR
\ENDFOR
\end{algorithmic}
\end{algorithm}

Similar as Proposition~\ref{prop:equivalence-NashV-framework} (with the same proof), the following equivalence between Q updates and V updates also holds for GDA-Critic.
\begin{prop}[Equivalence between Q updates and V updates for GDA-Critic]
Algorithm \ref{alg:gda-critic} is equivalent to our algorithm framework (Algorithm \ref{alg:framework}) with the $\matrixgamealg$ instantiated as
\eqref{eq:GDA-example}.
\end{prop}
\begin{lemma}[Per-state regret bound for GDA-Critic]\label{lemma:regret-GDA-critic} Algorithm \ref{alg:gda-critic} achieves the following per-state regret bound:
\begin{equation*}
    \reg_h^t \le \frac{2(H+1)}{\eta t} + \frac{\eta (A\vee B)H^2}{2}.
\end{equation*}
\begin{proof}
Fix any $(h,s)$ and $t\ge 1$. We apply Lemma \ref{lem:gd.regret} to the projected gradient descent (or ascent) update~\eqref{eq:mu-nu-update_GDA-critic}, with weights $w_i = \alpha_t^i$ and loss vectors $g_i$'s $-\brac{Q_h^{i}\nu_h^{i}}(s)$ or $\brac{\paren{Q_h^i}^\top\mu_h^{i}}(s)$ respectively. For the gradient ascent update for $\mu_h^t(\cdot|s)$, we get
\begin{align*}
    \reg_{h,\mu}^t(s) &\le \frac{\alpha_t^t}{2\eta}\cdot 4 + \frac{\eta\paren{\sum_{i=1}^t\alpha_t^i}AH^2}{2} = \frac{2}{\eta}\frac{H+1}{H+t} + \frac{\eta AH^2}{2}\\
    \reg_{h,\nu}^t(s) &\le \frac{\alpha_t^t}{2\eta}\cdot 4 + \frac{\eta\paren{\sum_{t=1}^T\alpha_t^i}BH^2}{2} = \frac{2}{\eta}\frac{H+1}{H+t} + \frac{\eta BH^2}{2}\\
    \Longrightarrow~~\reg_h^t &\le \frac{2(H+1)}{\eta t} + \frac{\eta (A\vee B)H^2}{2}.
\end{align*}
\end{proof}
\end{lemma}
\begin{prop}[Guarantee of GDA-Critic]
\label{prop:GDA-critic}
Algorithm \ref{alg:gda-critic} achieves
\begin{equation*}
    \negap(\hat{\mu}^T,\hat{\nu}^T)\le 14\eta(A\vee B) H^4\log(T) + \frac{208\log(T)^2H^3}{\eta T}.
\end{equation*}
Specifically, picking $\eta = \frac{4}{\sqrt{(A\vee B)HT}}$ yields
\begin{align*}
    &\negap(\hat{\mu}^T,\hat{\nu}^T) \le \frac{108\log(T)^2\sqrt{A\vee B} H^{7/2}}{\sqrt{T}}.
\end{align*}
\begin{proof}
 From Lemma \ref{lemma:regret-GDA-critic}, we can take $\barreg_h^t$ as $\frac{2(H+1)}{\eta t} + \frac{\eta (A\vee B)H^2}{2}$, then from Theorem \ref{theorem:master}
\begin{align*}
    &\quad \negap(\hat{\mu}^T,\hat{\nu}^T) \le 4H\max_h \barreg_{h}^T + 8H^2\paren{1+\frac{1}{H}}^{H}\log(T)\cdot \frac{1}{T}\sum_{t=1}^T\max_{h}\barreg_{h}^t\\
    &\le 4H\paren{\frac{2(H+1)}{\eta T} + \frac{\eta (A\vee B)H^2}{2}} + 24H^2\log(T)\cdot \frac{1}{T}\sum_{t=1}^T\paren{\frac{2(H+1)}{\eta t} + \frac{\eta (A\vee B)H^2}{2}}\\
    &\le 14\eta(A\vee B) H^4\log(T) + 48H^2\log(T)\cdot \frac{1}{T}\sum_{t=1}^T\frac{(H+1)}{\eta t} + \frac{16H^2}{\eta T}\\
    &\le 14\eta(A\vee B) H^4\log(T) + \frac{208\log(T)^2H^3}{\eta T}.
\end{align*}
Thus, pick $\eta = \frac{4}{(A\vee B)\sqrt{HT}}$, we get
\begin{align*}
    &\negap(\hat{\mu}^T,\hat{\nu}^T) \le \frac{108\log(T)^2\sqrt{A\vee B} H^{7/2}}{\sqrt{T}}.
\end{align*}
\end{proof}
\end{prop}

\subsection{Nash Q-Learning (full-information version)}
The Nash Q-Learning algorithm (Example~\ref{example:Nash-Q}) is described in Algorithm~\ref{alg:Nash-Q}.
\begin{algorithm}[h]
\caption{Nash Q-Learning}
\label{alg:Nash-Q}
\begin{algorithmic}
\STATE \textbf{Require:} Learning rate (For Nash Q-learning) $\{\beta_t = \alpha_t\}$;
\STATE \textbf{Initialize:} $Q_h^0(s,a,b)\setto H-h+1$ for all $(h,s,a,b)$.
\FOR{$k=1,\dots,K$}
\FOR{$h=H,\dots,1$}
\STATE Update policy for all $s\in\cS$:
\begin{align}\label{eq:Nash-Q-mu-nu-update}
    (\mu_h^t(\cdot|s), \nu_h^t(\cdot|s)) \setto \NE(Q_h^{t-1}(s,\cdot,\cdot)).
\end{align}
\STATE Update Q value for all $(s,a,b)\in\cS\times\cA\times\cB$:
\begin{align*}
        Q_h^{t}(s,a,b) \setto (1-\alpha_t) Q_h^{t-1}(s,a,b) + \alpha_t \paren{r_h + \P_h[ (\mu_{h+1}^t)^\top Q_{h+1}^t\nu_{h+1}^t ]} (s,a,b).
\end{align*}
\ENDFOR
\ENDFOR
\end{algorithmic}
\end{algorithm}
\begin{lemma}[Per-state regret bound for Nash Q-Learning]
\label{lemma:reg-Nash-Q}
Algorithm \ref{alg:Nash-Q} achieves the following per-state regret bound:
\begin{equation*}
    \reg_h^t \le \frac{(H+1)^2}{H+t} ,~~~\forall~h\in[H],~t\ge 1.
\end{equation*}
\begin{proof}
We have
\begin{align*}
    \reg_{h,\mu}^t(s) &= \max_{\mu^\dagger\in\Delta_\cA} \sum_{i=1}^t \alpha_t^i \<\mu^\dagger - \mu_h^i(\cdot|s), \brac{Q_h^i\nu_h^i}(s, \cdot)\>\\
    &= \max_{\mu^\dagger\in\Delta_\cA} \sum_{i=1}^t \alpha_t^i \<\mu^\dagger - \mu_h^i(\cdot|s), \brac{Q_h^{i-1}\nu_h^i}(s, \cdot)\> + \sum_{i=1}^t\alpha_t^i\|Q_h^i-Q_h^{i-1}\|_\infty\\
    &\le \sum_{i=1}^t \alpha_t^i \underbrace{ \max_{\mu^\dagger\in\Delta_\cA} \<\mu^\dagger - \mu_h^i(\cdot|s), \brac{Q_h^{i-1}\nu_h^i}(s, \cdot)\>}_{= 0 \textup{ from \eqref{eq:Nash-Q-mu-nu-update}}} + \sum_{i=1}^t\alpha_t^i\|Q_h^i-Q_h^{i-1}\|_\infty\\
    &= \sum_{i=1}^t\alpha_t^i\|Q_h^i-Q_h^{i-1}\|_\infty.
\end{align*}
The same bound also holds for $\reg_{h,\nu}^t(s)$, thus
\begin{equation*}
    \reg_{h}^t\le \sum_{i=1}^t\alpha_t^i\|Q_h^i-Q_h^{i-1}\|_\infty.
\end{equation*}
Since
\begin{align*}
    &Q_h^{i}(s,a,b) = (1-\alpha_i) Q_h^{i-1}(s,a,b) + \alpha_i \paren{r_h + \P_h[ (\mu_{h+1}^i)^\top Q_{h+1}^i\nu_{h+1}^i ]} (s,a,b)\\
    \Longrightarrow & |Q_h^{i}(s,a,b) -  Q_h^{i-1}(s,a,b)|\le \alpha_i \|Q_h^{i-1} - \paren{\paren{r_h + \P_h[ (\mu_{h+1}^i)^\top Q_{h+1}^i\nu_{h+1}^i ]}}\|_\infty
    \le \alpha_i H\\
    \Longrightarrow&\|Q_h^{i}-Q_h^{i-1}\|_\infty\le \alpha_i H ,
\end{align*}
substituting this into the above equations we have that
\begin{equation*}
    \reg_h^t \le H \sum_{i=1}^t \alpha_t^i \alpha_i\le \frac{(H+1)^2}{H+t}~~\textup{(From Lemma \ref{lemma:alpha-convolution} (b))},
\end{equation*}
which completes the proof.
\end{proof}
\end{lemma}
\begin{prop}[Guarantee for Nash Q-Learning]
\label{prop:Nash-Q}
Algorithm \ref{alg:Nash-Q} achieves
\begin{equation*}
    \negap(\hat{\mu}^T,\hat{\nu}^T)\le \frac{112\log(T)^2H^4}{T}.
\end{equation*}
\begin{proof}
From Theorem \ref{theorem:master} and Lemma \ref{lemma:reg-Nash-Q} we have that
\begin{align*}
     &\quad \negap(\hat{\mu}^T,\hat{\nu}^T)\le 4H\max_h \barreg_{h}^T + 8H^2\paren{1+\frac{1}{H}}^{H}\log(T)\cdot \frac{1}{T}\sum_{t=1}^T\max_{h}\barreg_{h}^t\\
     &\le 4H\frac{(H+1)^2}{H+T} + 24H^2\log(T)\cdot \frac{1}{T}\sum_{t=1}^T\frac{(H+1)^2}{H+t}\le \frac{112\log(T)^2H^4}{T}.
\end{align*}
\end{proof}
\end{prop}

\subsection{Nash Policy Iteration}
\begin{algorithm}[h]
\caption{Nash Policy Iteration (Nash-PI)}
\label{alg:Nash-PI}
\begin{algorithmic}
\STATE \textbf{Initialize:} $Q_h^0(s,a,b)\setto H-h+1$ for all $(h,s,a,b)$.
\FOR{$t=1,\dots,T$}
\FOR{$h=H,\dots,1$}
\STATE Update policy for all $s\in\cS$:
\begin{align*}
    (\mu_h^t(\cdot|s), \nu_h^t(\cdot|s)) \setto \NE(Q_h^{t-1}(s,\cdot,\cdot)).
\end{align*}
\STATE Update Q value for all $(s,a,b)\in\cS\times\cA\times\cB$:
\begin{align}\label{eq:Nash-PI-Q-update}
    Q_h^t(s,a,b) \setto \brac{r_h + \P_h[(\mu_{h+1}^t)^\top Q_{h+1}^t\nu_{h+1}^t] }(s,a,b) .
\end{align}
\ENDFOR
\ENDFOR
\end{algorithmic}
\end{algorithm}

The full description of Nash Policy Iteration (Nash-PI, Example \ref{example:Nash-PI}) is presented in Algorithm \ref{alg:Nash-PI}. 

Note that from \eqref{eq:Nash-PI-Q-update}, we have that $Q_h^k$ equals to $Q_h^{\mu^k\times\nu^k}$. Based on this observation, we have the following lemma.
\begin{lemma}[Exact learning of Q functions]
\label{lemma:Nash-PI-Q-h}
For Algorithm~\ref{alg:Nash-PI}, we have for any $h\in[H]$ and $t\ge H-h+1$ that
\begin{equation*}
    Q_h^t(s,a,b) = Q_h^\star(s,a,b),~~\forall~(s,a,b)\in\cS\times\cA\times\cB.
\end{equation*}
\begin{proof}
We prove this by backward induction over $h$. For $h = H$, we have that
\begin{equation*}
    Q_H^{t}(s,a,b) = r_H(s,a,b),~~\forall t \ge 1.
\end{equation*}
Assume that for $h+1$, the condition holds, then for time horizon $h$ and iteration step $t\ge H-h+1$, we have that
\begin{align*}
    Q_h^t(s,a,b) &= \brac{r_h + \P_h[(\mu_{h+1}^t)^\top Q_{h+1}^t\nu_{h+1}^t] }(s,a,b)\\
    & = \brac{r_h + \P_h[(\mu_{h+1}^t)^\top Q_{h+1}^\star\nu_{h+1}^t] }(s,a,b).
\end{align*}
Additionally, from the inductive hypothesis
\begin{align*}
    (\mu_{h+1}^t(\cdot|s), \nu_{h+1}^t(\cdot|s)) =  \NE(Q_{h+1}^{t-1}(s,\cdot,\cdot)) = \NE(Q_{h+1}^\star(s,\cdot,\cdot)),
\end{align*}
we have that
\begin{align*}
    [(\mu_{h+1}^t)^\top Q_{h+1}^\star\nu_{h+1}^t](s) = V_{h+1}^\star(s).
\end{align*}
Thus
\begin{align*}
    Q_h^t(s,a,b)
    & = \brac{r_h + \P_h[(\mu_{h+1}^t)^\top Q_{h+1}^\star\nu_{h+1}^t] }(s,a,b)\\
    & = \brac{r_h + \P_hV_{h+1}^\star}(s,a,b) = Q_h^\star(s,a,b),
\end{align*}
which completes the proof.
\end{proof}
\end{lemma}

\begin{prop}[Guarantee for Nash-PI]
\label{prop:Nash-PI}
Algorithm~\ref{alg:Nash-PI} achieves $\negap(\hat{\mu}^T,\hat{\nu}^T) = 0$ for $T\ge H$.
\begin{proof}
For this proposition we will not proof by calling Theorem \ref{theorem:master}, but instead directly apply Lemma \ref{lem.duality}, which is an auxiliary lemma for proving Theorem \ref{theorem:master}.

Note that Nash-PI corresponds is equivalent to using $\beta_t=1$ in Algorithm \ref{alg:framework}, so that $\beta_t^i = \mathbf{1}\{i=t\}$. From Lemma \ref{lemma:Nash-PI-Q-h} we have that
\begin{equation*}
    Q_h^t = Q_h^\star, ~~\forall~t\ge H, h\in [H].
\end{equation*}
Thus for $t\ge H$,
\begin{align*}
&\quad\max_{\mu^\dagger\in\Delta_\cA,\nu^\dagger\in\Delta_\cB}\brac{(\mu^\dagger)^\top Q_h^\star \hat\nu_h^T -\paren{\hat\mu_h^T}^\top Q_h^\star \nu^\dagger }(s) \\
    &= \max_{\mu^\dagger\in\Delta_\cA,\nu^\dagger\in\Delta_\cB}\brac{(\mu^\dagger)^\top Q_h^T \nu_h^T -\paren{\mu_h^T}^\top Q_h^T \nu^\dagger }(s) +2\sum_{t=1}^T\beta_T^t\delta_h^{t} \\
    &=0, ~~\forall h.
\end{align*}
Then applying Lemma \ref{lem.duality}, we obtain
\begin{align*}
    & \quad \negap(\hat{\mu}^T,\hat{\nu}^T)\le \paren{V_1^{\dagger,\hat{\nu}^T}(s_1) - V_1^\star(s_1)} + \paren{V_1^\star(s_1)- V_1^{\hat{\mu}^T,\dagger}(s_1)}\\
    &\le \sum_{h=1}^H\max_{\mu^\dagger\in\Delta_\cA,\nu^\dagger\in\Delta_\cB}\brac{(\mu^\dagger)^\top Q_h^\star \hat\nu_h^T -\paren{\hat\mu_h^T}^\top Q_h^\star \nu^\dagger }(s) =0,~~~\textup{ for } T \ge H.
\end{align*}
This is the desired result.
\end{proof}
\end{prop}



\section{Proof of Theorem~\ref{theorem:oftrl-main}}
\label{appendix:proof-oftrl-main}


\begin{algorithm}[t]
\caption{OFTRL for two-player zero-sum Markov Games}
\label{algorithm:oftrl-mg}
\begin{algorithmic}[1]
\STATE \textbf{Initialize:} $Q_h^0(s,a,b)\setto H-h+1$ for all $(h,s,a,b)$.
\FOR{$t=1,\dots,T$}
\FOR{$h=H,\dots,1$}
\STATE Update policies for all $s\in\cS$ by OFTRL:
\begin{align*}
    & \textstyle \mu^{t}_h(a | s) \propto_a \exp\paren{ (\eta/w_t) \cdot \brac{ \sum_{i=1}^{t-1} w_i(Q_h^i\nu_h^i)(s, a) + w_{t-1}(Q_h^{t-1}\nu_h^{t-1})(s, a) } }, \\
    & \textstyle \nu^{t}_h(b | s) \propto_b \exp\paren{ -(\eta/w_t) \cdot \brac{ \sum_{i=1}^{t-1} w_i((Q_h^i)^\top\mu_h^i)(s, b) + w_{t-1}((Q_h^{t-1})^\top\mu_h^{t-1})(s, b) } }. 
\end{align*}
\STATE Update Q-value for all $(s,a,b)\in\cS\times\cA\times\cB$:
\begin{align}\label{eq:Q-update-2}
    Q_h^{t}(s,a,b) \setto (1-\alpha_t) Q_h^{t-1}(s,a,b) + \alpha_t \paren{r_h + \P_h[ (\mu_{h+1}^t)^\top Q_{h+1}^t\nu_{h+1}^t ]} (s,a,b).
\end{align}
\ENDFOR
\ENDFOR
\STATE Output state-wise average policy:
\vspace{-5pt}
\begin{equation*}
    \widehat{\mu}_h^T(\cdot|s) \setto \sum_{t=1}^T\alpha_T^t\mu_{h}^{t}(\cdot|s),\quad 
    \widehat{\nu}_h^T(\cdot|s)\setto\sum_{t=1}^T\alpha_T^t \nu_h^{t}(\cdot|s).
\end{equation*}
\end{algorithmic}
\end{algorithm}

In this section we prove Theorem~\ref{theorem:oftrl-main}. The full algorithm box of OFTRL for Markov Games is provided in Algorithm~\ref{algorithm:oftrl-mg}.

We aim to show that
\begin{align}
\label{equation:oftrl-bound-with-rate}
\begin{aligned}
    & \quad \negap(\hat{\mu}^T, \hat{\nu}^T) \\
    & \le \cO\paren{ H^{14/3} (\log(A\vee B))^{5/6} (\log T)^{11/6} \cdot T^{-5/6} + H^5\log(A\vee B)(\log T)^2\cdot T^{-1} }.
\end{aligned}
\end{align}

\paragraph{Bounding per-state regret}
We first bound $\reg_{\nu, h}^t(s)$, i.e. the per-state regret for the min-player, for any fixed $(h,s,t)\in[H]\times[S]\times[T]$. (The bound for $\reg_{\mu, h}^t(s)$ follows similarly.) This is the main part of this proof. 

Throughout this part, we will fix $(h,s)$ and omit these subscripts within the policies and Q functions, so that $\nu_h^t(\cdot|s)$ will be abbreviated as $\nu^t$ (and similarly for $\mu^t$ and $Q^t$). We will also overload $T\ge 1$ to be any positive integer (instead of the fixed total number of iterations).

Observe that the above update for $\nu^t$ is equivalent to the OFTRL algorithm (Algorithm~\ref{alg:OFTRL}) with loss vectors $g_t=w_t(Q^t)^\top \mu^t$ (understanding $g_0=0$ and $Q_h^0=0$), prediction vector $M_t=g_{t-1}=w_{t-1}(Q^{t-1})^\top\mu^{t-1}$, and learning rate $\eta_t=\eta/w_t$. Therefore we can apply the regret bound for OFTRL in Lemma~\ref{lemma:OFTRL-auxillary} and obtain for any $T\ge 1$ that
\begin{equation}
\label{equation:oftrl-regret-initial}
\begin{aligned}
    & \quad \max_{\nu^\dagger\in\Delta_{\cB}} \sum_{t=1}^T w_t \<\nu^t - \nu^\dagger, (Q^t)^\top\mu^t\> = \max_{\nu^\dagger\in\Delta_{\cB}} \sum_{t=1}^T \<\nu^t - \nu^\dagger, g_t\>\\
    & \le \frac{\log B}{\eta_T} + 
    \sum_{t=1}^{T} \eta_t \|g_t-M_t\|_\infty^2  -\sum_{t=1}^{T-1} \frac{1}{8\eta_t}\|\nu^t-\nu^{t+1}\|_1^2 \\
    & = \frac{\log B\cdot w_T}{\eta} + 
    \sum_{t=1}^{T} \frac{\eta}{w_t} \|w_t(Q^t)^\top\mu^t-w_{t-1}(Q^{t-1})^\top\mu^{t-1}\|_\infty^2  -\sum_{t=2}^{T} \frac{w_{t-1}}{8\eta}\|\nu^t-\nu^{t-1}\|_1^2 .
\end{aligned}
\end{equation}
We now relate the terms above to the stability of $\set{\mu^t}_{t\ge 1}$ (the other player's policies). Let \begin{align*}
    \Delta_t\defeq \linf{w_tQ^t - w_{t-1}Q^{t-1}}
\end{align*}
for all $t\ge 1$ for shorthand, where $\linf{\cdot}$ for a matrix denotes its infinity norm (i.e. entry-wise max absolute value). Then we have
\begin{align*}
    & \quad \|w_t(Q^t)^\top\mu^t-w_{t-1}(Q^{t-1})^\top\mu^{t-1}\|_\infty^2 \\
    & \le 2\linf{(w_tQ^t - w_{t-1}Q^{t-1})^\top\mu^{t-1}}^2 + 2\linf{(w_{t-1}Q^{t-1})^\top(\mu^t - \mu^{t-1})}^2 \\
    & \le 2\Delta_t^2 + 2w_{t-1}^2 \linf{Q^{t-1}}^2 \lone{\mu^t - \mu^{t-1}}^2 \\ 
    & \le 2\Delta_t^2 + 2w_{t-1}^2H^2 \lone{\mu^t - \mu^{t-1}}^2\indic{t\ge 2}.
\end{align*}
By symmetry, the similar bound also holds for $\set{\nu^t}$, from which we obtain for any $t\ge 2$ that
\begin{align*}
    -w_{t-1}\lone{\nu^t - \nu^{t-1}}^2 \le \frac{\Delta_t^2}{w_{t-1}H^2} - \frac{1}{2w_{t-1}H^2} \linf{w_tQ^t\nu^t-w_{t-1}Q^{t-1}\nu^{t-1}}^2.
\end{align*}
Plugging the above two bounds into~\eqref{equation:oftrl-regret-initial}, we get
\begin{align}
    \label{equation:oftrl-regret-middle}
    \begin{aligned}
    & \quad \max_{\nu^\dagger\in\Delta_{\cB}} \sum_{t=1}^T w_t \<\nu^t - \nu^\dagger, (Q^t)^\top\mu^t\> \\
    & \le \frac{\log B\cdot w_T}{\eta} + 
    \sum_{t=1}^{T} \brac{ \frac{2\eta}{w_t} \Delta_t^2 + \frac{2\eta H^2w_{t-1}^2}{w_t} \lone{\mu^t - \mu^{t-1}}^2\indic{t\ge 2} } \\
    & \qquad + \sum_{t=2}^{T} \brac{ \frac{\Delta_t^2}{8\eta H^2 w_{t-1}} - \frac{1}{16\eta H^2w_{t-1}} \linf{w_tQ^t\nu^t-w_{t-1}Q^{t-1}\nu^{t-1}}^2 } \\
    & \stackrel{(i)}{\le} \frac{\log B\cdot w_T}{\eta} + 
    \underbrace{\sum_{t=1}^{T} \brac{ \frac{2\eta}{w_t} + \frac{1}{8\eta H^2w_{t-1}}\indic{t\ge 2}} \Delta_t^2}_{\defeq \err_T} + \underbrace{4\eta^2H^2\cdot \sum_{t=2}^T \frac{1}{2\eta/w_{t}} \lone{\mu^t - \mu^{t-1}}^2}_{\defeq \stab_{t}} \\
    & \qquad - \sum_{t=2}^T \frac{1}{16\eta H^2w_{t-1}} \linf{w_tQ^t\nu^t-w_{t-1}Q^{t-1}\nu^{t-1}}^2.
\end{aligned}
\end{align}
Above, (i) rearranges terms and used the fact that $w_{t-1}\le w_t$.

The following lemma (proof deferred to Section~\ref{appendix:proof-err-t}) bounds term $\err_T$.
\begin{lemma}[Bound on $\err_T$]
\label{lemma:err-t}
Suppose $\eta \le 1/H$. Then for any $T\ge 1$, we have
\begin{align*}
    \alpha_T^1\cdot \err_T \le \frac{192H^2}{\eta T}.
\end{align*}
\end{lemma}

To bound term $\stab_T$, note that it is exactly the total distance (in squared $L_1$ norm) of the sequence $\set{\mu^t}_{t\ge 1}$, which itself follows an OFTRL algorithm with loss sequence $g_t'\defeq -w_tQ^t\nu^t$, $M_t'\defeq g_{t-1}'$, and $\eta_t=\eta/w_t$. Therefore we can apply the stability bound~\eqref{equation:l1-bound-3} in Lemma~\ref{lemma:oftrl-stability} to obtain that
\begin{align*}
    & \quad \stab_T = 4\eta^2H^2 \cdot \sum_{t=2}^T \frac{1}{2\eta_t}\lone{\mu^t - \mu^{t-1}}^2 \\
    & \le 4\eta^2H^2 \paren{ \frac{2\log A}{\eta_T} + \sum_{t=1}^{T-1} (1 + \eta_tG_t') \linf{g_t' - g_{t-1}'} + \linf{g_{T-1}'} } \\
    & = 4\eta^2H^2 \paren{ \frac{2w_T\log A}{\eta} + \sum_{t=1}^{T-1} (1 + \eta H) \linf{w_tQ^t\nu^t - w_{t-1}Q^{t-1}\nu^{t-1}} + \linf{w_{T-1}Q^{T-1}\nu^{T-1}} } \\
    & \stackrel{(i)}{\le} 4\eta^2H^2 \paren{ \frac{2w_T\log A}{\eta} + 2\sum_{t=1}^{T-1} \linf{w_tQ^t\nu^t - w_{t-1}Q^{t-1}\nu^{t-1}} + w_{T-1}H } \\
    & \stackrel{(ii)}{\le} 4\eta^2H^2 \paren{ \frac{4w_T\log A}{\eta} + 2\sum_{t=1}^{T-1} \linf{w_tQ^t\nu^t - w_{t-1}Q^{t-1}\nu^{t-1}}},
\end{align*}
where here we take $G_t'=w_tH\ge \linf{g_t'-g_{t-1}'}$, (i) holds whenever $\eta\le 1/H$, and (ii) follows as $w_{T-1}H\le w_T/\eta\le 2w_T\log A/\eta$.

Plugging the above bounds into~\eqref{equation:oftrl-regret-middle} yields that for any $T\ge 1$,
\begin{align*}
    & \quad \reg_{\nu, h}^T(s) =  \max_{\nu^\dagger\in\Delta_{\cB}} \sum_{t=1}^T \underbrace{\alpha_T^t}_{\alpha_T^1\cdot w_t} \<\nu^t - \nu^\dagger, (Q^t)^\top\mu^t\> = \alpha_T^1 \max_{\nu^\dagger\in\Delta_{\cB}} \sum_{t=1}^T w_t \<\nu^t - \nu^\dagger, (Q^t)^\top\mu^t\> \\
    & \le \frac{\log B\cdot (\alpha_T^1 w_T)}{\eta} + \alpha_T^1\err_T + \alpha_T^1\brac{\stab_T - \sum_{t=2}^T \frac{1}{16\eta H^2w_{t-1}} \linf{w_tQ^t\nu^t-w_{t-1}Q^{t-1}\nu^{t-1}}^2} \\
    & \le \frac{\log B\cdot \alpha_T^T}{\eta} + \frac{192H^2}{\eta T} + \alpha_T^1\bigg[ 16\eta H^2 w_T\log A \\
    & \qquad + 8\eta^2H^2\sum_{t=1}^{T-1} \linf{w_tQ^t\nu^t - w_{t-1}Q^{t-1}\nu^{t-1}} - \sum_{t=2}^T \frac{1}{16\eta H^2w_{t-1}} \linf{w_tQ^t\nu^t-w_{t-1}Q^{t-1}\nu^{t-1}}^2 \bigg] \\
    & \stackrel{(i)}{\le} \frac{\log B\cdot \alpha_T^T}{\eta} + \frac{192H^2}{\eta T} + \alpha_T^1\bigg[ 32\underbrace{\eta H^2}_{\le 1/\eta} w_T\log A \\
    & \qquad + \sum_{t=2}^{T-1} \paren{ 8\eta^2H^2\linf{w_tQ^t\nu^t - w_{t-1}Q^{t-1}\nu^{t-1}} - \frac{1}{16\eta H^2w_{t-1}} \linf{w_tQ^t\nu^t-w_{t-1}Q^{t-1}\nu^{t-1}}^2 } \bigg] \\
    & \stackrel{(ii)}{\le} \frac{33\log (A\vee B)\cdot \alpha_T^T}{\eta} + \frac{192H^2}{\eta T} + \alpha_T^1 \sum_{t=2}^{T-1} 256\eta^5H^6w_{t-1} \\
    & \le \frac{33\log (A\vee B)\cdot \alpha_T^T}{\eta} + \frac{192H^2}{\eta T} + 256\eta^5H^6 \underbrace{\sum_{t=2}^{T-1} \alpha_T^{t-1}}_{\le 1} \\
    & \stackrel{(iii)}{\le} C\brac{ \frac{H^2\log (A\vee B)}{\eta T} + \eta^5H^6 }.
\end{align*}
Above, (i) used the fact that $8\eta^2H^2\linf{w_1Q^1\nu^1}\le 8\eta^2H^3w_1\le 8\eta^2H^3w_T\le 16\eta H^2w_T\log A$, (ii) used the fact that $8\eta^2H^2z - z^2/(16\eta H^2w_{t-1})\le 256\eta^5H^6w_{t-1}$ by the AM-GM inequality, and (iii) used the fact that $\alpha_T^T=\alpha_T=(H+1)/(H+T)\le 2H/T$, where $C\le 256$ is an absolute constant. 

By symmetry, the same regret bound also holds for $\reg_{\mu, h}^T(s)$, which gives that for any $t\ge 1$
\begin{align*}
    \reg_h^t \defeq \max_{s\in\cS} \max\set{\reg_{\mu, h}^t(s), \reg_{\nu, h}^t(s)} \le \underbrace{C\brac{ \frac{H^2\log (A\vee B)}{\eta t} + \eta^5H^6 }}_{\defeq \barreg_h^t}.
\end{align*}
Note that $\barreg_h^t$ is decreasing in $t$. This is the desired regret bound.

\paragraph{Performance of output policy}
As our algorithm chooses $\beta_t=\alpha_t=(H+1)/(H+t)$, we can invoke Theorem~\ref{theorem:master} with $\cbeta=1+1/H\ge \sum_{t=j}^\infty\alpha_t^j$ (by Lemma~\ref{lemma:alpha-ti}) so that $\cbeta^H=(1+1/H)^H\le e\le 3$. Further,
by the above regret bound,
\begin{align*}
    \max_{h\in[H]} \barreg_h^t \le C\brac{ \frac{H^2\log (A\vee B)}{\eta t} + \eta^5H^6 }.
\end{align*}
Plugging this into Theorem~\ref{theorem:master} yields that the output policy $(\hat{\mu}^T, \hat{\nu}^T)$ satisfies
\begin{align*}
    & \quad \negap(\hat{\mu}^T, \hat{\nu}^T) \\
    & \le \cO\paren{ H\max_{h\in[H]} \barreg_h^T + H^2 \cbeta^H \cdot \frac{\log T}{T} \sum_{t=1}^T \max_{h\in[H]}\barreg_h^t } \\
    & \le H \cdot \cO\paren{ \frac{H^2\log (A\vee B)}{\eta T} + \eta^5H^6 } + H^2 \frac{\log T}{T} \cdot \cO\paren{ \frac{H^2\log (A\vee B)\log T}{\eta} + \eta^5H^6 T} \\
    & = \cO\paren{ \frac{H^4\log(A\vee B)(\log T)^2}{\eta T} + \eta^5H^8\log T }.
\end{align*}
Choosing $\eta=(\log T\log(A\vee B)/H^4T)^{1/6}\wedge (1/H)$, we get
\begin{align*}
    \negap(\hat{\mu}^T, \hat{\nu}^T) \le \cO\paren{ H^{14/3}(\log (A\vee B))^{5/6}(\log T)^{11/6}\cdot T^{-5/6} + H^5\log(A\vee B)(\log T)^2/T }.
\end{align*}
This proves~\eqref{equation:oftrl-bound-with-rate} and thus Theorem~\ref{theorem:oftrl-main}.



\subsection{Proof of Lemma~\ref{lemma:err-t}}
\label{appendix:proof-err-t}
Recall our notation $Q^t\defeq Q^t_h(s,\cdot,\cdot)\in[0,H]^{A\times B}$ for some fixed $(h,s)\in[H]\times \cS$. We first note that, for any $t\ge 2$,
\begin{align*}
    & \quad \linf{w_tQ^t - w_{t-1}Q^{t-1}}^2 \le 2\linf{w_tQ^t - w_{t-1}Q^t}^2 + 2\linf{w_{t-1}(Q^t - Q^{t-1})}^2 \\
    & \le 2(w_t - w_{t-1})^2 H^2 + 2w_{t-1}^2 \alpha_t^2 H^2 \\
    & = 2w_{t-1}^2 H^2 \brac{\alpha_t^2 + \frac{H^2}{(t-1)^2}} \le 2w_{t-1}^2 H^2 \cdot \frac{8H^2}{t^2} = 16w_{t-1}^2 H^4 / t^2.
\end{align*}
For $t=1$, we have $\linf{w_tQ^t - w_{t-1}Q^{t-1}}^2\le w_1^2H^2=H^2$. Substituting this into the expression of $\err_T$ gives
\begin{align*}
    & \quad \alpha_T^1 \err_T \\
    & = \alpha_T^1 \sum_{t=1}^T \paren{ \frac{2\eta}{w_t} + \frac{1}{8\eta w_{t-1}H^2}\indic{t\ge 2} } \cdot \paren{ H^2\indic{t=1} + 16w_{t-1}^2H^4/t^2 \cdot \indic{t\ge 2}} \\
    & = 2\eta \alpha_T^1 H^2 + \alpha_T^1 \sum_{t=2}^T \paren{ \frac{2\eta w_{t-1}^2}{w_t} + \frac{w_{t-1}}{8\eta H^2} } \cdot \frac{16H^4}{t^2} \\
    & \stackrel{(i)}{\le} 2\eta \alpha_T^1 H^2 + \sum_{t=2}^T \paren{2\eta \alpha_T^t H^2 + \frac{\alpha_T^t}{8\eta} } \cdot \frac{16H^2}{t^2} \\
    & \stackrel{(ii)}{\le} 2\eta \alpha_T^1 H^2 + \sum_{t=2}^T \alpha_T^t \cdot  \frac{3}{\eta} \cdot \frac{16H^2}{t^2} \\
    & \stackrel{(iii)}{\le} \frac{48H^2}{\eta} \sum_{t=1}^T \alpha_T^t \cdot \frac{1}{t^2} \stackrel{(iv)}{\le} \frac{192 H^2}{\eta T}.
\end{align*}
Above, (i) used $w_{t-1}\le w_t$ and $\alpha_T^1 w_t = \alpha_T^t$; (ii) used the fact that $2\eta H^2 \le 2/\eta$ (as $\eta\le 1/H$) and thus $2\eta H^2 + 1/(8\eta) \le (2+1/8)/\eta \le 3/\eta$; (iii) used the fact that $2\eta H^2\le 48 H^2/\eta$ which also follows from $\eta\le 1/H\le 1$; (iv) used Lemma~\ref{lemma:alpha-convolution}(a). This is the desired result.
\qed

\section{A modified OFTRL algorithm with $\tO(T^{-1})$ rate}\label{apdx:modified-OFTRL}

In this section we show that a slightly modified OFTRL algorithm (described Algorithm \ref{algorithm:modified-oftrl-q}) achieves $\tO(T^{-1})$ convergence rate for finding NE in two-player zero-sum Markov Games, improving over the $\tO(T^{-5/6})$ of Algorithm~\ref{algorithm:oftrl-mg}.

\begin{algorithm}[htbp]
\caption{Modified OFTRL}
\label{algorithm:modified-oftrl-q}
\begin{algorithmic}[1]
\STATE \textbf{Initialize:} $\up{Q}_h^0(s,a,b)\setto H-h+1, \low{Q}_h^0 \setto 0$ for all $(h,s,a,b)$.
\FOR{$t=1,\dots,T$}
\FOR{$h=H,\dots,1$}
\STATE Update policies for all $s\in\cS$ by OFTRL:
\vspace{-5pt}
\begin{align*}
    & \textstyle \mu^{t}_h(a | s) \propto_a \exp\paren{ (\eta/w_t) \cdot \brac{ \sum_{i=1}^{t-1} w_i(\up{Q}_h^i\nu_h^i)(s, a) + w_{t-1}(\up{Q}_h^{t-1}\nu_h^{t-1})(s, a) } };\\
     & \textstyle \nu^{t}_h(b | s) \propto_b \exp\paren{ -(\eta/w_t) \cdot \brac{ \sum_{i=1}^{t-1} w_i((\low Q_h^i)^\top\mu_h^i)(s, b) + w_{t-1}((\low Q_h^{t-1})^\top\mu_h^{t-1})(s, b) } }. 
\end{align*}
\vspace{-10pt}
\STATE Update Q-values for all $(s,a,b)\in\cS\times\cA\times\cB$:
\vspace{-5pt}
\begin{equation}\label{eq:up-Q}
\begin{split}
    \textstyle \up{Q}_h^{t}(s,a,b) \setto r_h(s,a,b) + \P_h\brac{\max_{\mu^\dagger\in\Delta_\cA} \<\mu^\dagger, \sum_{i=1}^t \alpha_t^i \up{Q}_{h+1}^i\nu_{h+1}^i \>} (s,a,b);\qquad\\
    \textstyle \low{Q}_h^{t}(s,a,b) \setto r_h(s,a,b) + \P_h\brac{\min_{\nu^\dagger\in\Delta_\cB} \<\nu^\dagger, \sum_{i=1}^t \alpha_t^i \paren{\low{Q}_{h+1}^i}^\top\mu_{h+1}^i \>} (s,a,b).
\end{split}
\end{equation}
\vspace{-5pt}
\ENDFOR
\ENDFOR
\STATE Output state-wise average policy for all $(h,s)$:
\vspace{-5pt}
\begin{equation*}
    \textstyle \widehat{\mu}_h^T(\cdot|s) \setto \sum_{t=1}^T\alpha_T^t\mu_{h}^{t}(\cdot|s),\quad \widehat{\nu}_h^T(\cdot|s) \setto \sum_{t=1}^T\alpha_T^t\nu_{h}^{t}(\cdot|s).
\end{equation*}
\end{algorithmic}
\end{algorithm}

Algorithm~\ref{algorithm:modified-oftrl-q} keeps track of a series of $\up Q_h^t, \low Q_h^t$'s that are upper-bounds and lower-bounds of $Q_h^\star$ respectively. The policy update is similar to the update as the OFTRL algorithm (Algorithm \ref{alg:OFTRL}), but here $\mu$ is performing OFTRL with respect to $\up Q_h^t$'s while $\nu$ with respect to $\low Q_h^t$'s. The value updates \eqref{eq:up-Q} are slightly different from the value update in our unified framework, however, we remark that it is still an incremental update because the terms inside the inner product $\sum_{i=1}^t \alpha_t^i \up{Q}_{h+1}^i\nu_{h+1}^i , \sum_{i=1}^t \alpha_t^i \paren{\low{Q}_{h+1}^i}^\top\mu_{h+1}^i$ are incremental updates, which leads to that fact that $\up Q_h^t, \low Q_h^t$'s are also updating incrementally.. 
Further, the algorithm can be performed in a decentralized manner, which is stated in Algorithm \ref{algorithm:modified-oftrl-v}. The convergence result is stated in Theorem \ref{theorem:modified-alg}.

\begin{algorithm}[h]
\caption{Modified OFTRL (Equivalent V-form)}
\label{algorithm:modified-oftrl-v}
\begin{algorithmic}[1]
\STATE \textbf{Initialize:} $\up{V}_h^1(s)\setto H-h+1, \low{V}_h^1(s)\setto 0$ for all $(h,s,a,b)$.
\FOR{$t=1,\dots,T$}
\FOR{$h=H,\dots,1$}
\STATE Update policies for all $s\in\cS$ by OFTRL:
\vspace{-5pt}
\begin{align*}
    & \textstyle \mu^{t}_h(a | s) \propto_a \exp\paren{ (\eta/w_t) \cdot \brac{ \sum_{i=1}^{t-1} w_i\up{L}_h^i(s, a) + w_{t-1}\up{L}_h^{t-1}(s, a) } }\\
    & \textstyle \nu^{t}_h(b | s) \propto_b \exp\paren{ -(\eta/w_t) \cdot \brac{ \sum_{i=1}^{t-1} w_i((Q_h^i)^\top\mu_h^i)(s, b) + w_{t-1}((Q_h^{t-1})^\top\mu_h^{t-1})(s, b) } }. 
\end{align*}
\vspace{-10pt}
\STATE Update losses for all $(s,a)\in\cS\times\cA$:
\vspace{-5pt}
\begin{align*}
    \textstyle \up{L}_h^t(s,a) \!\setto\! \< r_h(s,a,\cdot) + \brac{\P_h\up{V}_{h+1}^t}(s,a,\cdot), \nu_h^t(\cdot|s) \>,~~ \textstyle \low{L}_h^t(s,a) \!\setto\! \< \brac{r_h(s,a,\cdot) + \brac{\P_h\low{V}_{h+1}^t}(s,a,\cdot)}^{\!^\top}\!\!, \mu_h^t(\cdot|s) \>.
\end{align*}
\vspace{-10pt}
\STATE Update V-value for all $s\in\cS$:
\vspace{-5pt}
\begin{align}
    \textstyle \up{V}_h^t(s) \setto \max_{\mu^\dagger\in\Delta_\cA} \< \mu^\dagger, \sum_{i=1}^t \alpha_t^i \up{L}_h^i(s,\cdot) \>,\quad \low{V}_h^t(s) \setto \min_{\nu^\dagger\in\Delta_\cB} \< \nu^\dagger, \sum_{i=1}^t \alpha_t^i \low{L}_h^i(s,\cdot) \>.
\end{align}
\vspace{-5pt}
\ENDFOR
\ENDFOR
\STATE Output state-wise average policy for all $(h,s)$:
\vspace{-5pt}
\begin{equation*}
    \textstyle 
    \widehat{\mu}_h^T(\cdot|s) \setto \sum_{t=1}^T\alpha_T^t\mu_{h}^{t}(\cdot|s), \quad \widehat{\nu}_h^T(\cdot|s) \setto \sum_{t=1}^T\alpha_T^t\nu_{h}^{t}(\cdot|s).
\end{equation*}
\end{algorithmic}
\end{algorithm}

\begin{theorem}[Convergence rate of modified OFTRL]
\label{theorem:modified-alg}
Algorithm \ref{algorithm:modified-oftrl-q} with $\eta = \frac{1}{16H}$ guarantees that
\begin{equation*}
    \negap(\hat{\mu}^T,\hat{\nu}^T)\le C\brac{\frac{H^4 \log (A\vee B) \paren{\log T}^2}{T}},
\end{equation*}
where $C$ is some absolute constant.
\end{theorem}

\subsection{Proof of Theorem \ref{theorem:modified-alg}}
In this section, we consider the following definitions of regret, which is slightly different from the definition in \eqref{eq:def-regret}:
\begin{talign*}
    & \reg_{h,\mu}^t(s) \defeq \max_{\mu^\dagger\in\Delta_\cA} \sum_{i=1}^t \alpha_t^i \<\mu^\dagger - \mu_h^i(\cdot|s), \brac{\up Q_h^i\nu_h^i}(s, \cdot)\>, \\
    & \reg_{h,\nu}^t(s) \defeq \max_{\nu^\dagger\in\Delta_\cB} \sum_{i=1}^t \alpha_t^i \< \nu_h^i(\cdot|s) - \nu^\dagger, \brac{(\low Q_h^i)^\top\mu_h^i}(s, \cdot) \>,  \\
    & \reg_{h,\mu+\nu}^t \defeq \max_{s\in\cS} \reg_{h,\mu}^t(s) + \reg_{h,\nu}^t(s).
\end{talign*}

We first prove that $\low Q_h^t$ and $\up Q_h^t$ upper and lower bounds $Q^\star_h$ respectively.
\begin{lemma}\label{lemma:up-low-sandwich}
\begin{equation*}
    \low{Q}_h^t(s,a,b)\le Q^\star_h(s,a,b) \le \up Q_h^t(s,a,b).
\end{equation*}
\begin{proof}
We prove by induction on $(h,t)$. Given the initialization, for $t=0$ the condition holds. Since $\up Q_{H+1}^t, \low Q_{H+1}^t = 0$, we have that for $h=H+1$ the condition holds. Assume that the condition hold for $(i, h+1), i\le t$, then
\begin{align*}
    \up{Q}_h^{t}(s,a,b) &= r_h(s,a,b) + \P_h\brac{\max_{\mu^\dagger\in\Delta_\cA} \<\mu^\dagger, \sum_{i=1}^t \alpha_t^i \up{Q}_{h+1}^i\nu_{h+1}^i \>} (s,a,b)\\
    &\ge r_h(s,a,b) + \P_h\brac{\max_{\mu^\dagger\in\Delta_\cA} \<\mu^\dagger, {Q}_{h+1}^\star\paren{\sum_{i=1}^t \alpha_t^i \nu_{h+1}^i} \>} (s,a,b)\\
    &\ge r_h(s,a,b) + \P_h\brac{\max_{\mu^\dagger\in\Delta_\cA}\min_{\nu^\dagger\in\Delta_\cB} \<\mu^\dagger, {Q}_{h+1}^\star\nu^\dagger\>} (s,a,b)\\
    &= Q_h^\star(s,a,b).
\end{align*}
Using similar strategy, we can also show that $\low Q_h^t(s,a,b)\le Q_h^\star(s,a,b)$, which implies that the condition hold for $(t,h)$, and thus finishes the proof by induction.
\end{proof}
\end{lemma}



Throughout the rest of this section, we define the following shorthand for the gap between $\up Q_h^t, \low Q_h^t$ defined in \eqref{eq:up-Q}:
\begin{equation*}
    \delta_h^t := \|\up Q_h^t - \low Q_h^t\|_\infty = \max_{s,a,b}\brac{\up Q_h^t(s,a,b) - \low Q_h^t(s,a,b)},
\end{equation*}
\begin{lemma}[Recursion of $\delta_h^t$]
Algorithm \ref{algorithm:modified-oftrl-q} guarantees that for all $(t,h)\in[T]\times[H]$,
\begin{align*}
    \delta_h^t \le \sum_{i=1}^t \alpha_t^i \delta_{h+1}^i + \reg_{h+1,\mu+\nu}^t.
\end{align*}
Further, suppose that $\reg_{h,\mu+\nu}^t\le \barreg_{h,\mu+\nu}^t$ for all $(h,t)\in[H]\times[T]$, where $\barreg_{h,\mu+\nu}^t$ is non-increasing in $t$: $\barreg_{h,\mu+\nu}^t\ge \barreg_{h,\mu+\nu}^{t+1}$ for all $t\ge 1$. Then we have
\begin{align*}
    \delta_h^t \le 2H\cdot\frac{1}{t}\sum_{i=1}^t\max_{h'}\barreg_{h',\mu+\nu}^i.
\end{align*}
\begin{proof}The proof structure resembles Lemma \ref{lemma:recursion-of-value-estimation}. From the definition of $\up Q_h^t, \low Q_h^t$, we have that
\begin{align*}
    &\quad \up Q_h^t(s,a,b) - \low Q_h^t(s,a,b)\\
    &\le \P_h\max_{\mu^\dagger\in\Delta_\cA,\nu^\dagger\in\Delta_\cB}\<\mu^\dagger, \sum_{i=1}^t\alpha_t^i \up Q_{h+1}^i\nu_{h+1}^i\> - \<\nu^\dagger, \sum_{i=1}^t\alpha_t^i(\low Q_{h+1}^i)^\top\mu_{h+1}^i\>\\
    &= \P_h\left[\max_{\mu^\dagger\in\Delta_\cA}\<\mu^\dagger, \sum_{i=1}^t\alpha_t^i \up Q_{h+1}^i\nu_{h+1}^i\> - \sum_{i=1}^t\alpha_t^i (\mu_{h+1}^i)^\top\up Q_{h+1}^i \nu_{h+1}^i\right.\\
    &+ \max_{\nu^\dagger\in\Delta_\cB} \sum_{i=1}^t\alpha_t^i (\mu_{h+1}^i)^\top\up Q_{h+1}^i \nu_{h+1}^i - \<\nu^\dagger, \sum_{i=1}^t\alpha_t^i(\low Q_{h+1}^i)^\top\mu_{h+1}^i\>\\
    &+ \left.\sum_{i=1}^t\alpha_t^i (\mu_{h+1}^i)^\top\up Q_{h+1}^i \nu_{h+1}^i - \sum_{i=1}^t\alpha_t^i (\mu_{h+1}^i)^\top\up Q_{h+1}^i \nu_{h+1}^i\right] \\
    &\le \reg_{h+1,\mu+\nu}^t + \sum_{i=1}^t\alpha_t^i \|\up Q_{h+1}^i-\low Q_{h+1}^i\|_\infty= \sum_{i=1}^t\alpha_t^i\delta_{h+1}^i + \reg_{h+1,\mu+\nu}^t.
\end{align*}
Then using the same argument as Lemma \ref{lemma:recursion-of-value-estimation}, we can consider an auxiliary sequence
\begin{equation}
    \begin{cases}
    \Delta_h^t =  \sum_{i=1}^{t} \alpha_{t}^i \Delta_{h+1}^{i} + \barreg_{h+1,\mu+\nu}^t,\\
    \Delta_{H+1}^t= 0,~ \text{ for all } t.
    \end{cases}
\end{equation}
 Observe that $\{\Delta_h^t \}_{h,t}$ satisfies the following properties 
\begin{equation}
    \begin{cases}
    \Delta_h^t \ge \delta_h^t\qquad &\text{ (by definition)},\\
    \Delta_h^t \le \Delta_h^{t-1} \quad &\text{ (by Lemma \ref{lem:delta.monotonic})}.
    \end{cases}
\end{equation}
Therefore, to control $\delta_h^t$, it suffices to bound $\Delta_h^t\le \frac{1}{t}\sum_{i=1}^t \Delta_h^i$, which follows from the standard argument in \cite{jin2018q}:
\begin{align*}
     \frac{1}{t}\sum_{i=1}^t \Delta_h^i 
     &=  \frac{1}{t}\sum_{i=1}^t  \sum_{j=1}^{i} \alpha_{i}^j \Delta_{h+1}^{j} + \frac{1}{t}\sum_{i=1}^t \barreg_{h+1,\mu+\nu}^i\\
     & \le\frac{1}{t}\sum_{j=1}^{t}\paren{\sum_{i=j}^t\alpha_{i}^j} \Delta_{h+1}^{j} + \frac{1}{t}\sum_{i=1}^t \barreg_{h+1,\mu+\nu}^i\\
     &\le \paren{1+\frac{1}{H}} \cdot \frac{1}{t}\sum_{i=1}^{t}\Delta_{h+1}^{i} + \frac{1}{t}\sum_{i=1}^t \barreg_{h+1,\mu+\nu}^i\\
     &\le \paren{1+\frac{1}{H}}^2 \cdot \frac{1}{t}\sum_{i=1}^{t}\Delta_{h+2}^{i} + \paren{1+\frac{1}{H}}\cdot \frac{1}{t}\sum_{i=1}^t \barreg_{h+2,\mu+\nu}^i +
     \frac{1}{t}\sum_{i=1}^t \barreg_{h+1,\mu+\nu}^i\\
     &\le \cdots\\
     &\le \paren{\sum_{h'=h}^{H}\paren{1+\frac{1}{H}}^{h'-h}}\cdot\frac{1}{t}\sum_{i=1}^t \max_{1\le h'\le H}\barreg_{h',\mu+\nu}^i\\
     &\le (e-1)H\cdot \frac{1}{t}\sum_{i=1}^t \max_{1\le h'\le H}\barreg_{h',\mu+\nu}^i \le 
     2H\cdot \frac{1}{t}\sum_{i=1}^t \max_{1\le h'\le H}\barreg_{h',\mu+\nu}^i.
\end{align*}
which completes the proof.
\end{proof}
\end{lemma}

\begin{lemma}[Bound the $\negap$ by $\reg_{h,\mu+\nu}$]\label{lemma:bound-NE-gap-by-regret}
Suppose that the per-state regrets (summing over the two agents) can be upper-bounded as $\reg_{h,\mu+\nu}^t\le \barreg_{h,\mu+\nu}^t$ for all $(h,t)\in [H]\times[T]$ where $\barreg_{h,\mu+\nu}^t$ is non-increasing in $t$: $\barreg_{h,\mu+\nu}^t\ge\barreg_{h,\mu+\nu}^{t+1}$ for all $t\ge 1$. Then, the output policy $(\hat{\mu}^T,\hat{\nu}^T)$ of Algorithm \ref{algorithm:modified-oftrl-q} satisfies
\begin{equation*}
    \negap(\hat{\mu}^T,\hat{\nu}^T)\le 2H\max_h\barreg_{h,\mu+\nu}^T + 24H^2\log T\cdot \frac{1}{T}\sum_{t=1}^T\max_h\barreg_{h,\mu+\nu}^t
\end{equation*}
\begin{proof}
From Lemma \ref{lem.duality} we have that
\begin{align*}
    & \quad \negap(\hat{\mu}^T,\hat{\nu}^T)=\paren{V_1^{\dagger,\hat{\nu}^T}(s_1) - V_1^\star(s_1)} + \paren{V_1^\star(s_1)- V_1^{\hat{\mu}^T,\dagger}(s_1)}\\
    &\le 2\sum_{h=1}^H \max_s\max_{\mu^\dagger\in\Delta_\cA,\nu^\dagger\in\Delta_\cB}\brac{(\mu^\dagger)^\top Q_h^\star \hat\nu_h^T \!-\!\paren{\hat\mu_h^T}^\top Q_h^\star \nu^\dagger }(s)\\
    &= 2\sum_{h=1}^H \max_s\max_{\mu^\dagger\in\Delta_\cA,\nu^\dagger\in\Delta_\cB}\sum_{t=1}^T\alpha_T^t\brac{(\mu^\dagger)^\top Q_h^\star \nu_h^t \!-\!\paren{\mu_h^t}^\top Q_h^\star \nu^\dagger }(s)\\
    &\le 2\sum_{h=1}^H \max_s\max_{\mu^\dagger\in\Delta_\cA,\nu^\dagger\in\Delta_\cB}\sum_{t=1}^T\alpha_T^t\brac{(\mu^\dagger)^\top \up Q_h^t \nu_h^t \!-\!\paren{\mu_h^t}^\top \low Q_h^t \nu^\dagger}(s)\\
    &\le 2\sum_{h=1}^H \max_s\left(\max_{\mu^\dagger\in\Delta_\cA}\sum_{t=1}^T\alpha_T^t\brac{(\mu^\dagger)^\top \up Q_h^t \nu_h^t \!-\!\paren{\mu_h^t}^\top \up Q_h^t \nu_h^t}(s)\right.\\
    &\qquad \qquad+\left.\max_{\nu^\dagger\in\Delta_\cB}\sum_{t=1}^T\alpha_T^t\brac{(\mu_h^t)^\top \low Q_h^t \nu_h^t \!-\!\paren{\mu_h^t}^\top \low Q_h^t \nu^\dagger}(s)\right)\\
    &+2\sum_{h=1}^H \max_s\max_{\mu^\dagger\in\Delta_\cA,\nu^\dagger\in\Delta_\cB}\sum_{t=1}^T\alpha_T^t\brac{(\mu_h^t)^\top \up Q_h^t \nu_h^t \!-\!\paren{\mu_h^t}^\top \low Q_h^t \nu_h^t}(s)\\
    &\le 2\sum_{h=1}^H\barreg_{h,\mu+\nu}^T + 2\sum_{h=1}^H\sum_{t=1}^T \alpha_T^t\delta_h^t\\
    &\le 2H\max_h\barreg_{h,\mu+\nu}^T + 4H^2\sum_{t=1}^T \alpha_T^t \frac{1}{t}\sum_{i=1}^t \max_h \reg_{h,\mu+\nu}^i \quad \textup{(Lemma \ref{lemma:up-low-sandwich})}\\
    &\le 2H\max_h\barreg_{h,\mu+\nu}^T + 4H^2\paren{\sum_{t=1}^T\frac{1}{t}\alpha_T^t}\paren{\sum_{i=1}^T\max_h\reg_{h,\mu+\nu}^i}\\
    &\le 2H\max_h\barreg_{h,\mu+\nu}^T + 24H^2\log T\cdot \frac{1}{T}\sum_{t=1}^T\max_h\reg_{h,\mu+\nu}^t~~ \textup{(Lemma \ref{lemma:alpha-convolution})},
\end{align*}
\end{proof}
\end{lemma}

\begin{lemma}[Bound $\reg_{h,\mu+\nu}^t$]\label{lemma:bound-regret-sum}
Running Algorithm \ref{algorithm:modified-oftrl-q} with $\eta = \frac{1}{16H}$ can guarantee that 
\begin{equation*}
    \reg_{h,\mu+\nu}^T(s) \le \frac{36H^2\log(A\vee B)}{T}
\end{equation*}
\begin{proof}
From Lemma \ref{lemma:OFTRL-auxillary}, substituting $g_t = w_t\up Q_h^t\nu_h^t(s), M_t = w_t\up Q_h^{t-1}\nu_h^{t-1}(s), \eta_t = \frac{\eta}{w_t}$, we can get that
\begin{align*}
    &\sum_{t=1}^T w_t\brac{\<\mu^\dagger, \up Q_h^t\nu_h^t\> - \<\mu_h^t, \up Q_h^t\nu_h^t\>}\le \frac{w_T\log A }{\eta} + \eta\sum_{t=1}^Tw_t\|\up Q_h^t\nu_h^t(s)\!-\!\up Q_h^{t-1}\nu_h^{t-1}(s)\|_\infty^2 - \sum_{t=2}^T \frac{w_t}{8\eta}\|\mu_h^t(\cdot|s) \!-\! \mu_h^{t-1}(\cdot|s)\|_1^2\\
    &\Longrightarrow ~~\reg_{h,\mu}^T \le \alpha_T^1\sum_{t=1}^T w_{t-1}\<\mu^\dagger, \up Q_h^t\nu_h^t\> - \<\mu_h^t, \up Q_h^t\nu_h^t\>\\
    &\qquad\qquad\quad~ \le \frac{\alpha_T\log A }{\eta} + \eta\sum_{t=1}^T\alpha_T^t\|\up Q_h^t\nu_h^t(s)-\up Q_h^{t-1}\nu_h^{t-1}(s)\|_\infty^2 - \sum_{t=2}^T \frac{\alpha_T^{t-1}}{8\eta}\|\mu_h^t(\cdot|s) - \mu_h^{t-1}(\cdot|s)\|_1^2.
\end{align*}
Further we have that
\begin{align*}
    \|\up Q_h^t\nu_h^t(s)-\up Q_h^{t-1}\nu_h^{t-1}(s)\|_\infty^2 \le 2\|\up Q_h^t-\up Q_h^{t-1}\|_\infty^2 + 2\|\nu_h^t(s)-\nu_h^{t-1}(s)\|_1^2.
\end{align*}
From the definition of $\up Q_h^t$ we have that
\begin{align*}
    \|\up Q_h^t - \up Q_h^{t-1}\| &\le \left\|\sum_{i=1}^t \alpha_t^i \up Q_{h+1}^i \nu_{h+1}^i - \sum_{i=1}^{t-1}\alpha_{t-1}^i \up Q_{h+1}^i \nu_{h+1}^i\right\|_\infty\\
    &= \left\|\alpha_t \up Q_{h+1}^t\nu_{h+1}^t + (1-\alpha_t)\sum_{i=1}^{t-1} \alpha_{t-1}^i \up Q_{h+1}^i \nu_{h+1}^i - \sum_{i=1}^{t-1}\alpha_{t-1}^i \up Q_{h+1}^i \nu_{h+1}^i\right\|_\infty\\
    & = \left\|\alpha_t \up Q_{h+1}^t\nu_{h+1}^t  -\alpha_t\sum_{i=1}^{t-1} \alpha_{t-1}^i \up Q_{h+1}^i \nu_{h+1}^i\right\|_\infty\le \alpha_t H.
\end{align*}
Substitute this inequality to the regret bound we have
\begin{align*}
    \reg_{h,\mu}^T(s) &\le \frac{\alpha_T\log A }{\eta} + 2\eta\sum_{t=1}^T\alpha_T^t \alpha_t^2 H^2 + 2\eta\sum_{t=1}^T\alpha_T^t\|\nu_h^t(s)-\nu_h^{t-1}(s)\|_1^2  - \sum_{t=2}^T \frac{\alpha_T^{t-1}}{8\eta}\|\mu_h^t(\cdot|s) - \mu_h^{t-1}(\cdot|s)\|_1^2\\
    &\le \frac{\alpha_T\log A }{\eta} + \frac{8\eta H^3}{T} + 2\eta\sum_{t=1}^T\alpha_T^t\|\nu_h^t(s)-\nu_h^{t-1}(s)\|_1^2 + - \sum_{t=2}^T \frac{\alpha_T^{t-1}}{8\eta}\|\mu_h^t(\cdot|s) - \mu_h^{t-1}(\cdot|s)\|_1^2. ~~\textup{(Lemma \ref{lemma:alpha-convolution})}
\end{align*}
Similar bound holds for $\reg_{h,\nu}^T$:
\begin{equation*}
    \reg_{h,\nu}^T(s)\le \frac{\alpha_T\log B }{\eta} + \frac{8\eta H^3}{T} + 2\eta\sum_{t=1}^T\alpha_T^t\|\mu_h^t(s)-\mu_h^{t-1}(s)\|_1^2 + - \sum_{t=2}^T \frac{\alpha_T^{t-1}}{8\eta}\|\nu_h^t(\cdot|s) - \nu_h^{t-1}(\cdot|s)\|_1^2.
\end{equation*}
Summing $\reg_{h,\mu}^T(s),\reg_{h,\nu}^T(s)$ together we get
\begin{align*}
    \reg_{h,\mu+\nu}^T(s) \le \frac{2\alpha_T\log(A\vee B)}{\eta} + \frac{16\eta H^3}{T} + 16\eta \alpha_T^1 +\sum_{t=2}^T \paren{2\eta\alpha_T^t - \frac{\alpha_T^{t-1}}{8\eta}} \paren{\|\mu_h^t(\cdot|s) - \mu_h^{t-1}(\cdot|s)\|_1^2 + \nu_h^t(\cdot|s) - \nu_h^{t-1}(\cdot|s)\|_1^2}.
\end{align*}
Since $\frac{\alpha_T^{t-1}}{\alpha_T^t} \ge \frac{1}{H}$ for $t\ge 2$, by setting $\eta = \frac{1}{16 H}$ we can guarantee that $2\eta\alpha_T^t - \frac{\alpha_T^{t-1}}{8\eta} \le 0$, thus
\begin{align*}
    \reg_{h,\mu+\nu}^T(s) \le \frac{2\alpha_T\log(A\vee B)}{\eta} + \frac{16\eta H^3}{T} + 16\eta \alpha_T^1\le \frac{32H^2 \log(A\vee B)}{T} + \frac{H^2}{T} + \frac{1}{T}\le \frac{36H^2\log(A\vee B)}{T}
\end{align*}
\end{proof}
\end{lemma}

Given Lemma \ref{lemma:bound-NE-gap-by-regret} and \ref{lemma:bound-regret-sum}, we are now ready to prove Theorem \ref{theorem:modified-alg}.
\begin{proof}[Proof Theorem \ref{theorem:modified-alg}]
From Lemma \ref{lemma:bound-NE-gap-by-regret} and \ref{lemma:bound-regret-sum} we have that:
\begin{align*}
        \negap(\hat{\mu}^T,\hat{\nu}^T)&\le 2H\max_h\barreg_{h,\mu+\nu}^T + 24H^2\log T\cdot \frac{1}{T}\sum_{t=1}^T\max_h\barreg_{h,\mu+\nu}^t\\
        &\le 2H\frac{36H^2\log(A\vee B)}{T} + 24H^2\log T\cdot \frac{1}{T}\sum_{t=1}^T\frac{36H^2\log(A\vee B)}{t}\\
        &\le \frac{936H^4 \log(A\vee B) \paren{\log T+1}^2}{T},
\end{align*}
which completes the proof.
\end{proof}

\section{Optimistic policy optimization for general-sum Markov Games}
\label{appendix:general-sum}

\subsection{Preliminaries}
\label{appendix:general-sum-prelim}

Here we formally present the preliminaries for multi-player general-sum Markov games, parallel to the zero-sum setting considered in Section~\ref{section:prelim}.

\paragraph{Multi-player general-sum Markov games} We consider tabular episodic (finite-horizon) $m$-player general-sum Markov games (MGs), which can be denoted as $\cM(H,\cS, \set{\cA_i}_{i=1}^m,\P, \set{r_i}_{i=1}^m)$, where $H$ is the horizon length; $\cS$ is the state space with $|\cS| = S$; $\cA_i$ is the action space of the $i$-th player, with $|\cA_i| = A_i$. We use $\a=(a_1,\dots,a_m)\in\prod_{i\in[m]}\cA_i\eqdef \cA$ to denote a joint action taken by all players; $\P = \{\P_h\}_{h=1}^H$ is the transition probabilities, where each $\P_h(s'|s, \a)$ gives the probability of transition to state $s'$ from state-action $(s,\a)$; $r_i = \{r_{i,h}\}_{h=1}^H$ are the reward functions, where each $r_{i,h}(s,\a)$ is the deterministic reward function of the $i$-th player at time step $h$ and state-action $(s, \a)$.
In each episode, the MG starts with a deterministic initial state $s_1$. Then at each time step $1\le h\le H$, all players observes the state $s_h$, each player takes an action $a_{i,h}\in \cA_i$. Then, each player receive their rewards $r_{i,h}(s_h, \a_h)$, and the game transitions to the next state $s_{t+1}\sim \P_h(\cdot|s_h,\a_h)$. 



\paragraph{Policies \& value functions}
A (Markov) policy $\pi_i$ of the $i$-th player is a collection of policies $\pi_i = \{\pi_{i,h}: \cS \to \Delta_{\cA_i}\}_{h=1}^H$, where each $\pi_{i,h}(\cdot|s_h)\in\Delta_{\cA_i}$ specifies the probability of taking action $a_{i,h}$ at $(h, s_h)$. We use $\pi=\set{\pi_i}_{i\in[m]}$ to denote a product policy of all players. For any joint policy $\pi$ (not necessarily a product policy), we use $V_{i,h}^{\pi}: \cS\to\R$ and $Q_{i,h}^{\pi}:\cS\times\cA\to \R$ to denote the ($i$-th player's) value function and Q-function at time step $h$, respectively, i.e.
\begin{talign}
    V_{i,h}^{\pi}(s)&:= \E_{\pi}\brac{\sum_{h=h'}^H r_{i,h'}(s_{h'}, \a_{h'})~|~s_h = s}, \\
    Q_{i,h}^{\pi}(s,\a)&:= \E_{\pi}\brac{\sum_{h=h'}^H r_{i,h'}(s_{h'}, \a_{h'})~|~s_h = s, \a_h = \a}.
\end{talign}
For notational simplicity, we use the following abbreviation: $[\P_h V](s,\a):= \E_{s'\sim\P_h(\cdot|s,\a)}V(s')$ for any value function $V$. By definition of the value functions and Q-functions, we have the following Bellman equations for all Markov product policy $\pi$ and all $(i,h,s,\a)$:
\begin{align*}
    Q_{i,h}^{\pi}(s,\a)&= \paren{r_{i,h} + \P_hV_{i,h+1}^{\pi}} (s,\a), \\
    V_{i,h}^{\pi}(s,\a)&=\E_{\a\sim \pi_h(\cdot|s)}\brac{Q_{i,h}^{\pi}(s,\a)} = \<Q_{i,h}^{\pi}(s,\cdot), \pi_h(\cdot|s)\>.
\end{align*}
The goal for the $i$-th player is to maximize their own value function.


\paragraph{Correlated policy \& best response}
A (general) correlated policy $\pi$ is any policy for which players may take actions in a history-dependent and correlated fashion. More precisely, a correlated policy $\pi$ is a mapping $\set{\pi_{h}:\Omega\times (\cS\times\cA)^{h-1}\times \cS\to \Delta_{\cA}}$, and executes as follows. At the beginning of an episode, a random seed $w\in \Omega$ is sampled from some distribution (also denoted as $\Omega$ with slight abuse of notation). Then, at each step $h$ and state $s_h$, suppose the history so far is $(s_1,\a_1,\dots,s_{h-1},\a_{h-1})$. Then, $\pi$ samples a joint action $\a_h\sim \pi_h(\cdot|\omega,(s_1,\a_1,\dots,s_{h-1},\a_{h-1}); s_h)$. This formulation allows each $\pi_h(\cdot|\omega, \cdot, \cdot)$ to be a Markov product policy for any \emph{fixed} $\omega$ while still making $\pi$ to be a correlated policy, due to the correlation introduced by $\omega$.

For any correlated policy $\pi$, let $\pi_{-i}$ denote the (marginal) policy of all but the $i$-th player. Then, the ($i$-th) player's best-response value function is
\begin{align*}
    V_{i,1}^{\dagger, \pi_{-i}}(s_1) \defeq \max_{\pi_i^\dagger} V_{i,1}^{\pi_i^\dagger\times \pi_{-i}}(s_1),
\end{align*}
where the max is over all (potentially history-dependent) policy $\pi_i^\dagger$ for the $i$-th player. 

\paragraph{Coarse Correlated Equilibrium (CCE)}
For general-sum MGs, we consider learning an approximate Coarse Correlated Equilibrium~\citep{liu2021sharp,song2021can} defined as follows.
\begin{definition}[$\eps$-approximate Coarse Correlated Equilibrium]
\label{definition:cce}
For any $\eps\ge 0$, a correlated policy $\pi$ is an $\epsilon$-approximate Coarse Correlated Equilibrium ($\eps$-CCE) if 
\begin{align*}
    \ccegap(\pi) \defeq \max_{i\in[m]} V_{i,1}^{\dagger, \pi_{-i}}(s_1) - V_{i,1}^\pi(s_1) \le \eps.
\end{align*}
\end{definition}


\paragraph{Additional notation}
For any Q function $Q_{i,h}(s, \cdot):\cS\times(\prod_{i=1}^m \cA_i)\to\R$ and joint policy $\pi_h(\cdot|s)$, we use $[Q_{i,h}\pi_h](s)\defeq \< Q_{i,h}(s, \cdot), \pi_h(\cdot|s)\>$ for shorthand. Similarly, for any joint policy $\pi_{-i,h}(\cdot|s)$ over all but the $i$-th player, $[Q_{i,h}\pi_{-i,h}](s, a_i)\defeq \< Q_{i,h}(s, a_i, \cdot), \pi_{-i,h}(\cdot|s)\>$.

\subsection{Algorithm and formal statement of result}
\label{appendix:general-sum-formal}

\begin{algorithm}[t]
\caption{OFTRL for multi-player general-sum Markov games}
\label{algorithm:oftrl-general-sum-mg}
\begin{algorithmic}[1]
\STATE \textbf{Initialize:} $Q_h^0(s,\a)\setto H-h+1$ for all $(h,s,a,b)$.
\FOR{$t=1,\dots,T$}
\FOR{$h=H,\dots,1$}
\STATE Update policies for all $s\in\cS$ and $i\in[m]$ by OFTRL
\begin{align}
\label{equation:oftrl-pii}
    & \textstyle \pi_{i,h}^t(a_i | s) \propto_{a_i} \exp\paren{ (\eta/w_t) \cdot \brac{ \sum_{j=1}^{t-1} w_j(Q_{i,h}^j\pi_{-i,h}^j)(s, a_i) + w_t(Q_{i,h}^{t-1}\pi_{-i,h}^{t-1})(s, a_i) } }. 
\end{align}
\STATE Update Q-value for all $(i,s,\a)\in[m]\times \cS\times\cA$:
\begin{align}
\label{eq:Q-update-general-sum}
    Q_{i,h}^{t}(s,\a) \setto (1-\alpha_t) Q_{i,h}^{t-1}(s,\a) + \alpha_t \paren{r_h + \P_h[ Q_{i,h+1}^t\pi_{h+1}^t ]} (s,\a).
\end{align}
\ENDFOR
\ENDFOR
\STATE Output policy $\hat{\pi}^T=\hat{\pi}_1^T$, where $\hat{\pi}_1^T$ is defined in Algorithm~\ref{algorithm:certified-policy}.
\end{algorithmic}
\end{algorithm}

\begin{algorithm}[h]
\caption{Policy $\hat{\pi}_h^t$}
\label{algorithm:certified-policy}
\begin{algorithmic}[1]
\REQUIRE Product policies $\pi_{h'}^{t'}(\cdot|s')=\prod_{i=1}^m \pi_{i,h'}^{t'}(\cdot|s')$ for all $(h',t',s')\in[H]\times[T]\in\cS$.
\STATE Sample $j\in[t]$ with probability $\P(j=i)=\alpha_t^i$.
\STATE Play policy $\pi_h^j$ at the $h$-th step of the game.
\STATE Play policy $\hat{\pi}_{h+1}^j$ for step $h+1$ onward.
\end{algorithmic}
\end{algorithm}




\begin{theorem}[Formal version of Theorem~\ref{theorem:oftrl-general-sum}]
\label{theorem:oftrl-general-sum-formal}
Suppose Algorithm~\ref{algorithm:oftrl-general-sum-mg} is run for $T$ rounds. Then the per-state regret can be bounded as follows for some absolute constant $C>0$:
\begin{align*}
    \reg_h^t \le \barreg_h^t \defeq C \brac{ \frac{H\log\Amax}{\eta t} + \frac{\eta H^3}{t} + (m-1)^2\eta^3H^4 }~~~\textrm{for all}~(h,t)\in[H]\times[T].
\end{align*}
Further, choosing $\eta=(\log\Amax \log T/(H^3T))^{1/4} (m-1)^{-1/2}$, the output policy $\hat{\pi}^T$ achieves
\begin{align*}
    \ccegap(\hat{\pi}^T) \le & \cO\Big( H^{11/4}(\log\Amax \log T)^{3/4} \sqrt{m-1} \cdot T^{-3/4} \\
    & \qquad + H^{13/4}(\log\Amax)^{1/4}(\log T)^{5/4}(m-1)^{-1/2}\cdot  T^{-5/4}  \Big).
\end{align*}
\end{theorem}

\paragraph{Proof overview and remarks} 
The proof of Theorem~\ref{theorem:oftrl-general-sum-formal} also follows by relating the performance of the output policy by per-state regrets via performance difference (Lemma~\ref{lem.regret-decompose-general-sum}, similar as Theorem~\ref{theorem:master}), and bounding per-state regrets as $\reg_h^t\le \barreg_h^t\defeq \tO(1/(\eta t) + \eta^3(m-1)^2)$ which gives the theorem. The latter builds upon the fast convergence analysis of OFTRL in multi-player normal-form games~\citep{syrgkanis2015fast} as well as additional handling of the changing game rewards, similar as in Theorem~\ref{theorem:oftrl-main}. Note that the $\tO(T^{-3/4})$ rate here is worse than $\tO(T^{-5/6})$ for the zero-sum setting in Theorem~\ref{theorem:oftrl-main}. This happens as the fine-grained analysis of OFTRL~\citep{chen2020hedging} used there relies critically on the game having two players (for translating between the iterate stabilities and loss stabilities between each other), and becomes infeasible when there are more than 2 players. 

We first present some lemmas in Section~\ref{appendix:general-sum-lemmas}. The proof of Theorem~\ref{theorem:oftrl-general-sum-formal} is then provided in Section~\ref{appendix:proof-oftrl-general-sum-formal}.

\subsection{Useful lemmas}
\label{appendix:general-sum-lemmas}

We additionally define the V-values maintained by Algorithm~\ref{algorithm:oftrl-general-sum-mg} as
\begin{align}
    V_{i,h}^t(s) \defeq \sum_{j=1}^t \alpha_t^j \brac{Q_{i,h}^j \pi_h^j}(s)
\end{align}
for all $(i,h,t,s)\in[m]\times[H]\times[T]\times\cS$, where $Q_{i,h}^t$ and $\pi_h^t$ are the Q-functions and joint policies maintained within Algorithm~\ref{algorithm:oftrl-general-sum-mg}. Note that by~\eqref{eq:Q-update-general-sum}, we immediately have
\begin{align}
    \label{equation:q-v-general-sum}
    \begin{aligned}
    & \quad Q_{i,h}^t(s,\a) = \sum_{j=1}^t \alpha_t^j \brac{ r_h + \P_h[Q_{i,h+1}^j\pi_{h+1}^j]}(s, \a) \\
    & = \paren{r_h + \P_h\brac{ \sum_{j=1}^t \alpha_t^j Q_{i,h+1}^j\pi_{h+1}^j } }(s, \a) = \paren{r_h + \P_h V_{i,h+1}^t}(s, \a).
\end{aligned}
\end{align}

We also define the value functions of $\hat{\pi}_h^t$ and of its best response for any $(i,h,t,s)$ as (see e.g.~\citep[Definition C.4 \& Eq.(8)]{song2021can}):
\begin{align*}
    & V_{i,h}^{\hat\pi_{h}^t}(s) \defeq \E_{\hat{\pi}_h^t}\brac{ \sum_{h'=h}^H r_{i,h'} | s_h = s }, \\
    & V_{i,h}^{\dagger, \hat\pi_{-i,h}^t}(s) \defeq \max_{\pi_{i,h:H}} \E_{\pi_{i,h:H}\times \hat{\pi}_{-i,h}^t} \brac{ \sum_{h'=h}^H r_{i,h'} | s_h = s }.
\end{align*}

\begin{lemma}[Equivalence of value functions]
\label{lem.certified-value}
For Algorithm~\ref{algorithm:oftrl-general-sum-mg}, we have for all $i\in[m]$ and all $(h,s,t)\in[H+1]\times\cS\times[T]$ that
\begin{align*}
    V_{i,h}^t(s) = V_{i,h}^{\hat\pi_h^t}(s).
\end{align*}
\end{lemma}
\begin{proof}
We prove this by backward induction over $h\in[H+1]$. The claim trivially holds for $h=H+1$. Suppose the claim holds for steps $h+1$ onward and all $(s,t)\in\cS\times[T]$. For step $h$ and any fixed $(s,t)\in\cS\times[T]$, note that
\begin{align*}
    & \quad V_{i,h}^t(s) = \sum_{j=1}^t \alpha_t^j \brac{Q_{i,h}^j\pi_h^j}(s) \stackrel{(i)}{=} \sum_{j=1}^t \alpha_t^j \brac{(r_h + \P_h V_{i,h+1}^j)\pi_h^j}(s) \\
    & \stackrel{(ii)}{=} \sum_{j=1}^t \alpha_t^j \brac{(r_h + \P_h V_{i,h+1}^{\hat{\pi}_{h+1}^j})\pi_h^j}(s) \stackrel{(iii)}{=} V_{i,h}^{\hat{\pi}_h^t}(s).
\end{align*}
Above, (i) follows by~\eqref{equation:q-v-general-sum}; (ii) uses the inductive hypothesis; (iii) uses the definition of the output policy $\hat{\pi}_h^t$ (cf. Algorithm~\ref{algorithm:certified-policy}), which samples $j\in[t]$ with probability $\alpha_t^j$, plays $\pi_h^j(\cdot|s)$, and plays $\hat{\pi}_{h+1}^j$ for the rest of the game. This proves the case for step $h$ and thus the lemma.
\end{proof}


Define the weighted per-state regrets as
\begin{align}
    & \reg_{h,i}^t(s) \defeq \max_{\pi_i^\dagger\in\Delta_{\cA_i}} \sum_{j=1}^t \alpha_t^j \< Q_h^j(s, \cdot), (\pi_i^\dagger\times\pi_{-i,h}^j)(\cdot|s) - \pi_h^j(\cdot|s)\> , \label{equation:reghits} \\
    & \reg_h^t \defeq \max_{s\in\cS} \max_{i\in[m]} \reg_{h,i}^t(s). \label{equation:reght-general-sum}
\end{align}
The following lemma bounds the difference between the values of the certified policy $\pi_h^t$ (Algorithm~\ref{algorithm:certified-policy}) and its best-response.

\begin{lemma}[Recursion of best-response values]
\label{lem.regret-decompose-general-sum}
For the policy $\hat{\pi}_h^t$ defined in Algorithm~\ref{algorithm:certified-policy}, we have for all $(i, h,t)\in[m]\times[H]\times[T]$ that
$$
\max_{s\in\cS} \paren{V_{i,h}^{\dagger,\hat\pi_{-i,h}^t}(s) - V_{i,h}^{\hat\pi_{h}^t}(s) }
\le \reg_{h}^t + \sum_{j=1}^t \alpha_t^j \max_{s\in\cS} \paren{ V_{i,h+1}^{\dagger,\hat\pi_{-i,h+1}^{j}}(s) - V_{i,h+1}^{\hat\pi_{h+1}^{j}}(s) }.
$$
\end{lemma}
\begin{proof}
Fix $(i,h,t)\in[m]\times[H]\times[T]$. We have for any $s\in\cS$ that
\begin{align*}
    & \quad V_{i,h}^{\dagger,\hat\pi_{-i,h}^t}(s) - V_{i,h}^{\hat\pi_{h}^t}(s) \\
    & = \max_{\pi_i^\dagger\in\Delta_{\cA_i}} \<\pi_i^\dagger, \sum_{j=1}^t \alpha_t^j \brac{(r_h + \P_h V_{i,h+1}^{\dagger, \hat{\pi}_{-i,h+1}^j}) \pi_{-i,h}^j}(s, \cdot) \> - \sum_{j=1}^t \alpha_t^j \<\pi_{i,h}^j, \brac{(r_h + \P_h V_{i,h+1}^{\hat{\pi}_{h+1}^j}) \pi_{-i,h}^j}(s, \cdot) \> \\
    & \stackrel{(i)}{\le} \sum_{j=1}^t \alpha_t^j \max_{s'\in\cS} \paren{ V_{i,h+1}^{\dagger,\hat\pi_{-i,h+1}^{j}}(s') - V_{i,h+1}^{\hat\pi_{h+1}^{j}}(s') } + \max_{\pi_i^\dagger\in\Delta_{\cA_i}} \sum_{j=1}^t \alpha_t^j \<\pi_i^\dagger - \pi_{i,h}^j, \brac{(r_h + \P_h V_{i,h+1}^{\hat{\pi}_{h+1}^j}) \pi_{-i,h}^j}(s, \cdot)\> \\
    & \stackrel{(ii)}{=} \sum_{j=1}^t \alpha_t^j \max_{s'\in\cS} \paren{ V_{i,h+1}^{\dagger,\hat\pi_{-i,h+1}^{j}}(s') - V_{i,h+1}^{\hat\pi_{h+1}^{j}}(s') } + \underbrace{\max_{\pi_i^\dagger\in\Delta_{\cA_i}} \sum_{j=1}^t \alpha_t^j \<\pi_i^\dagger - \pi_{i,h}^j, \brac{(r_h + \P_h V_{i,h+1}^{j}) \pi_{-i,h}^j}(s, \cdot)\>}_{\reg_{i,h}^t(s)} \\
    & \le \sum_{j=1}^t \alpha_t^j \max_{s'\in\cS} \paren{ V_{i,h+1}^{\dagger,\hat\pi_{-i,h+1}^{j}}(s') - V_{i,h+1}^{\hat\pi_{h+1}^{j}}(s') } + \reg_h^t.
\end{align*}
Above, (i) follows by substituting $V_{i,h+1}^{\dagger, \hat{\pi}_{-i,h+1}^j}$ with $V_{i,h+1}^{\hat{\pi}_{h+1}^j}$ and paying the additive error; (ii) follows from Lemma~\ref{lem.certified-value}. This proves the desired result.
\end{proof}

\begin{lemma}[Guarantee of Algorithm~\ref{algorithm:oftrl-general-sum-mg} via per-state regrets]
\label{lemma:master-general-sum}
Suppose that the per-state regrets~\eqref{equation:reght-general-sum} can be upper-bounded as $\reg_h^t\le \barreg_h^t$ for all $(h,t)\in[H]\times[T]$, where $\barreg_h^t$ is non-increasing in $t$: $\barreg_h^t\ge \barreg_h^{t+1}$ for all $t\ge 1$. Then running Algorithm \ref{algorithm:oftrl-general-sum-mg} will guarantee that
\begin{equation}\label{eq:NE-gap-general-sum}
    \ccegap(\hat{\pi}^T)\le CH \cdot \frac{1}{T} \sum_{t=1}^T \max_{h\in[H]}\barreg_h^t.
\end{equation}
for all $T\ge 2$, where $C>0$ is an absolute constant.
\end{lemma}
\begin{proof}
For any $(h,t)\in[H+1]\times[T]$, define
\begin{align*}
    \delta_h^t \defeq \max_{i\in[m]} \max_{s\in\cS} \paren{ V_{i,h+1}^{\dagger,\hat\pi_{-i,h+1}^{j}}(s) - V_{i,h+1}^{\hat\pi_{h+1}^{j}}(s) }.
\end{align*}
Then Lemma~\ref{lem.regret-decompose-general-sum} implies the recursive relationship
\begin{align*}
    \delta_h^t \le \reg_h^t + \sum_{j=1}^t \alpha_t^j \delta_{h+1}^j.
\end{align*}
(With $\delta_{H+1}^t\equiv 0$ for all $t\in[T]$.) 
Therefore we can imitate the proof of Lemma~\ref{lemma:recursion-of-value-estimation} and obtain that, for any $\barreg_h^t$ such that $\reg_h^t\le \barreg_h^t$ and $\barreg_h^t\ge \barreg_h^{t+1}$,
\begin{align*}
    \delta_h^t \le H\cbeta^{H-1}\cdot \frac{1}{t}\sum_{j=1}^t \max_{h'\in[H]}\barreg_{h'}^j,
\end{align*}
where $\cbeta=1+1/H= \sup_{j\ge 1} \sum_{t=j}^\infty \alpha_t^j$ by Lemma~\ref{lemma:alpha-ti}(a).

Further, by definition of the output policy $\hat{\pi}^T=\hat{\pi}_1^T$ (cf. Algorithm~\ref{algorithm:certified-policy}), we have
\begin{align*}
    & \quad \ccegap(\hat{\pi}^T) = V_{i,1}^{\dagger, \hat{\pi}_{-i,1}^T}(s_1) - V_{i,1}^{\hat{\pi}_1^T}(s_1) \le \delta_1^T \le H\cbeta^{H-1}\cdot \frac{1}{T}\sum_{t=1}^T \max_{h\in[H]}\barreg_{h}^t \\
    & \le CH \cdot \frac{1}{T}\sum_{t=1}^T \max_{h\in[H]}\barreg_{h}^t,
\end{align*}
where $C\le 3$ as $\cbeta^{H-1}\le (1+1/H)^H\le e\le 3$. This is the desired result.
\end{proof}

\subsection{Proof of Theorem~\ref{theorem:oftrl-general-sum-formal}}
\label{appendix:proof-oftrl-general-sum-formal}

\paragraph{Bounding per-state regret}
We first bound $\reg_{i, h}^t(s)$ (cf. definition in~\eqref{equation:reghits}), i.e. the per-state regret for the $i$-th player, for any fixed $(i,h,s,t)\in[m]\times[H]\times[S]\times[T]$. This is the main part of this proof. 

Throughout this part, we will fix $(i,h,s,t)$, and omit the subscript $(h,s)$ within the policies and Q functions, so that $\pi_{i,h}^t(\cdot|s)$ will be abbreviated as $\pi_i^t$, and $Q_{i,h}^t(s, \cdot)$ will be abbreviated as $Q_i^t$. We will also reload $T\ge 1$ to be any positive integer (instead of the fixed total number of iterations).

We first observe that the update~\eqref{equation:oftrl-pii} for $\pi_{i,h}^t(\cdot|s)$ is exactly equivalent to the OFTRL algorithm (Algorithm~\ref{alg:OFTRL}) with loss vectors $g_t=w_t(Q_i^t)^\top \pi_{-i}^t$ (understanding $g_0=0$ and $Q_i^0=0$), prediction vector $M_t=w_{t}(Q_i^{t-1})^\top\pi_{-i}^{t-1}$, and learning rate $\eta_t=\eta/w_t$. Therefore we can apply the regret bound for OFTRL in Lemma~\ref{lemma:OFTRL-auxillary} and obtain for any $T\ge 1$ that
\begin{equation}
\label{equation:oftrl-regret-initial-general-sum}
\begin{aligned}
    & \quad \max_{\pi_i^\dagger\in\Delta_{\cA_i}} \sum_{t=1}^T w_t \<\pi_i^t - \pi_i^\dagger, Q_i^t\pi_{-i}^t\> = \max_{\pi_i^\dagger\in\Delta_{\cA_i}} \sum_{t=1}^T \<\pi_i^t - \pi_i^\dagger, g_t\>\\
    & \le \frac{\log A_i}{\eta_T} + 
    \sum_{t=1}^{T} \eta_t \|g_t-M_t\|_\infty^2  -\sum_{t=1}^{T-1} \frac{1}{8\eta_t}\|\pi_i^t-\pi_i^{t+1}\|_1^2 \\
    & \le \frac{\log A_i\cdot w_T}{\eta} + 
    \sum_{t=1}^{T} \eta w_t\|Q_i^t\pi_{-i}^t-Q_i^{t-1}\pi_{-i}^{t-1}\|_\infty^2 \\
    & \le \frac{\log \Amax \cdot w_T}{\eta} + 
    \underbrace{\sum_{t=1}^{T} 2\eta w_t \linf{Q_i^t-Q_i^{t-1}}^2}_{\rm I} + \underbrace{\sum_{t=2}^T 2\eta w_t H^2\lone{\pi_{-i}^t - \pi_{-i}^{t-1}}^2}_{\rm II}.
\end{aligned}
\end{equation}
Above, the last inequality uses the fact that $\linf{(Q_i^t-Q_i^{t-1})\pi_{-i}^t}\le \linf{Q_i^t-Q_i^{t-1}}$ for $t\ge 1$, and $\linf{Q_i^{t-1}(\pi_{-i}^t - \pi_{-i}^{t-1})}\le H\lone{\pi_{-i}^t - \pi_{-i}^{t-1}}$ for $t\ge 2$.

For term ${\rm I}$, noticing that $\linf{Q_i^t - Q_i^{t-1}}\le \alpha_t H$ by~\eqref{eq:Q-update-general-sum}, we have
\begin{align*}
    {\rm I} \le \sum_{t=1}^T 2\eta w_t \alpha_t^2 H^2 = 2\eta H^2 \sum_{t=1}^T w_t \alpha_t^2.
\end{align*}

Bounding term ${\rm II}$ requires the following lemma on the stability of the iterates. The proof can be found in Section~\ref{appendix:proof-general-sum-stability}.
\begin{lemma}[Stability of iterates]
\label{lemma:general-sum-stability}
We have for any $i\in[m]$ and any $t\ge 2$ that (recall the subscripts $(h,s)$ are omitted below):
\begin{align}
\label{equation:stability-general-sum}
    \lone{\pi_i^t - \pi_{i}^{t-1}} \le 4\eta H.
\end{align}
Consequently,
\begin{align}
\label{equation:stability-total-general-sum}
    \lone{\pi_{-i}^t - \pi_{-i}^{t-1}} \le 4(m-1)\eta H.
\end{align}
\end{lemma}

Using Lemma~\ref{lemma:general-sum-stability}, we have
\begin{align*}
    {\rm II} \le \sum_{t=2}^T 2\eta w_t H^2 \cdot 16(m-1)^2 \eta^2 H^2 = 32\eta^3 H^4 (m-1)^2 \sum_{t=2}^T w_t.
\end{align*}

Plugging the preceding bounds into~\eqref{equation:oftrl-regret-initial-general-sum} and using $\alpha_T^1\cdot w_t=\alpha_T^t$ yields that
\begin{align*}
    & \quad \reg_{i,h}^T(s) = \max_{\pi_i^\dagger\in\Delta_{\cA_i}} \sum_{t=1}^T \underbrace{\alpha_T^t}_{\alpha_T^1\cdot w_t} \<\pi_i^t - \pi_i^\dagger, Q_i^t\pi_{-i}^t\> = \alpha_T^1 \max_{\pi_i^\dagger\in\Delta_{\cA_i}} \sum_{t=1}^T w_t \<\pi_i^t - \pi_i^\dagger, Q_i^t\pi_{-i}^t\> \\
    & \le \frac{\log\Amax \cdot \alpha_T^T}{\eta} + 2\eta H^2 \sum_{t=1}^T \alpha_T^t \alpha_t^2 + 32\eta^3 H^4 (m-1)^2 \underbrace{\sum_{t=2}^T \alpha_T^t}_{\le 1} \\
    & \stackrel{(i)}{\le} \frac{2H\log\Amax}{\eta T} + \frac{8\eta H^3}{T} + 32\eta^3 H^4 (m-1)^2 \\
    & \le C\brac{ \frac{H\log\Amax}{\eta T} + \frac{\eta H^3}{T} + \eta^3 H^4 (m-1)^2 } \eqdef \barreg_h^T.
\end{align*}
Above, (i) used the fact that $\alpha_T^T=\alpha_T=(H+1)/(H+T)\le 2H/T$, and $\sum_{t=1}^T\alpha_T^t\alpha_t^2 \le 4H/T$ by Lemma~\ref{lemma:alpha-convolution}(c), and $C\le 32$ is an absolute constant. This proves the per-state regret bounds claimed in Theorem~\ref{theorem:oftrl-general-sum-formal}.

\paragraph{Overall policy guarantee}
Plugging the above per-state regret bounds into Lemma~\ref{lemma:master-general-sum} yields that, the output policy $\hat{\pi}^T$ of Algorithm~\ref{algorithm:oftrl-general-sum-mg} achieves
\begin{align*}
    & \quad \ccegap(\hat{\pi}^T) \le CH \cdot \frac{1}{T} \sum_{t=1}^T \max_{h\in[H]} \barreg_h^T \\
    & \le \cO\paren{ \frac{H^2\log\Amax \log T}{\eta T} + \frac{\eta H^4\log T}{T} + \eta^3 H^5 (m-1)^2 }.
\end{align*}
Choosing $\eta=(\log\Amax \log T/(H^3T))^{1/4} (m-1)^{-1/2}$, the above can be upper bounded as
\begin{align*}
    \cO\paren{ H^{11/4}(\log\Amax \log T)^{3/4} \sqrt{m-1} \cdot T^{-3/4} + H^{13/4}(\log\Amax)^{1/4}(\log T)^{5/4}(m-1)^{-1/2}\cdot  T^{-5/4}  },
\end{align*}
which is the desired result.
\qed


\subsection{Proof of Lemma~\ref{lemma:general-sum-stability}}
\label{appendix:proof-general-sum-stability}

We first prove~\eqref{equation:stability-general-sum}. By the OFTRL update~\eqref{equation:oftrl-pii} and the smoothness of exponential weights (Lemma~\ref{lemma:Hedge-smoothness}), we have for any $t\ge 2$ that
\begin{align*}
    \lone{\pi_i^t - \pi_i^{t-1}} \le 2\linf{G^t - G^{t-1}},
\end{align*}
where $G^t, G^{t-1}$ are the (weighted) total losses in~\eqref{equation:oftrl-pii}:
\begin{align*}
    & G^t \defeq \frac{\eta}{w_t} \brac{\sum_{j=1}^{t-1} w_j Q_i^j \pi_{-i}^j + w_t Q_i^{t-1} \pi_{-i}^{t-1}}, \\
    & G^{t-1} \defeq \frac{\eta}{w_{t-1}} \brac{\sum_{j=1}^{t-2} w_j Q_i^j \pi_{-i}^j + w_{t-1} Q_i^{t-2} \pi_{-i}^{t-2}}. \\
\end{align*}
Therefore we have
\begin{align*}
    & \quad \lone{\pi_i^t - \pi_i^{t-1}} \le 2\linf{G^t - G^{t-1}} \\
    & \le 2\eta \linf{ \paren{\frac{1}{w_t} - \frac{1}{w_{t-1}}} \sum_{j=1}^{t-1} w_j Q_i^j \pi_{-i}^j } + 2\eta \linf{2Q_i^{t-1}\pi_{-i}^{t-1} - Q_i^{t-2}\pi_{-i}^{t-2}} \\
    & \stackrel{(i)}{\le} 2\eta H \cdot \underbrace{\paren{\frac{1}{w_{t-1}} - \frac{1}{w_{t}}} \sum_{j=1}^{t-1} w_j}_{\stackrel{(ii)}{=} H/(H+1)\le 1} + 2\eta H \le 4\eta H.
\end{align*}
Above, (i) uses the fact that $Q_i^j\pi_{-i}^j\in[0, H]$ entry-wise for all $j\ge 1$, and (ii) uses Lemma~\ref{lemma:wt}(b). This proves~\eqref{equation:stability-general-sum}.

The above directly implies~\eqref{equation:stability-total-general-sum} by the following bound on the TV distance (or $L_1$ norm) of product distributions~\citep{syrgkanis2015fast}:
\begin{align*}
    \lone{\pi_{-i}^t - \pi_{-i}^{t-1}} \le \sum_{j\neq i} \lone{\pi_j^t - \pi_j^{t-1}}.
\end{align*}
This completes the proof.
\qed
\section{Experimental details and additional studies}
\label{apdx:eager}

\subsection{Experimental details for Section~\ref{section:eager}}
\label{apdx:details}

\paragraph{Details about the game} The simulations in Section \ref{section:eager} is performed on the following two-player-zero-sum Markov game with $H=2$. The state space at $h=1$ only consists of a single state $\cS_1 =\{s_0\}$. The state space at $h=2$ consists of four different states $\cS_2 = \{s_{11},s_{12},s_{21},s_{22}\}$. The action spaces are the same for every state, namely $\cA = \{a_1,a_2\}, \cB =\{b_1, b_2\}$, i.e. each player has two actions. The transition from $\cS_1\times\cA\times\cB\to\cS_2$ is deterministic, which takes the following form:
\begin{align*}
    s_0\times a_i\times b_j \to s_{ij},\quad 1\le i,j\le 2.
\end{align*}
The instantaneous reward $r_h$ depends only on the action (and not the state), i.e., $r_h:\cA\times\cB\to [0,1]$, which takes values as (scaled) identity matrices:
\begin{align}
\label{equation:reward-matrix}
r_1(\cdot, \cdot) = \begin{bmatrix}
0.1 & 0 \\
0 & 0.1
\end{bmatrix},
~~~
r_2(\cdot, \cdot) = \begin{bmatrix}
1 & 0 \\
0 & 1
\end{bmatrix}.
\end{align}
Direct calculation yields that the Nash values and policies for this game is given by
\begin{align}
    & V_2^\star(s) = 0.5~~{\rm for}~s\in\set{s_{11}, s_{12}, s_{21}, s_{22}},
    \nonumber \\
    & Q_1^\star(s_0, \cdot, \cdot) = \begin{bmatrix}
    0.6 & 0.5 \\
    0.5 & 0.6
    \end{bmatrix},
    ~~~\mu_1^\star(\cdot|s_0) = \nu_1^\star(\cdot|s_0) = \begin{bmatrix}
    0.5 \\ 0.5 
    \end{bmatrix},
    ~~~V_1^\star(s_0) = 0.55. \label{equation:example-ne-layer-1}
\end{align}

\paragraph{Initialization} 
All algorithms in Figure~\ref{fig:main} use the following initialization $(\mu^0,\nu^0)$:
\begin{align*}
    \textup{At } h=1:~~&\mu^0_1(a_1|s_0) = 0.3,~~~~~~\mu^0_1(a_2|s_0) = 0.7,~~~~~~~\nu^0_1(b_1|s_0) = 0.7, ~~~~~~~\nu^0_1(b_2|s_0) = 0.3;\\
    \textup{At } h=2:~~&\mu^0_2(a_1|s_{11}) = 0.248,~\mu^0_2(a_2|s_{11}) = 0.752,~\nu^0_2(b_1|s_{11}) = 0.248, ~\nu^0_2(b_2|s_{11}) = 0.752;\\
    &\mu^0_2(a_1|s_{12}) = 0.500,~\mu^0_2(a_2|s_{12}) = 0.500,~\nu^0_2(b_1|s_{12}) = 0.168, ~\nu^0_2(b_2|s_{12}) = 0.832;\\
    &\mu^0_2(a_1|s_{21}) = 0.500,~\mu^0_2(a_2|s_{21}) = 0.500,~\nu^0_2(b_1|s_{21}) = 0.168, ~\nu^0_2(b_2|s_{21}) = 0.832;\\
    &\mu^0_2(a_1|s_{22}) = 0.752,~\mu^0_2(a_2|s_{22}) = 0.248,~\nu^0_2(b_1|s_{22}) = 0.248, ~\nu^0_2(b_2|s_{22}) = 0.752.
\end{align*}

Standard FTRL (Algorithm \ref{alg:FTRL}) and OFTRL (Algorithm \ref{alg:OFTRL}) by default uses the uniform distribution as the initialization, as it minimizes the (neg)entropy $\Phi(\cdot)$. To make them initialized at $\mu^0$, we change the regularizers for $\mu_h(\cdot|s)$ to be $\KL(\cdot \| \mu_h^0(\cdot|s))$ for the max-player. (And similarly $\KL(\cdot \| \nu_h^0(\cdot|s))$ for the min-player.) Note that our actual initialization above satisfies the property that all its values are bounded within the interval $[0.15,0.85]$. In particular, $\KL(\mu' \| \mu^0_h(\cdot|s))$ for any other $\mu'\in\Delta_{\cA}$ is bounded by $O(\log(1/0.15))=O(1)$, and thus all the convergence theorems will still hold with this modified regularizer, with at most a larger (multiplicative) constant than with the $\Phi(\cdot)$ regularizer.



\paragraph{Remark on runtime}
Running all our experiments takes approximately 6.46 hours CPU running time (Intel(R) Core(TM) i5-8250U CPU).
\begin{figure}[t]
\centering
\begin{minipage}{.4\textwidth}
    \centering
    \subcaption{{\footnotesize Value functions at $h=2$}}\label{fig:appendix1}
    \vspace{-.7em}
    \includegraphics[width=.99\textwidth]{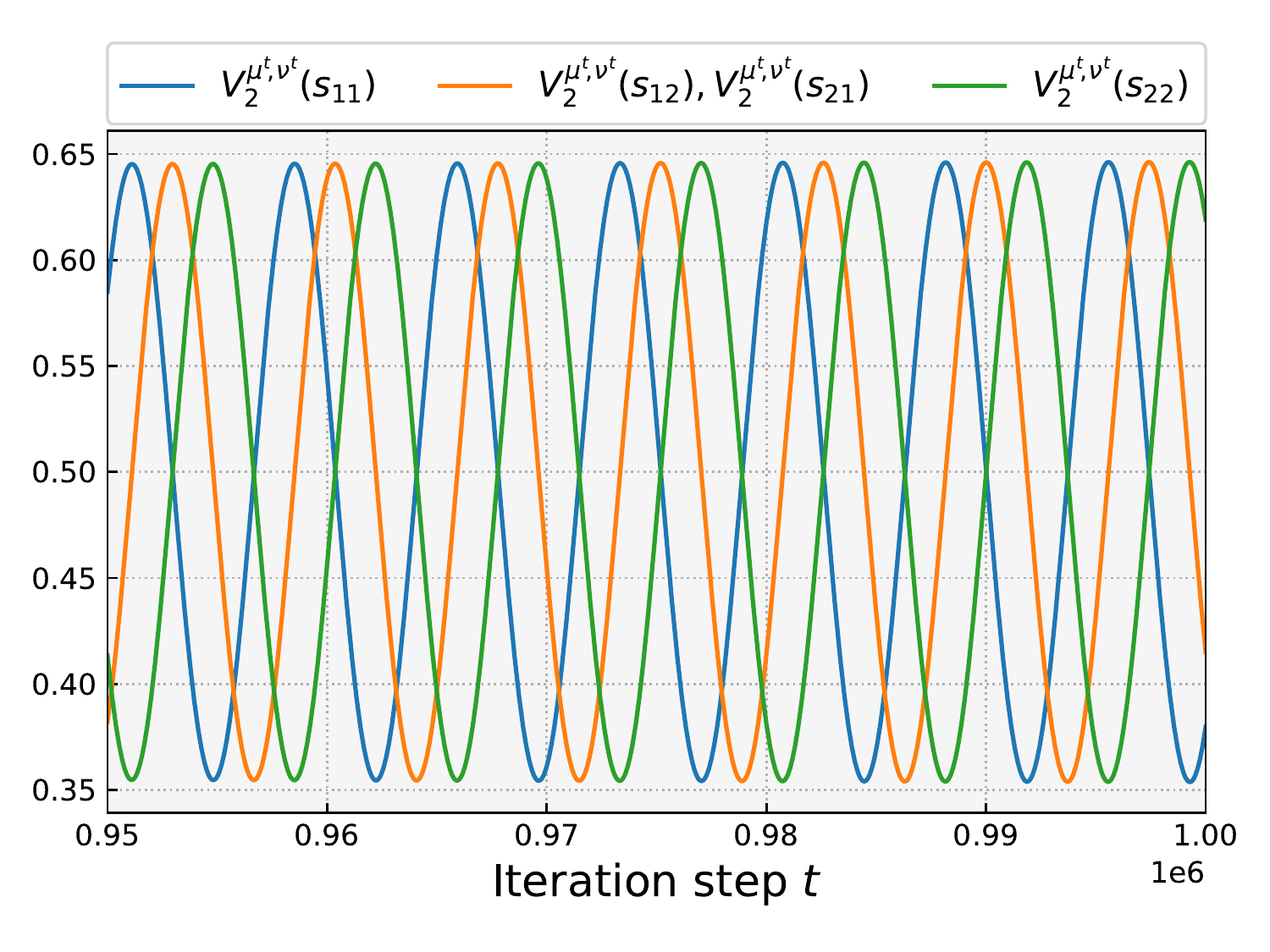}
\end{minipage}
\begin{minipage}{.4\textwidth}
    \centering
    \subcaption{{\footnotesize Policy at $h=1$ ($\mu_1^t(a_1|s_0)$)}}\label{fig:appendix2}
    \vspace{-.7em}
    \includegraphics[width=.99\textwidth]{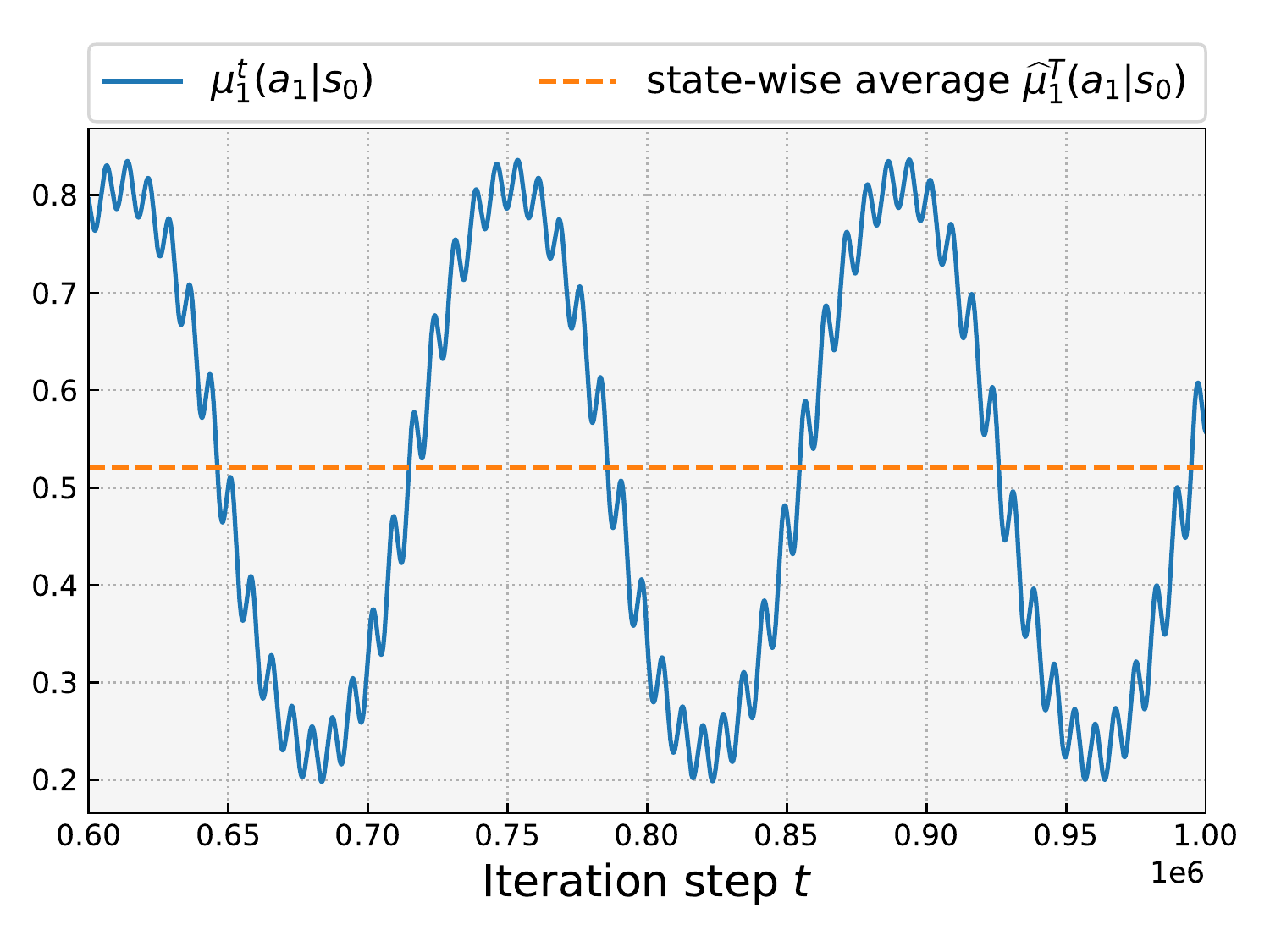}
\end{minipage}
\vspace{-.5em}
\caption{
\small Visualizations of the optimization trajectory of the INPG algorithm {\bf along a single run} with $T=10^6$ and $\eta=1/\sqrt{T}$. {\bf (a)} Value functions in the second layer $V^{\mu^t, \nu^t}_2(s)$ for all four states $s\in\set{s_{11}, s_{12}, s_{21}, s_{22}}$ over the last $5\times 10^4$ steps. {\bf (b)} Policy in the first layer, i.e. $\mu_1^t(a_1|s_0)$ over the last $4\times 10^5$ steps. The horizontal line plots the value of the final averaged policy $\hat{\mu}_1^T(a_1|s_0)$ (where the averaging is over the {\bf entire run} $t\in[T]$).
}
\label{fig:additional}
\end{figure}






\subsection{Additional visualizations for the INPG algorithm}
\label{apdx:additional}

Figure~\ref{fig:main} shows that the INPG algorithm (with $\eta=1/\sqrt{T}$) appears to converge much slower than $O(T^{-1/2})$ (which is the rate for FTRL with $\eta=1/\sqrt{T}$). Here we present some further understandings of this phenomenon by visualizing the optimization trajectories of the INPG algorithm.

Figure~\ref{fig:appendix1} shows the evolution of the value functions at $h=2$ over iteration step $t$, for the last $5\times10^4$ steps. For all four states, the policy optimization is equivalent to Hedge on the matrix game with identity reward matrix~\ref{equation:reward-matrix}, and thus exhibits an expected cyclic behavior and leads to the sinusoidal-like curves shown in Figure~\ref{fig:appendix1}. However, due to the choice of our specific initialization $(\mu^0, \nu^0)$, the four curves behave like the same periodic curves with different \emph{``phases''}. 




Figure~\ref{fig:appendix2} shows the evolution of the policy at $h=1$ (specifically, $\mu_1^t(a_1|s_0)$ which is the probability of the max-player taking action $a_1$) over $t$, for the last $4\times 10^5$ steps. (The result for the min-player is similar.) The curve also behaves periodically, and appears to be a superposition of two waves, one \emph{main waive} with larger magnitude and period, and another \emph{oscillation} with smaller magnitude and period. Qualitatively, the main wave is caused by the intrinsic cyclic behavior of learning with respect to the (fixed) reward at the $h=1$, while the oscillation is caused by the changing reward that is backed-up from $h=2$. 
Further, as the reward in the second layer has much higher magnitude than the first layer in this game, the oscillation has a non-negligible magnitude. 

The horizontal line in Figure \ref{fig:appendix2} plots the final output policy $\hat{\mu}_1^T(a_1|s_0)\approx 0.52$, which we recall is the average of $\mu_1^t(a_1|s_0)$ over the entire run $t\in[T]$ (cf. Section~\ref{section:eager}). Note that the unique Nash equilibrium satisfies $\mu^\star_1(a_1|s_0) = 0.5$~\eqref{equation:example-ne-layer-1}, and the error $\hat{\mu}_1^T(a_1|s_0)-\mu_1^\star(a_1|s_0)\approx 0.02$. We suspect that this may be an intrinsic bias caused by the aforementioned correlation between the two layers' learning processes (in particular, the different ``phases'' of the second-layer's learning over the four states), and may also be the cause of the slow convergence for INPG shown in Figure~\ref{fig:main}.


 

\subsection{Additional theoretical justifications}
\label{apdx:exp-theoretical}
\paragraph{INPG as an instantiation of Algorithm~\ref{alg:framework}}
Here we show why the instantiation of Algorithm~\ref{alg:framework} with $\beta_t = 1$ and
 \begin{equation*}
        \mu_h^t(a|s) \propto_a \mu_h^{t-1}(a|s)\exp\!\paren{\eta\brac{Q_h^{t-1}\nu_h^{t-1}}(s)},~~
        \nu_h^t(b|s)\propto_b \nu_h^{t-1}(b|s)\exp\!\paren{ - \eta  \brac{\paren{Q_h^{t-1}}^{\!\top}\!\! \mu_h^{t-1}}(s)}.
\end{equation*}
considered in Section \ref{section:eager} is equivalent to the Independent Natural Policy Gradient (INPG) algorithm. Indeed, choosing $\beta_t=1$ in Algorithm~\ref{alg:framework} ensures that $Q_h^t = Q_h^{\mu^t,\nu^t}$ (the true value function of $(\mu^t, \nu^t)$). Therefore, the above update is equivalent to 
 \begin{equation*}
        \mu_h^t(a|s)\! \propto_a\! \mu_h^{t-1}(a|s)\exp\!\paren{\eta\brac{Q_h^{\!\mu^{t\!-\!1}\!\!\!,\nu^{t\!-\!1}}\!\!\!\nu_h^{t\!-\!1}}(s)\!},~~
        \nu_h^t(b|s)\!\propto_b\! \nu_h^{t-1}(b|s)\exp\!\paren{\!\! - \eta  \brac{\paren{Q_h^{\!\mu^{t\!-\!1}\!\!\!,\nu^{t\!-\!1}}}^{\!\top}\!\! \mu_h^{t\!-\!1}}(s)\!}.
\end{equation*}
This is exactly an independent two-player version of the Natural Policy Gradient algorithm (e.g.~\cite{agarwal2021theory}), where each player plays an NPG algorithm as if they are facing their own Markov Decision Process, with the opponent fixed.

\paragraph{\NElayerone~lower bounds $\negap$} 
Here we show \NElayerone$(\mu,\nu)\le \negap(\mu,\nu)$ for any $(\mu, \nu)$.
From the definition of $V_h^\star$ we have that
\begin{align*}
    &V_h^\star(s) = \inf_{\nu} V_h^{\dagger,\nu}(s) = \sup_{\mu} V_h^{\mu,\dagger},\\
    \Longrightarrow~~ &V_h^{\mu,\dagger}\le V_h^\star(s) \le V_h^{\dagger,\nu}(s),~~\forall {\mu,\nu}.
\end{align*}
Thus
\begin{align*}
    &Q_h^\star(s,a,b) = \brac{r_h + \P_hV_{h+1}^\star}(s,a,b) \le  \brac{r_h + \P_hV_{h+1}^{\dagger,\nu}}(s,a,b) = Q_h^{\dagger,\nu}(s,a,b),\\
    &Q_h^\star(s,a,b) = \brac{r_h + \P_hV_{h+1}^\star}(s,a,b) \ge  \brac{r_h + \P_hV_{h+1}^{\mu,\dagger}}(s,a,b) = Q_h^{\mu, \dagger}(s,a,b)\\
    \Longrightarrow&\quad  Q_h^{\mu, \dagger}(s,a,b)\le Q_h^\star(s,a,b)\le Q_h^{\dagger,\nu}(s,a,b).
\end{align*}
Thus for our example
\begin{align*}
    \textup{\NElayerone}(\mu,\nu) &= \max_{\mu_1^\dagger} \brac{(\mu_1^\dagger)^\top Q^\star_1 \nu_1}(s_0) -  \min_{\nu_1^\dagger} \brac{\mu_1^\top Q^\star_1 \nu_1^\dagger}(s_0)\\
    &\le  \max_{\mu_1^\dagger} \brac{(\mu_1^\dagger)^\top Q^{\dagger,\nu}_1 \nu_1}(s_0) -  \min_{\nu_1^\dagger} \brac{\mu_1^\top Q^{\mu,\dagger}_1 \nu_1^\dagger}(s_0)\\
    &= V_1^{\dagger,\nu}(s_0) - V_1^{\mu,\dagger}(s_0) = \negap(\mu,\nu).
\end{align*}



\end{document}